\documentclass[lettersize,journal]{IEEEtran}
\usepackage{amsmath,amsfonts}
\usepackage{algorithmic}
\usepackage{algorithm}
\usepackage{array}
\usepackage{textcomp}
\usepackage{stfloats}
\usepackage{url}
\usepackage{verbatim}
\usepackage{graphicx}
\usepackage{cite}
\usepackage{algorithm}
\usepackage{algorithmic}
\usepackage{xcolor}
\usepackage{multirow}
\usepackage{subfigure}
\usepackage{amsmath}
\usepackage{amssymb}
\usepackage{mathtools}
\usepackage{amsthm}
\newtheorem{theorem}{Theorem}

\newtheorem{lemma}{Lemma}
\newtheorem{corollary}{Corollary}
\newtheorem{definition}{Definition}
\newtheorem{assumption}{Assumption}
\newtheorem{remark}{Remark}
\DeclareMathOperator*{\E}{\mathbb{E}}

\allowdisplaybreaks[4]

\begin{document}

\title{
% {\color{black}CaDiff: Asynchronous Diffusion–Driven Causal State Representation for Robust Decision-Making under Perturbations}}
Learning Causal States Under Partial Observability and Perturbation
}

\author{
        Na Li,~\IEEEmembership{Graduate~Student~Member,~IEEE,}
        Hangguan~Shan,~\IEEEmembership{Senior~Member,~IEEE,}
		Wei~Ni,~\IEEEmembership{Fellow,~IEEE,}
        Wenjie~Zhang,~\IEEEmembership{Senior~Member,~IEEE,}
        Xinyu~Li,~\IEEEmembership{Member,~IEEE,}
		and Yamin~Wang
        % \thanks{This work was supported in part by the Natural Science Foundation of China under Grants U21B2029 and U21A20456, and in part by the Zhejiang Provincial Natural Science Foundation of China under Grant LR23F010006. The work of Na Li was supported by the China Scholarship Council.
        
        % N. Li and H. Shan are with the College of Information Science and Electronic Engineering, Zhejiang University, Hangzhou 310027, China (e-mails: nlee@zju.edu.cn; hshan@zju.edu.cn). 
        
        % W. Ni is with the School of Engineering, Edith Cowan University, Perth, WA 6027, and the School of Computer Science and Engineering, University of New South Wales, Sydney, NSW 2052, Australia (e-mail: wei.ni@ieee.org). 
        
        % W. Zhang is with the School of Computer Science and Engineering, The University of New South Wales, Sydney, NSW 2052, Australia (e-mail: wenjie.zhang@unsw.edu.au).
        
        % X. Li is with the 
        % % Department of Industrial and Manufacturing Systems Engineering, 
        % School of Mechanical Science and Engineering, Huazhong University of Science and Technology, Wuhan 430074, China (e-mail: lixinyu@hust.edu.cn). 
        
        % Y. Wang is with China Mobile Research Institute, Beijing 100053, China (e-mail: wangyamin@chinamobile.com).}
        }

\markboth{Journal of \LaTeX\ Class Files,~Vol.~14, No.~8, August~2021}%
{Shell \MakeLowercase{\textit{et al.}}: A Sample Article Using IEEEtran.cls for IEEE Journals}


\maketitle

\begin{abstract}
    A critical challenge for reinforcement learning (RL) is making decisions based on incomplete and noisy observations, especially in perturbed and partially observable Markov decision processes (P$^2$OMDPs). Existing methods fail to mitigate perturbations while addressing partial observability. We propose \textit{Causal State Representation under Asynchronous Diffusion Model (CaDiff)}, a framework that enhances any RL algorithm by uncovering the underlying causal structure of P$^2$OMDPs. This is achieved by incorporating a novel asynchronous diffusion model (ADM) and a new bisimulation metric. ADM enables forward and reverse processes with different numbers of steps, thus interpreting the perturbation of P$^2$OMDP as part of the noise suppressed through diffusion. The bisimulation metric quantifies the similarity between partially observable environments and their causal counterparts. Moreover, we establish the theoretical guarantee of CaDiff by deriving an upper bound for the value function approximation errors between perturbed observations and denoised causal states, reflecting a principled trade-off between approximation errors of reward and transition-model. Experiments on Roboschool tasks show that CaDiff enhances returns by at least 14.18\% compared to baselines. CaDiff is the first framework that approximates causal states using diffusion models with both theoretical rigor and practicality.
\end{abstract}


\section{Introduction}
\IEEEPARstart{P}{artially} observable Markov decision processes (POMDPs) are relevant to a wide range of real-world applications, such as robotics, autonomous driving, and finance~\cite{RL_2}. 
% \cite{RL_2, RL_1}
Perturbed POMDPs (P$^2$OMDPs) have been generally challenging for deep reinforcement learning (DRL), where deep RL (DRL) agents must make optimal decisions under both incomplete state observations and additional perturbations caused by sensor noise, environmental disruptions, or discrepancies between training datasets and actual conditions.
This noise complicates state inference by masking essential features or introducing false signals.

% % Challenges
% When dealing with perturbed POMDPs, two major challenges emerge, i.e., noise interference and causal state extraction. 
% Specifically, in POMDPs, causal states 
% % are crucial as they 
% capture the essential relationships between past observations and future outcomes, filter out irrelevant data, and contribute to improved efficiency.
% On the one hand, noise exacerbates the already difficult task of making inferences from partially observable states, making it even more challenging to extract meaningful information under perturbations. The denoising process is further complicated by the risk of distorting the original data, posing a challenge in overcoming perturbations without losing crucial causal information. Striking a balance between effective denoising and preserving the original causal structure becomes a significant hurdle.
% On the other hand, existing definitions of causal states are largely based on MDPs~\cite{estruch22a}, with limited research addressing causal states within the context of POMDPs. Designing causal state representations that are specifically tailored to the unique characteristics of POMDPs represents another challenge. Therefore, developing robust causal state representations that can uncover causal relationships while seamlessly integrating denoising techniques is crucial for addressing perturbed POMDPs.

% \textit{A critical challenge is developing algorithms capable of mitigating the impact of perturbations and partial observability, and uncovering the underlying causality of perturbed POMDPs.}
Causal state extraction has offered an effective means to capture the essential relationships between past observations and future returns, filter out irrelevant data, and contribute to improved efficiency.
Existing studies have been primarily based on Markov decision processes (MDPs)~\cite{estruch22a}, with little effort on POMDPs. Let alone P$^2$OMDPs, where noise can comprise causal state extraction from the partially observable states. 
While it is possible to denoise the perturbed partially observed states, balancing between denoising and preserving the original causal structure is vital.

% A promising approach emerging to address perturbed POMDPs is combining the strengths of diffusion models and bisimulation techniques. 

% DM_1, 
\textit{Diffusion models (DMs)}, as the state-of-the-art generative models, have offered effective denoising capability.
DMs \cite{DM_2} are adept at removing noise while preserving essential data features, through an iterative process that transforms noisy samples into high-quality, real samples. 
In contrast, traditional generative model-based approaches, which are integrated into POMDPs through algorithms, including deep variational RL~\cite{DVRL} and structured sequential variational auto-encoders~\cite{ASR_base}, often generate samples by learning latent data representations, as opposed to tackling noise.
% This makes diffusion models a superior choice for denoising and preserving features simultaneously.
% Empowered by the ability of generative models to learn complex multimodal distributions, they 
DMs have gained significant attention in decision-making tasks, utilized as trajectory generators or tools for state representation~\cite{DM_offline_2, DM_offline_1, DM_offline_3},
% Despite their promising potential in addressing POMDP challenges, 
but have not accounted for causality.
% leaving this as an open area for further exploration.

\textit{Bisimulation} offers a rigorous mathematical framework for evaluating state equivalence based on outcomes.  
Research on causal state representation (CSR) has been developed to extract abstract features from perturbed observations. Leveraging these abstract representations (instead of raw data) has enhanced decision-making efficiency in both MDPs~\cite{Repre} and POMDPs~\cite{CSR_base}.  
{\color{black}Related approaches} include bisimulation-based methods \cite{bisimulation}, Kalman filters \cite{Kalman}, ordinary differential equation-based recurrent models \cite{ODE_GRU}, world models \cite{world}, and studies establishing a connection between predictive state representations and bisimulation via causal states \cite{CSR_base}. 
% This concept has also been explored in \cite{Intro_other1, Intro_other2}.  
However, no existing methods have captured the impact of perturbations.

\subsection{Contribution}
This paper enhances the decision-making of DRL for \text{P}$^2$\text{OMDPs} by extracting and denoising causal states. A novel approach, \textit{Causal State Representation under Asynchronous Diffusion Model (CaDiff)}, is proposed to mitigate perturbations and uncover causality within denoised and partially observable states of P$^2$OMDPs. CaDiff offers a generic framework for any RL algorithm, 
% {\color{red}This work provides a new theoretical perspective and methodological foundation for achieving interpretable and robust RL under incomplete observations.}
with the following contributions:

\textbf{Algorithm Design}: 
Tailored for P$^2$OMDPs, CaDiff enhances any DRL operating with noisy and incomplete observations and noisy rewards, by designing \textit{a novel asynchronous diffusion model (ADM)} and \textit{a new bisimulation metric}. 
\begin{itemize}
    \item 
\underline{\textit{ADM}} enables forward and reverse processes with different numbers of steps in an asynchronous setting.
% , allowing flexible approximation of the distributions governing transition dynamics and rewards. 
% Weaker noise is injected during the reverse denoising process, compared to the noise suppressed during the forward denoising process. 
Less noise is injected via fewer steps in the forward process than the noise suppressed in the reverse process, suppressing the noisy perturbation of P$^2$OMDPs.
% This match noisy inputs, achieving finer modeling of perturbed observations.}
The difference in the number of steps is regulated by noise intensity, effectively denoising each dimension of the observation features and preserving inherent causality.
    \item 
\underline{\textit{The new bisimulation metric}} quantifies the similarity between a partially observable environment and its causal state. This metric, grounded in the consistency of input states and rewards, measures the alignment between underlying distributions of transition dynamics and rewards.    
By incorporating causal consistency constraints, the denoised causal state captures the underlying causal structure for better decision-making.
\end{itemize}
CaDiff combines causality assessment with denoising. Guided by causality measured by the new bisimulation metric, ADM effectively denoises the perturbed partial observations of POMDPs. 
ADM is adept at handling disturbances across various scales and serves as a standalone denoising module.

% {\color{black} \textbf{COMMENT:} Please reorganise this contribution from the general contribution (addressing the nominated challenge) to detailed technical novelties contributing to the general contribution. E.g., the causal state representation is a detailed technical novelty.}

\textbf{Theoretical Guarantee}:
We establish a statistical theoretical guarantee for CaDiff with the upper bound of the value function approximation (VFA) error derived between perturbed observations and corresponding denoised causal states in P$^2$OMDPs. 
% Notably, we establish a general result incorporating the POMDP-based bisimulation metric under any distribution approximation errors. Specifically, for the ADM designed in this paper, the distribution estimation error is quantified using the Wasserstein-1 distance, which is further analyzed by decomposing it into initialization error, score estimation error, and discretization error.
Based on the analysis incorporating the new bisimulation metric under arbitrary distribution approximation errors, the distribution estimation error of ADM is quantified using the Wasserstein distance and analyzed through initialization, score estimation, and discretization errors.
% , effectively capturing the bisimulation metric within the ADM algorithm.
%
We also establish a sample complexity bound of CaDiff w.r.t. the smoothness of the data distribution and the dimensions of the observation and action spaces. 
% Our analysis demonstrates that CaDiff effectively tightens the upper bound of the VFA.
Compared to standard RL algorithms, the additional computational overhead by CaDiff is marginal and polynomial in the logarithm of the observation and action space dimensions.

% {\color{black} \textbf{COMMENT:} Meaning?}

\textbf{Extensive Simulation}: CaDiff is tested in six P$^2$OMDP environments with the soft actor-critic (SAC) RL algorithm. 
CaDiff outperforms all baselines by at least 14.18\% in return. Despite more parameters, CaDiff delivers superior early-stage training performance in five of the six environments, while ultimately achieving the highest final return across all six.
The ablation studies show that state denoising is much more effective than reward denoising due to typically higher observation dimensions.
CaDiff incurs a mild cost of less than 7 GFLOPs.

% {\color{black} \textbf{COMMENT:}  Any finding, or anything worth mentioning?}

\subsection{Related Work}
\subsubsection{Causal State Representation}
% To enhance decision-making under POMDPs, s
Several studies have focused on deriving CSR for decision-making generalization under POMDPs. For instance, the authors of \cite{CSR_base} approximated causal states in POMDPs by invariant prediction. {\color{black}This concept has also been explored in \cite{Intro_other1, Intro_other2}. }
Some studies~\cite{karimi2024diffusion, komanduri2024causal} leveraged DMs for causal representation by counterfactual inference or invention.
Utilizing domain-invariant causal features, the authors of \cite{CSR_16} proposed invariant causal imitation learning to address distribution shifts. Some other works, e.g., \cite{CSR_15}, proposed ensemble representations that leverage multi-modal sensor inputs to boost generalizability for self-driving agents under uncertainty quantification. The PlanT framework \cite{CSR_6} serves as a learnable planner module grounded in object-centric representations. Moreover, the realm of RL has witnessed advancements in state representation through self-supervised learning, including hierarchical skill decomposition \cite{CSR_17}, time-contrastive learning \cite{CSR_32}, variational auto-encoder \cite{han2024dynamical}, confounding and causal inference \cite{cannizzaro2023car}, and deep bisimulation metric learning \cite{bisimulation}.
However, there is a lack of consideration of perturbation-based CSR.

\subsubsection{RL with DM}
DM is a generative model, initially for image generation \cite{DM_2}. It has been adopted in decision-making for state-based tasks, especially for perturbed states. In RL, DMs can be utilized, not only for direct decision-making~\cite{DM_offline_1, DM_offline_2, DM_offline_3}, but also for effective denoising and distribution estimation.
% For instance, 
DMBP \cite{DM_offline_3} utilizes a DM as a denoiser rather than a generator. 
DIPO \cite{DM_online_1} utilizes a DM to address the denoising problem in model-free RL. 
The authors of \cite{DM_offline_9} presented a sharp statistical theory of distribution estimation for the conditional DM. 
However, these studies do not differentiate the noise 
% variance of data 
used for training, hence limiting the effectiveness of denoising.

% However, existing works lack theoretical guarantees.
% Compared to offline settings, online RL can better adapt to dynamic environments, yet relevant research is scarce. 
% Furthermore, UTOPIA \cite{DM_online_4} serves as a diffusion model application case in multi-agent RL. However, existing works in the online setting lack theoretical guarantees.

To the best of our knowledge, no existing frameworks have extracted causal states of P$^2$OMDPs, which is nontrivial and necessitates the integration of CSR with effective denoising.

% {\color{black} \textbf{COMMENT:} Add a concluding remark indicating no one has done what you plan to do, and it is non-trivial to apply the existing technique to achieve what you want to achieve. You have already had something, but commented on it.}

\section{Preliminaries and Problem Formulation}
\subsection{Wasserstein Distance and H\"older Ball}\label{appendix:Wasserstein}
The distance between two distributions can be measured using Wasserstein distance
% , which is employed 
to measure convergence of CaDiff.
% Let $d: X \times X \rightarrow [0, \infty)$ denote a distance function and $\Omega$ denote the set of all joint distributions with marginals $\mu$ and $\lambda$ over space $X$. The definition and properties of Wasserstein distance are provided as follows.

\begin{definition}[Wasserstein metric \cite{villani2008optimal}]
\label{def:wass-primal}
Let $d: X \times X \rightarrow [0, \infty)$ denote a distance function, and $\Omega$ denote the collection of all joint distributions on $X \times X$ whose marginals are $\mu$ and $\lambda$.
The Wasserstein metric is defined as
\begin{align}
\label{eq:wasserstein}
    W_p(d)(\mu, \lambda) = \left(\inf\nolimits_{\omega \in \Omega}\mathbb{E}_{(x_1, x_2) \sim \omega}[d(x_1, x_2)^p]  \right)^{1/p}.
\end{align}
\end{definition}

\smallskip

\begin{definition}[Dual formulation of Wasserstein metric \cite{villani2008optimal}]
A dual formulation of the Wasserstein metric is given by
\begin{align*}
    W_p(d)(\mu, \!\lambda) \!=\! \left(\sup\nolimits_{\zeta \oplus \psi \leq d^p}\mathbb{E}_{x_1 \sim \mu}[\zeta(x_1)] \!+\! \mathbb{E}_{x_2 \sim \lambda}[\psi(x_2)] \right)^{\!1/p}\!,
\end{align*}
where $\zeta \oplus \psi \leq d^p$ is equivalent to $\zeta(x) + \psi(y) \leq d{(x, y)}^p, ~\forall(x, y) \in X \times X$.

In the case of $p = 1$, the dual formulation simplifies to
\begin{align}
\label{eq:wass-dual}
    W_1(d)(\mu, \lambda) \!=\!\! \sup_{f \in \mathrm{Lip}_{1,d}(X)}\mathbb{E}_{x_1 \sim \mu}[f(x_1)] \!-\! \mathbb{E}_{x_2 \sim \lambda}[f(x_2)],
\end{align}
where $\mathrm{Lip}_{1,d}(X)$ denotes the set of all 1-Lipschitz functions $f: X \rightarrow \mathbb{R}$ satisfying $|f(x_1) - f(x_2)| \le d(x_1, x_2)$.
\end{definition}

\smallskip
\smallskip

There is a closed-form expression of Gaussian measures for the 2-Wasserstein metric $W_2(|\cdot|_2)$ (abbreviated as $W_2$)~\cite{olkin1982distance}:
\begin{align}
\label{eq: 2-Wasserstein}
    \!\!\!\!W_2(\mathcal{N}(\mu_i,\! \Sigma_i), \mathcal{N}(\mu_j,\! \Sigma_j)) \!\!=\!\! (\left\|{\mu_i \!-\! \mu_j}\right\|_2^2 \!+\! \left\|{\Sigma_i \!-\! \Sigma_j}\right\|_{\mathcal{F}}^2)^{1/2}\!\!.
\end{align}
Here, $\|\cdot\|_{\mathcal{F}}$ represents the Frobenius norm. If $\Sigma_i, \Sigma_j \to 0$ (i.e., both distributions collapse to point masses), the 2-Wasserstein metric recedes to the Euclidean distance between the points.
\begin{lemma}[$p$-Wasserstein Inequality \cite{villani2008optimal}]
\label{lemma:wass-lemma}
For any distributions $\mu$ and $\lambda$, $W_q(\mu, \lambda) \geq  W_p(\mu, \lambda), ~\forall q \geq p$.
% \begin{align}
%     W_p(\mu, \lambda) \leq  W_q(\mu, \lambda), ~\forall p \leq q.
% \end{align}
\end{lemma}

\begin{lemma}[Bounds on Wasserstein distance \cite{santambrogio2015optimal}]
\label{lemma:wass-diam}
For any distributions $\mu$ and $\lambda$ on $X$,
\begin{align}
\!\!\!W_1(\mu, \lambda) \!\leq\! W_p(\mu, \lambda) \!\leq\! \mathrm{diam}(X; d)^{\frac{p-1}{p}}W_1(\mu, \lambda)^{\frac{1}{p}}, \forall p \!\geq\! 1,\!\!
\end{align}
where $\mathrm{diam}(X; d)$ is the diameter of the metric space $(X; d)$.
\end{lemma}

% \subsection{H\"older norm and H\"older ball}\label{appendix:Holder}
To estimate the distribution of dynamic transitions {\color{black}for P$^2$OMDPs}, we define H\"older norm and H\"older ball, as follows:
\begin{definition}[H\"older norm and H\"older ball]\label{appendix:Holder_norm}
Let $b = m+\gamma > 0$ represent the smoothness degree, with $m = \lfloor b \rfloor$ an integer and $\gamma \in [0,1)$.
For function $f:\mathbb{R}^{w_x} \to \mathbb{R}$ and multi-index $\mathbf{s}$, we define the H\"older norm as
\begin{align*}
   \left|f\right|_{\mathcal{H}^b(\mathbb{R}^{w_x})} \!\coloneqq\!\!\!\!\max_{\mathbf{s}: \!\left\|\mathbf{s}\right\|\!_1 < m} \sup_{\mathbf{x}}\!|\partial^{\mathbf{s}}\! f(\mathbf{x})|  \!\!+\!\!\!  \max_{\mathbf{s}:\! \left\|\mathbf{s}\right\|\!_1 = s} \!\sup_{\mathbf{x}\neq \mathbf{z}} \frac{ \left|\partial^{\mathbf{s}} \!f(\mathbf{x}) \!\!-\!\! \partial^{\mathbf{s}} \!f(\mathbf{z})\right| }{ \left\|\mathbf{x} \!-\! \mathbf{z}\right\|_{\infty}^{ \gamma }}\!.
\end{align*}
% {\color{black}Here, $\mathbf{s}$ denotes a multi-index.

A function $f$ is said to be $b$-H\"older if and only if $|f|_{\mathcal{H}^{b}(\mathbb{R}^{w_x})} < \infty$.
% \end{definition}
% \smallskip
% \begin{definition}[H\"older ball]\label{appendix:Holder_ball}
A H\"{o}lder ball of radius $B > 0$ is defined as
\begin{align*}
    \mathcal{H}^{b}(\mathbb{R}^{w_x}, B) = \{f: \mathbb{R}^{w_x} \rightarrow \mathbb{R} \mid \left\|f\right\|_{\mathcal{H}^b(\mathbb{R}^{w_x})} < B\}.
\end{align*}
\end{definition}

% \section{Problem Formulation}

\begin{figure}[t]
    \centering
    \includegraphics[width=0.9\linewidth]{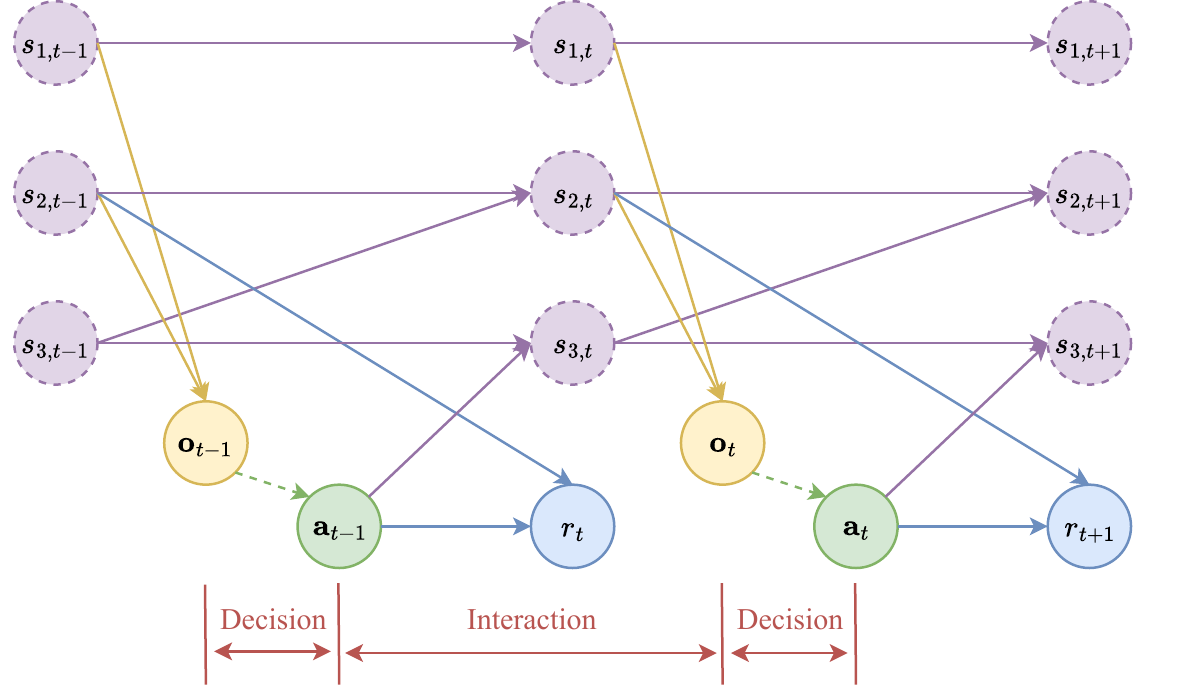}
    \vspace{-0.3cm}
    \caption{System model: Solid-line and dashed-line circles denote observed and unobserved variables, respectively; solid and dashed lines represent causality and decision relationships, respectively.}
    \vspace{-0.4cm}
    \label{fig:system}
\end{figure}
\subsection{RL for Perturbed POMDP}
Environments are modeled as $\mathcal{M}=\left(\mathcal{S}, \mathcal{A}, \mathcal{O}, \gamma, F, G, H\right)$ in POMDPs, where $\mathcal{S}$, $\mathcal{A}$, and $\mathcal{O}$ are the state, action, and observation spaces, respectively; $\gamma$ is the discount factor; $F$, $G$, and $H$ are the observation, reward, and transition functions, respectively.
Consider a sequence of samples $\left\{\langle \mathbf{o}_t, \mathbf{a}_t, r_t \rangle\right\}_{t=1}^T$, where $\mathbf{o}_t \in \mathcal{O}$ is the observation at time $t$; {\color{black}$\mathbf{a}_t \in \mathcal{A} \subseteq \mathbb{R}^{w_a}$ is the $w_a$-dimensional action chosen at time $t$}; and $r_t \in [0,1]$ denotes the reward.
Let $\mathbf{s}_t=\left\{{s}_{1,t},\dots, {s}_{w_s,t}\right\}\in \mathcal{S}$ denote the $w_s$-dimensional true state. 
We describe the P$^2$OMDP as the following functions and transitions:
\begin{subequations}\label{eq:op_origin}
    \vspace{-0.2cm}
    \begin{align}
        \mathbf{o}_{t}  &= F\left(\mathbf{s}_{t}, \mathbf{e}_{t}\right) \iff P\left(\mathbf{o}_{t}\mid \mathbf{s}_{t}\right), \label{eq:op_origin_observ}\\
        r_{t}  &= G\left(\mathbf{s}_{t-1}, \mathbf{a}_{t-1}, \mathbf{\varepsilon}_{t}\right) \iff P\left(r_{t}\mid \mathbf{s}_{t-1}, \mathbf{a}_{t-1}\right), \label{eq:op_origin_reward}\\
        \mathbf{s}_{t}  &= H\left(\mathbf{s}_{t-1}, \mathbf{a}_{t-1}, \mathbf{\eta}_{t}\right) \iff P\left(\mathbf{s}_{t}\mid \mathbf{s}_{t-1}, \mathbf{a}_{t-1}\right), \label{eq:op_origin_state}
    \end{align}
\end{subequations}
where $\mathbf{e}_t$, $\mathbf{\varepsilon}_t$, and $\mathbf{\eta}_t$ are the associated independent and identically distributed (i.i.d.) random noises for each $t$, respectively. 

Given $\mathbf{a}_{t-1}$ and $\mathbf{s}_{t-1}$, state $\mathbf{s}_t$ is independent of the states and actions occurred before time $t-1$. Action $\mathbf{a}_{t-1}$ affects $\mathbf{s}_t$, but does not directly affect $\mathbf{o}_t$, which is affected by $\mathbf{s}_t$. The reward $ r_{t}$ is influenced by $\mathbf{a}_{t-1}$ and $\mathbf{s}_{t-1}$. Let $\mathbf{\varepsilon}_{t}$ in the reward function capture noise, e.g., measurement errors.
% {\color{red} \textbf{COMMENT:} Please cite Fig. 1 here, since Fig. 1 describes POMDP.}

\subsection{Causal State Representation and Bisimulation}
% We extract causal states from observations.
% % to focus on those directly influencing rewards and subsequent states.
The structural relationships among the different dimensions of $\mathbf{s}_t$ indicate that action $\mathbf{a}_{t-1}$ may not affect all dimensions of $\mathbf{s}_t$, and the reward $r_t$ may not be affected by all dimensions of $\mathbf{s}_{t-1}$.
Fig.~\ref{fig:system} shows an example with $w_s=3$, i.e., $\mathbf{s}_t=\left[{s}_{1,t}, {s}_{2,t}, {s}_{3,t}\right]^{\rm T}$. State ${s}_{3,t-1}$ affects ${s}_{2,t}$, but there is no connection between $\mathbf{a}_{t-1}$ and ${s}_{3,t-1}$. Only ${s}_{2,t-1}$ has an edge towards $r_t$. 
By extracting causal state $\hat{\mathbf{s}}_t=\left[\hat{s}_{1,t}, \hat{s}_{2,t}\right]$ within the causal state space $\mathcal{S}_{\mathrm{c}}$ from $\mathbf{o}_{t}$, we suppress redundant information and retain the causal structure. 
% The causal state is $\mathbf{s}^\prime_t = \left[{s}_{1,t}, {s}_{2,t}, {s}^\prime_{3,t}\right]^{\rm T}$.

% CSR has been explored to differentiate pertinent information from irrelevant details \cite{Bisimulation_ori}. 
\textit{As a type of CSR, states and observations are bisimilar if they yield the same expected reward and have equivalent distributions over subsequent bisimilar states and observations}~\cite{Bisimulation_work}. 
% To this end, we assert that they 
{\color{black}To this end, they exhibit a bisimulation relationship.}
% , providing a mathematically rigorous definition of how any two bisimilar environments can yield the same outcome.
With the environment's dynamics $P(\mathbf{s}_{t+1}, r_{t+1}|\mathbf{s}_t, \mathbf{a}_t)$, the similarity between two environments can be expressed as that between their state transition and reward functions.
Following~\cite{Bisimulation_def},  equivalence in P$^2$OMDPs is defined as follows.
\begin{definition}[CSR under bisimulation]\label{def:bisimulation}
    Consider a P$^2$OMDP $\mathcal{M} = (\mathcal{S}, \mathcal{A}, \mathcal{O}, \gamma, F, G, H)$ with causal state space $\mathcal{S}_{\mathrm{c}}$ and learned encoder $\hat{F}: \mathcal{O} \to \mathcal{S}_{\mathrm{c}}$. $\hat{\mathbf{s}}_t = \hat{F}(\mathbf{o}_t)$ is a \textit{causal state} if, for any action $\mathbf{a}_t$, $P(r_{t+1} \mid \hat{\mathbf{s}}_t, \mathbf{a}_t) = P(r_{t+1} \mid \mathbf{o}_t, \mathbf{a}_t)$ and $P({\mathbf{s}}_{t+1} \mid \hat{\mathbf{s}}_t, \mathbf{a}_t) = P({\mathbf{s}}_{t+1} \mid \mathbf{o}_t, \mathbf{a}_t)$.
    % \begin{align*}
    %     P(r_{t+1} \mid \hat{\mathbf{s}}_t, \mathbf{a}_t) = P(r_{t+1} \mid \mathbf{o}_t, \mathbf{a}_t), \\
    %     P({\mathbf{s}}_{t+1} \mid \hat{\mathbf{s}}_t, \mathbf{a}_t) = P({\mathbf{s}}_{t+1} \mid \mathbf{o}_t, \mathbf{a}_t).
    % \end{align*}
\end{definition}

% {\color{black} \paragraph*{Goal}
% By denoising states and rewards, estimating environment dynamics, and extracting causal states, we aim to represent causal states $\mathbf{s}_t$ under a perturbed POMDP. We also wish to design a diffusion model considering the differentiation of noise intensity within the data.}
% % The ultimate goal is to achieve an RL algorithm under POMDP with samples as small as possible.

\section{Algorithm Design of CaDiff}\label{sec:algorithm}

% \subsection{CaDiff structure}\label{appendix:imple}
As summarized in Algorithm~\ref{algorithm-main}, the CaDiff framework consists of three modules. It employs ADM to denoise states and rewards separately. Then, it approximates causal states based on the denoised states and rewards with bisimulation. The approximated causal states and denoised rewards serve as samples for RL algorithms for decision-making.
% , with the diagram in Fig.~\ref{fig:algorithm}. 
% {\color{black} As a causal state presentation framework, CaDiff can be adapted to any RL algorithm.}

% {\color{black} \textbf{COMMENT:} Please strengthened the blue sentence by highlighting the effect of the CaDiff.}

% \begin{figure*}[!htp]
%     \centering
%     \includegraphics[width=1\linewidth]{./figure/algorithm.pdf}
%     \caption{Overview diagram of the proposed CaDiff, including dynamics estimating under asynchronous diffusion model and causal state representation under bisimulation.}
%     \label{fig:algorithm}
% \end{figure*} 

\begin{algorithm}[t]
    \caption{CaDiff}\label{algorithm-main}
    % \LinesNumbered
    \small 
    \begin{algorithmic}[1]
        \STATE Initialize: Discount factor $\gamma$, forward step $K$, noise intensity $\delta$, observation-denoising and reward-denoising models $\theta$ and $\phi$, bisimulation model $\zeta$, and replay memory~$\mathcal{D}$;
        \FOR{Epoch $t=1, \dots, T$}
        \STATE Compute the (approximate) denoised causal state $\hat{\mathbf{s}}_t$ from $\mathbf{o}_t$ using 
        % observation denoise model 
        $\theta$ and 
        % bisimulation model 
        $\zeta$;
        \STATE Select action $\mathbf{a}_t\sim \pi(\hat{\mathbf{s}}_t)$, and obtain $r_{t+1}$ and $\mathbf{o}_{t+1}$;
        \STATE Store transition $\left(\mathbf{o}_t, \mathbf{a}_t, r_{t+1}, \mathbf{o}_{t+1}\right)$ in $\mathcal{D}$;
        \STATE Sample a batch of transitions $\mathcal{B}$ randomly from $\mathcal{D}$;
        \STATE Obtain 
        % states 
        $\hat{\mathbf{s}}_{t}$ and $\hat{\mathbf{s}}_{t+1}$ from 
        % observations 
        $\mathbf{o}_{t}$ and $\mathbf{o}_{t+1}$ in $\mathcal{B}$, respectively;
        \STATE Take gradient descent on $\hat{\mathcal{L}}_{\text {State}}(\theta)+\hat{\mathcal{L}}_{\text {BS}}(\zeta)$;
        \STATE Take gradient descent on $\hat{\mathcal{L}}_{\text {Rew}}(\phi)+\hat{\mathcal{L}}_{\text {BR}}(\zeta)$;
        \ENDFOR
        % \OUTPUT Policy $\pi$
    \end{algorithmic}
\end{algorithm}

\subsection{Asynchronous Diffusion Model}
The objective of ADM is to derive $P\left(\hat{\mathbf{s}}_{t+1}\mid \hat{\mathbf{s}}_{t}, \mathbf{a}_t\right)$ and $P\left(\hat{r}_{t+1}\mid \hat{\mathbf{s}}_{t}, \mathbf{a}_t\right)$ from $(\mathbf{o}_t, \mathbf{a}_t, r_{t+1}, \mathbf{o}_{t+1})$, where 
% $\hat{\mathbf{s}}_{t}$ 
% and 
$\hat{\mathbf{s}}_{t+1}$ 
% are 
is the causal state estimated under denoised observations, and $\hat{r}_{t+1}$ is the denoised reward at time $t+1$.
Existing DM-based RL algorithms typically use $\mathbf{o}_{t+1}$ and $r_{t+1}$ as input data~\cite{DM_offline_3}; the distribution fitted by DM is affected by the input data's noise. 
% The existing diffusion model can struggle to denoise the intrinsic noise of the data, such as the noisy partial observations of perturbed POMDPs.
% A critical gap exists in the design of efficient diffusion model-based denoising algorithms for noise-perturbed data, specifically the absence of methods that explicitly address varying noise intensities during training.

% To address the differentiation of noise intensity, 
For conciseness, $t$ and $k\in\mathbb{N}$ indicate the RL iteration and DM's step with total forward step $K$, respectively.
ADM denoises the observations and rewards of P$^2$OMDPs, and estimates environmental dynamics by assuming that $\mathbf{o}_{t+1}$ and $r_{t+1}$ are superimposed by $\delta$-step Gaussian noise. 
To obtain the denoised causal state $\hat{\mathbf{s}}_{t+1}$, we input $r_{t+1}$ and $\tilde{\mathbf{s}}_{t+1}$ to ADM, along with $\hat{\mathbf{s}}_{t}$ and $\mathbf{o}_{t}$, where $\tilde{\mathbf{s}}_{t+1}$ is the noised causal state. 
When ADM accurately predicts the future causal state, $\tilde{\mathbf{s}}_t$ serves as a sufficient statistic for the latent variables. 
% {\color{black}corrupted by noise and obtained through the bisimulation learning network.}
$\mathbf{x}^{\delta}_{t+1}$ denotes the inputs under $\delta$-step Gaussian noises.
For simplification, we omit $t$. The input $\mathbf{x}^\delta$ corresponds to the results after a $\delta$-step forward process in the DM, as follows.
\begin{definition}\label{assump:distri}
    For a P$^2$OMDP, the sampled distribution $P_{\mathrm{in}}$ is the result of the noiseless distribution $P(\mathbf{x}_{\mathrm{tr}}|\hat{\mathbf{s}}_{t}, \mathbf{a}_t)$ after $\delta$ steps of the forward process, i.e.,
    \begin{align}
        P_{\mathrm{in}}(\mathbf{x}^{\delta}|\hat{\mathbf{s}}_{t}, \mathbf{a}_t)=&\int_{\mathbb{R}^{w_x}}P(\mathbf{x}_{\mathrm{tr}}|\hat{\mathbf{s}}_{t}, \mathbf{a}_t){\sigma_\delta^{w_x}(2\pi)^{w_x/2}}\\
        &\times\exp\left(-{\left\|\sqrt{\alpha_\delta}\mathbf{x}_{\mathrm{tr}}-\mathbf{x}^{\delta}\right\|^2}/{(2\sigma_\delta^2)}\right) d\mathbf{x}_{\mathrm{tr}},\notag
    \end{align}
    where $\mathbf{x}_{\mathrm{tr}}$ is the noiseless input, $w_x$ is the dimension of input data $\mathbf{x}^{\delta}$, {\color{black}$\alpha_\delta = 1 - \sigma^2_\delta$, and $\sigma^2_1, \ldots, \sigma^2_K > 0$ forming a predefined variance schedule}.
\end{definition}
Given the input conditional distribution $P(\mathbf{x}^{\delta}|\hat{\mathbf{s}}_{t}, \mathbf{a}_t)$, we seek for its denoised version $P(\mathbf{x}^{0}|\hat{\mathbf{s}}_{t}, \mathbf{a}_t)$. ADM adds Gaussian noise progressively through a forward diffusion process, i.e., a forward Ornstein–Uhlenbeck (OU) process as:
\begin{subequations}\label{eq:forward}
\begin{align}
    d\mathbf{x}^{k} &= -\frac{1}{2} \mathbf{x}^k d k + d \mathbf{w}^k\quad \forall k\ge \delta\notag\\
    &\qquad\text{with} \quad \mathbf{x}^\delta \sim P(\mathbf{x}^{\delta}|\hat{\mathbf{s}}_{t}, \mathbf{a}_t);\label{eq:forward1}\\
    d\mathbf{x}^{k} &= -\frac{1}{2} \mathbf{x}^k d k + d \mathbf{w}^k \quad \forall k\ge 0\notag\\
    &\qquad\text{with} \quad \mathbf{x}^{0}\coloneqq\hat{\mathbf{x}}^{0}= \left(\mathbf{x}^{\delta} - \sqrt{1 - \bar{\alpha}_\delta}\epsilon\right)/\sqrt{\bar{\alpha}_\delta},\label{eq:forward2}
\end{align}
\end{subequations}
where $\mathbf{w}^k$ denotes a Wiener process, {\color{black}$\bar{\alpha}_\delta = \prod_{j=0}^\delta \alpha_j$ with $\alpha_j = 1 - \sigma^2_j$,} and the noise $\epsilon$ is drawn from the normal distribution. 

% We suppress the RL iteration $t$ without ambiguity.
% We suppress the index of RL iterations, $t$, whenever the context is clear.

% {\color{red} $P$ stands for probability or not? $P$ is also probability. Be consistent!}

Without a step limit, $\mathbf{x}^\infty$ converges to a Gaussian distribution. At each step $k$, $P(\mathbf{x}^k|\hat{\mathbf{s}}_{t}, \mathbf{a}_t)$ is the conditional distribution of the intermediate variable $\mathbf{x}^k$ generated by the forward process, given the denoised causal states and actions. 
\eqref{eq:forward1} is a forward process that starts at step $\delta$ from the $\delta$-step perturbed input $\mathbf{x}^{\delta}$ and runs $(k-\delta)$ additional steps to reach $\mathbf{x}^{k}$ with $k\ge\delta$. \eqref{eq:forward2} refreshes the process with the denoised input $\hat{\mathbf{x}}^{0}$ at step $0$ and  applies the same schedule for $k$ steps to obtain~${\mathbf{x}}^{k}$.

The forward process of ADM terminates at a large enough step $K$. The reverse process generates samples by reversing results per step in \eqref{eq:forward} as $d \overline{\mathbf{x}}^k \!=\! \left[ \frac{1}{2} \overline{\mathbf{x}}^k \!+\! \nabla \log P(\overline{\mathbf{x}}^k|\hat{\mathbf{s}}_{t}, \mathbf{a}_t)\right] d k \!+\! d \overline{\mathbf{w}}^k$,
where $\overline{\mathbf{x}}^0 \sim P(\mathbf{x}^K|\hat{\mathbf{s}}_{t}, \mathbf{a}_t)$; $\overline{\mathbf{w}}^k$ and $\overline{\mathbf{x}}^k$ are the time-reversed Wiener and reverse processes, respectively; $\nabla \log P(\overline{\mathbf{x}}^k|\hat{\mathbf{s}}_{t}, \mathbf{a}_t)$ is the unknown conditional score function and estimated utilizing conditional score networks. 

{\color{black}In line with classifier-free guidance, a popular approach for conditional DMs \cite{DM_2}, we define a mask variable $\tau \in \{\varnothing, {\mathsf{id}}\}$ with equal probability, which determines whether the guidance is used or ignored. 
{\color{black}In this paper, we define the guidance as $\mathbf{y}=(\hat{\mathbf{s}}_{t}, \mathbf{a}_t)$.}
ADM performs conditional denoising when $\tau = {\mathsf{id}}$; the guidance is omitted for the unconditional case when $\tau = \varnothing$.
We further define two score-matching errors for reconstructing denoised data from P$^2$OMDP, i.e., the conditional score error $\mathcal{R}(\varphi) = \frac{1}{2}\int_{k_0}^K\frac{1}{K-k_0} \E\nolimits_{\mathbf{x}^k,\mathbf{y}}\left\|{\varphi(\mathbf{x}^k,\mathbf{y},k)-\nabla \log P(\mathbf{x}^k|\mathbf{y})}\right\|_2^2 d k +\frac{1}{2}\int_{\delta}^K\frac{1}{K\!-\delta} \E\nolimits_{\mathbf{x}^k,\mathbf{y}}\left\|{\varphi(\mathbf{x}^k,\mathbf{y},k)-\nabla \log P(\mathbf{x}^k|\mathbf{y})}\right\|_2^2 d k$ for $\tau = {\mathsf{id}}$ (under $\tau\mathbf{y}=\mathbf{y}$) and the unconditional score error $\mathcal{R}_0(\varphi)$ for $\tau = \varnothing$ by replacing the guidance $\mathbf{y}$ with $\varnothing$ in $\mathcal{R}(\varphi)$ (under $\tau\mathbf{y}=\varnothing$).
Here, $\varphi$ denotes the conditional score network.
Unifying the two cases, we design the overall loss:
\begin{align}
&\mathcal{R}_{\star}(\varphi)=  \mathcal{R}(\varphi)+\mathcal{R}_0(\varphi)\label{equ::popu loss score}\\
=&\int_{k_0}^K\!\!\frac{1}{K\!-\!k_0} \E\limits_{\mathbf{x}^k,\mathbf{y},\tau}\left\|{\varphi(\mathbf{x}^k,\tau\mathbf{y},k)\!-\!\nabla \log P(\mathbf{x}^k|\tau\mathbf{y})}\right\|_2^2 d k \nonumber\\
&+\!\int_{\delta}^K\!\frac{1}{K\!-\!\delta} \E\limits_{\mathbf{x}^k,\mathbf{y},\tau}\left\|{\varphi(\mathbf{x}^k,\tau\mathbf{y},k)\!-\!\nabla \log P(\mathbf{x}^k|\tau\mathbf{y})}\right\|_2^2 d k\nonumber
\end{align}
% In other words, by replacing $\tau\mathbf{y}=\mathbf{y}$ and $\tau\mathbf{y}=\varnothing$, we obtain $\mathcal{R}(\varphi)$ and $\mathcal{R}_0(\varphi)$.
% (\ref{equ::popu loss score})
which implies that $\mathcal{R}(\varphi) \le 2\mathcal{R}_{\star}(\varphi)$ with $\varphi(\mathbf{x}^k,\tau\mathbf{y},k)$ estimating $\nabla \log P(\mathbf{x}^k|\hat{\mathbf{s}}_t, \mathbf{a}_t)$, which is used in the proof of Theorem~\ref{thm:valboundmodelerror} (see Appendix~\ref{appendix:proof_Rs}).}
 $\hat{\mathbf{x}}^0 = \frac{1}{\sqrt{\bar{\alpha}_\delta}}\left(\mathbf{x}^{\delta} - \sqrt{1 - \bar{\alpha}_\delta}\epsilon\right)$ is the predicted denoised input at step $\delta$. $P(\mathbf{x}^k | \mathbf{x})$ is the Gaussian transition kernel for the forward process of ADM. 
% The gradient $\nabla \log P(\mathbf{x}^k| \mathbf{x}^0)$ has a closed form, i.e., $\nabla \log P(\mathbf{x}^k| \mathbf{x}^0) = - (\mathbf{x}^k - \sqrt{\alpha_k} \mathbf{x}^0) / \sigma_k^2$.

To prevent blow-up of score functions, we follow \cite{nichol2021improved}, introduce an early-stopping step $k_0$, and write $\mathcal{R}_{\star}$ equivalently: 
% {\color{black}Practically, we minimize an equivalent form of $\mathcal{R}_{\star}$ as}
\begin{align}\label{equ::popu loss}
&\ell(\varphi)\coloneqq\\
&\int_{k_0}^{K}\!\!\!\!\frac{1}{K\!-\!k_0}\!\E\limits_{\hat{\mathbf{x}}^0,\mathbf{y}}\!\bigg[\!\E\limits_{\tau,\mathbf{x}^k|\hat{\mathbf{x}}^0}\!\bigg[\!\left\|{\varphi(\mathbf{x}^k\!,\tau\mathbf{y}\!,k)\!-\!\nabla \log P(\mathbf{x}^k|\hat{\mathbf{x}}^0)}\right\|_2^2\!\bigg]\!\bigg]\!d k\nonumber\\
&+\!\!\!\int_{\delta}^{K}\!\!\!\!\!\!\frac{1}{K\!-\!\delta}\!\E\limits_{{\mathbf{x}}^\delta,\mathbf{y}}\!\bigg[\!\E\limits_{\tau,\mathbf{x}^k|{\mathbf{x}}^\delta}\!\bigg[\!\left\|{\varphi(\mathbf{x}^k\!,\tau\mathbf{y}\!,k)\!-\!\nabla \log P(\mathbf{x}^k|{\mathbf{x}}^\delta)}\right\|_2^2\!\bigg]\!\bigg]\!d k.\notag
\end{align}
% where $\hat{\mathbf{x}}^0 = \frac{1}{\sqrt{\bar{\alpha}_\delta}}\left(\mathbf{x}^{\delta} - \sqrt{1 - \bar{\alpha}_\delta}\epsilon\right)$ and $\epsilon$ follows a standard normal distribution.
According to \cite[Lemma C.3]{vincent2011connection}, (\ref{equ::popu loss score}) differs from (\ref{equ::popu loss}) by a constant independent of $\mathbf{x}^k$. ADM is optimized over a mini-batch $\mathcal{B}$ with $|\mathcal{B}|=n$ by minimizing the empirical loss:
\begin{align}\label{equ::empirical loss}
    \hat{\ell}(\varphi)=\sum\nolimits_{t\in\mathcal{B}}\ell(\mathbf{x}_t,\mathbf{y}_t; \varphi)/n.
\end{align}
Here, $\ell(\mathbf{x}_t,\mathbf{y}_t;\varphi)$ represents the loss computed on the training pair $(\mathbf{x}_t,\mathbf{y}_t)$, replacing the expectation on $({\hat{\mathbf{x}}^0,\mathbf{y}})$ in (\ref{equ::popu loss}).
% where
% \begin{align}\label{equ::single empirical loss recap}
%     &\ell(\mathbf{x},\mathbf{y}; \varphi)\coloneqq\\
%     &\int_{k_0}^{K}\frac{1}{K-k_0}\E\limits_{\tau,\mathbf{x}^k|\hat{\mathbf{x}}^0}\big[\left\|{\varphi(\mathbf{x}^k,\tau\mathbf{y},k)-\nabla \log P(\mathbf{x}^k|\hat{\mathbf{x}}^0)}\right\|_2^2\big]d k\nonumber\\
%     &+\int_{\delta}^{K}\!\frac{1}{K-\delta}\E\limits_{\tau,\mathbf{x}^k|{\mathbf{x}}^\delta}\big[\left\|{\varphi(\mathbf{x}^k,\tau\mathbf{y},k)\!-\!\nabla \!\log P(\mathbf{x}^k|{\mathbf{x}}^\delta)}\right\|_2^2\big]d k.\notag
% \end{align}

{\color{black}Substituting the generic denoising model $\varphi$ and its input $\mathbf{x}^k_{t+1}$ in (\ref{equ::empirical loss}) with the observation-denoising and reward-denoising models $\theta$ and $\phi$ (associated with inputs $\tilde{\mathbf{s}}_{t+1}$ and $r_{t+1}$, respectively), and recalling that $\mathbf{y} = (\hat{\mathbf{s}}_{t}, \mathbf{a}_t)$, we derive the following objectives for state and reward estimation:
\begin{align}
    \hat{\mathcal{L}}_{\text {State}}(\theta) \!&=\! \sum\nolimits_{t\in\mathcal{B}}\ell(\tilde{\mathbf{s}}_{t+1}, \hat{\mathbf{s}}_{t+1}, \!\mathbf{a}_t; \theta)/n;\\
    \hat{\mathcal{L}}_{\text {Rew}}(\phi) \!&=\! \sum\nolimits_{t\in\mathcal{B}}\ell(r_{t+1}, \hat{\mathbf{s}}_{t+1}, \mathbf{a}_t; \phi)/n.
\end{align}}

% The denoising procedure for rewards follows analogously and is omitted for brevity.

% {\color{black} \textbf{COMMENT:} Please put Eq. (8) at the top of this page.}

\subsection{Bisimulation Metric}
We extend the concept of bisimulation to POMDPs to achieve effective CSR, i.e., estimating $P\left(\hat{\mathbf{s}}_{t}\mid \mathbf{o}_{t}\right)$.
Based on the Wasserstein metric, a new bisimulation metric is defined:
\begin{definition}
% [Bisimulation metric]
\label{def:bisimilarity}
    Given constants $C_{\mathrm{r}}, C_{\mathrm{s}}\in\left(0,1\right)$, for any pair of causal state and observation $\left\{\hat{\mathbf{s}}_t\in\mathcal{S}_{\mathrm{c}}, \mathbf{o}_t\in\mathcal{O}\right\}$ of a P$^2$OMDP, the bisimulation metric is defined as
    \begin{align}\label{eq:p-Wass-bisim-metric}
        d\!\left(\hat{\mathbf{s}}_t,\! \mathbf{o}_{t}\right)
        \!&=\max\nolimits_{\mathbf{a}\in\mathcal{A}}(C_{\mathrm{r}}W_p(d)\!\left(P\!\left(r_{t+1}\!\mid\! \hat{\mathbf{s}}_{t},\! \mathbf{a}\right), P\left(r_{t+1}\!\mid\! \mathbf{o}_{t},\! \mathbf{a}\right)\right)\!\notag\\
        &+\!C_{\mathrm{s}} W_p(d)\!\left(P\left(\hat{\mathbf{s}}_{t+1}\!\mid\! \hat{\mathbf{s}}_{t},\! \mathbf{a}\right), P\left(\hat{\mathbf{s}}_{t+1}\!\mid\! \mathbf{o}_{t},\! \mathbf{a}\right)\right)).
    \end{align}
% where $W_p(d)$ denotes the Wasserstein distance between probability distributions.

\end{definition}
A distance of zero indicates bisimilarity.
Combining \eqref{eq:op_origin_observ} with $\hat{F}(\cdot)$ in Definition~\ref{def:bisimulation}, it follows that the causal state learning objective is equivalent to inferring the conditional distribution $P(\tilde{\mathbf{s}}_{t}\mid {\mathbf{o}}_t)$. We employ a recurrent neural network (RNN) $\zeta({\mathbf{o}}_t)$ to approximate the noised causal state, i.e., fitting $P\left(\tilde{\mathbf{s}}_{t}\mid {\mathbf{o}}_t\right)$, which is utilized by ADM to obtain denoised causal state $\hat{\mathbf{s}}_{t}$.
{\color{black}Besides, we here can obtain $\tilde{\mathbf{s}}_{t+1}$ introduced above as $\tilde{\mathbf{s}}_{t+1}=\zeta(\mathbf{o}_{t+1})$.}
% , i.e., $\hat{\mathbf{s}}_{t}=\theta(\zeta(\mathbf{o}_{t}))$. 
% Recurrent neural networks (RNNs) are unifilar hidden Markov models with continuous states, where the transitions and output probabilities are parameterized by differentiable functions \cite{RNN}. 
Given Definition~\ref{def:bisimilarity}, we can empirically estimate the CSR by minimizing 
\begin{subequations}
    \begin{align}
        \hat{\mathcal{L}}_{\text {BS}}(\zeta) &= \frac{1}{2} \E\left[W_d\left(P\left(\hat{\mathbf{s}}_{t+1}\mid\hat{\mathbf{s}}_{t}, \mathbf{a}_{t}\right), \theta\left(\zeta\left(\mathbf{o}_{t}\right), \mathbf{a}_{t}\right)\right)\right];\notag\\
        \hat{\mathcal{L}}_{\text {BR}}(\zeta) &= \frac{1}{2} \E\left[W_d\left(P\left(r_{t+1}\mid \hat{\mathbf{s}}_{t},\! \mathbf{a}_{t}\right), \phi\left(\zeta\left(\mathbf{o}_{t}\right), \mathbf{a}_{t}\right)\right)\right].\notag
    \end{align}
\end{subequations}
%\limits_{\left\{\mathbf{o}_t, \mathbf{a}_t, r_{t+1}, \mathbf{o}_{t+1}\right\}}
Thus, we implement CSR and assist RL decisions, by iteratively optimizing $\hat{\mathcal{L}}_{\text {State}}(\theta)+\hat{\mathcal{L}}_{\text {BS}}(\zeta)$ and $\hat{\mathcal{L}}_{\text {Rew}}(\phi)+\hat{\mathcal{L}}_{\text {BR}}(\zeta)$.

\section{Theoretical Guarantee of CaDiff}\label{sec:thm}
This section provides a theoretical guarantee for CaDiff with the value function used to measure the discrepancy between observations and their causal states. {\color{black}Following \cite{DM_offline_9}, we begin by imposing a mild light-tail assumption on the initial conditional data distribution as follows.}
% , i.e., 
% % The noise intensity is expressed in $\delta$ as mentioned in Section~3.
% % We mathematically reformulate the assumption considering the noise intensity of the input data, as follows:
% % We introduce 
% a mild tail condition on the initial conditional data distribution, 
% % which pertains solely to the regularity of the original data distribution and does not pose constraints on the resulting conditional score function. In other words, we assume 
% and an additional bounded H\"older norm condition (Definition~\ref{appendix:Holder_norm}) on true data distributions.
% , as follows: [0,1]^{w_y}
\begin{assumption}\label{assump:distri_true}
    For a fixed radius $B$ and a P$^2$OMDP $\mathcal{M}$, define the $b$-H\"{o}lder function $f \in \mathcal{H}^{b}(\mathbb{R}^{w_x}\times \mathbb{R}^{w_y}, B)$ with $w_y = w_s+w_a$. For $C$, $C_1>0$, we assume $f(\mathbf{x}_{\mathrm{tr}}, \hat{\mathbf{s}}_{t}, \mathbf{a}_t) \geq C$, $\forall (\mathbf{x}_{\mathrm{tr}}, \hat{\mathbf{s}}_{t}, \mathbf{a}_t)$, and the true conditional density function yields $P(\mathbf{x}_{\mathrm{tr}} | \hat{\mathbf{s}}_{t}, \mathbf{a}_t) = \exp(-C_1\left\|\mathbf{x}_{\mathrm{tr}}\right\|^2_2/2) \cdot f(\mathbf{x}_{\mathrm{tr}}, \hat{\mathbf{s}}_{t}, \mathbf{a}_t)$.
\end{assumption}

As provable tightness implies theoretical guarantees in VFA, the key to bisimulation metrics is their connection to value functions. 
To generalize the VFA bound, we assume the existence and uniqueness of the $p$-Wasserstein bisimulation metric for any state pair to measure their similarity, as follows.
\begin{assumption}[$p$-Wasserstein bisimulation metric]\label{assum:p-wass}
For P$^2$OMDP $\mathcal{M}$, any given $C_{\mathrm{r}}, C_{\mathrm{s}} \in (0, 1)$, $C_{\mathrm{r}}+C_{\mathrm{s}}< 1$, any $(\hat{\mathbf{s}}_i, \hat{\mathbf{s}}_j) \in \mathcal{S}_{\mathrm{c}} \times \mathcal{S}_{\mathrm{c}}$, and $p \geq 1$, we assume that the bisimulation metric $d\left(\hat{\mathbf{s}}_i, \hat{\mathbf{s}}_j\right)$ in (\ref{eq:bisimulation}) exists and is unique:
\begin{align}
    &d(\hat{\mathbf{s}}_i,\hat{\mathbf{s}}_j) = \max_{\mathbf{a}\in\mathcal{A}}(C_{\mathrm{r}}W_p(d)(P(r_{i+1}\mid \hat{\mathbf{s}}_{i},\mathbf{a}), P(r_{i+1}\mid \hat{\mathbf{s}}_{j},\mathbf{a}))\notag\\
    &\qquad\quad+C_{\mathrm{s}}W_p(d)(P(\hat{\mathbf{s}}_{i+1}\mid \hat{\mathbf{s}}_{i},\mathbf{a}), P(\hat{\mathbf{s}}_{j+1}\mid \hat{\mathbf{s}}_{j},\mathbf{a}))).\label{eq:bisimulation}
\end{align}
\end{assumption}
% % Consider $p=1$ for our analysis. 
% We validate Assumption~\ref{assum:p-wass} for specific cases in Remark~\ref{remark:rem1}, with more general cases to be proved in the future.
\begin{remark}
\label{remark:rem1}
    If both policy and environment are deterministic or $p=1$, Assumption~\ref{assum:p-wass} holds.
\end{remark}
\begin{proof}
    See Appendix~\ref{appendix:remark}.
\end{proof}

% Notably, Assumption~\ref{assum:p-wass} does not restrict the state, action, or observation spaces to be finite (or any other conditions).
To verify the denoised causal state approximation, we analyze CaDiff
% under the perturbed POMDPs
% The analysis comprises 
in the following four steps, by establishing the upper bound of VFA under the ADM approximation error: 
% including (i) establishing the upper bound of VFA for causal states overlooking observations; (ii) refining the upper bound to the observations and causal states under any model approximations; (iii) analyzing the model approximation under a specific model, i.e., the asynchronous diffusion model; and (iv) combining the results in (ii) and (iii) and deriving the upper bound of VFA under ADM. 
% {\color{red} Why? How? What?}

\textit{Step 1: $p$-Wasserstein value difference bound for any state pair:}
Like the bounds in \cite{castro2020scalable, Ferns2011Bisimulation} for policy-independent bisimulation metrics, the bisimulation metric can be bounded: $|V^\pi(\hat{\mathbf{s}}_i) - V^\pi(\hat{\mathbf{s}}_j)| \leq d(\hat{\mathbf{s}}_i, \hat{\mathbf{s}}_j)$ with $d(\hat{\mathbf{s}}_i, \hat{\mathbf{s}}_j)$ defined in Assumption~\ref{assum:p-wass}, where $V^\pi(\hat{\mathbf{s}})=\mathbb{E}_{\pi}[\sum_{i=0}^\infty \gamma^t r_{t+i+1}|{\mathbf{s}}_t=\hat{\mathbf{s}}]$. 
% With the proof in Appendix~\ref{appendix:proofs_state}, we establish the value difference bound.
\begin{theorem}[Value difference bound for $p$-Wasserstein distance]
\label{thm:generalized-ct}
For the bisimulation metric defined in (\ref{eq:bisimulation}), for any $p \geq 1$, $C_{\mathrm{s}}\in[\gamma, 1)$, $C_{\mathrm{r}} \in (0, 1)$, and $C_{\mathrm{r}}+C_{\mathrm{s}}<1$, the bisimulation distance between two states provides the bound on the discrepancy in the VFA:
\begin{align}\label{eq:generalized-ct}
    C_{\mathrm{r}}|V^\pi(\hat{\mathbf{s}}_i) \!-\! V^\pi(\hat{\mathbf{s}}_j)| \!\leq\! d(\hat{\mathbf{s}}_i, \hat{\mathbf{s}}_j), ~\forall(\hat{\mathbf{s}}_i, \hat{\mathbf{s}}_j) \in \mathcal{S}_{\mathrm{c}} \!\!\times\! \mathcal{S}_{\mathrm{c}}.
\end{align}
\end{theorem}
\begin{proof}
    See Appendix~\ref{appendix:proofs_state}.
\end{proof}
As revealed in \eqref{eq:generalized-ct}, the bisimulation metric in (\ref{eq:bisimulation}) reflects the upper bound of the value gap.

\textit{Step 2: Value difference bound for observation/state pairs:}
% Recall that the definitions of reward function and transition function are independent of $\zeta$ and $\theta$.
The influence of the model errors on the VFA with the optimal policy-dependent bisimulation distance is as follows.
% in Theorem~\ref{thm:valboundmodelerror}, with the proof in Appendix~\ref{appendix:theorem:valboundmodelerror}.
\begin{theorem}[Value difference bound with model errors]
\label{thm:valboundmodelerror}
For P$^2$OMDP $\mathcal{M}$, let $\zeta: \tilde{\mathcal{O}} \rightarrow \tilde{\mathcal{S}}_{\mathrm{c}}$ be a function mapping observations to noised causal states such that $ \zeta({\mathbf{o}}_i) = \zeta({\mathbf{o}}_j)$ is equivalent to $\widehat{d}_{\zeta}(\tilde{\mathbf{s}}_i, \tilde{\mathbf{s}}_j) = \left\|\zeta({\mathbf{o}}_i)-\zeta({\mathbf{o}}_j)\right\|_q \leq 2\widehat{\epsilon}$ for $\tilde{\mathbf{s}}_i=\zeta({\mathbf{o}}_i)$ and $\tilde{\mathbf{s}}_j=\zeta({\mathbf{o}}_j)$. For $C_{\mathrm{r}} \in (0, 1)$, $C_{\mathrm{s}}\in [\gamma, 1)$, $C_{\mathrm{r}}+C_{\mathrm{s}}<1$, and $p=1$, then: $\forall \mathbf{s}_{\mathrm{c}} \in \mathcal{S}_{\mathrm{c}}$,
\begin{align}
    | V^\pi(\mathbf{s}_{\mathrm{c}}) - V^\pi( \hat{F}( {\mathbf{o}}) ) |
    \leq& ( 2 \widehat{\epsilon} \!+\! \mathcal{E}_\zeta +{2C_{\mathrm{r}}\mathcal{E}_\phi}/{(1 \!-\! C_{\mathrm{s}}\!-\!C_{\mathrm{r}})} \notag\\
    &+{2C_{\mathrm{s}}\mathcal{E}_{\theta}}/{(1 \!-\! C_{\mathrm{s}}\!-\!C_{\mathrm{r}})} )/C_{\mathrm{r}}(1 \!-\! \gamma),\notag
\end{align}
where 
$ \mathcal{E}_\zeta \coloneqq \|\widehat{d}_{\zeta} - \widehat{d}\|_\infty $ 
    is the bisimulation metric learning error,
$ \mathcal{E}_\phi \coloneqq W_1(d)(P(r\mid \mathbf{s}_{\mathrm{c}}, \mathbf{a}), P(r\mid \hat{F}( {\mathbf{o}}), \mathbf{a})) $ 
    is the reward approximation error, and 
    $ \mathcal{E}_{\theta} \coloneqq W_1(d)(P(\mathbf{s}'\mid \mathbf{s}_{\mathrm{c}}, \mathbf{a}), P(\mathbf{s}'\mid \hat{F}( {\mathbf{o}}), \mathbf{a})) $
    is the state transition model error.
    $\widehat{\epsilon}$ denotes the aggregation radius in $\zeta$-space. $\widehat{d}(\tilde{\mathbf{s}}_i, \tilde{\mathbf{s}}j)$ denotes the bisimulation metric computed under the estimated environment dynamics, while $\widehat{d}_{\zeta}(\tilde{\mathbf{s}}_i, \tilde{\mathbf{s}}_j)$ represents its approximation derived from $\zeta$. 
\end{theorem}
\begin{proof}
    See Appendix~\ref{appendix:theorem:valboundmodelerror}.
\end{proof}
By Theorem~\ref{thm:valboundmodelerror}, we can quantify the upper bound of the value gap under arbitrary model errors. This can be extended to different probability density estimation models to establish specific convergence properties.

\textit{Step 3: Distribution estimation under ADM:}
Since $\mathcal{E}_{\theta}$ and $\mathcal{E}_{\phi}$ are based on the same ADM architecture, we define the approximation error of the conditional probability as $\varphi(\mathbf{x}^{k},\hat{\mathbf{s}}_{t}, \mathbf{a}_t,k)$, where $\mathbf{x}^{k}$ can be replaced by  $\hat{\mathbf{s}}_{t+1}$ or $r_{t+1}$.
Under Assumption~\ref{assump:distri_true}, we  measure the ADM's distribution estimation by considering the initialization error, score estimation error, and discretization error, and provide sample complexity bounds for each of these errors using the Wasserstein-1 distance.
We utilize an approximation theory for estimating the conditional score with ReLU neural networks (NNs).

\begin{theorem}[Approximation error by ADM]\label{thm:wass_CSR_ADM}
    Under Assumption \ref{assump:distri_true}, for any given $(\mathbf{s}^{\star}_{\mathrm{c}}, \mathbf{a}^{\star})$, terminal step $K=\frac{2b}{2w_s+w_a+2b}\log n$, and early-stopping step $k_0=n^{-\frac{4b}{2w_s+w_a+2b}-1}$, the estimated error of the conditional probability of noiseless data is given by
    \begin{align*}
        &\E\left[W_1(P(\hat{\mathbf{x}}^0|\hat{\mathbf{s}}_{t}, \mathbf{a}_t), \hat{P}({\mathbf{x}}^{k_0}|\hat{\mathbf{s}}_{t}, \mathbf{a}_t))\right]\\
        =&\mathcal{T}(\mathbf{s}^{\star}_{\mathrm{c}}, \mathbf{a}^{\star})O\left(n^{-\frac{b}{2w_s+w_a+2b}} (\log 
          \!n)^{\max(19/2,(b+2)/2)} \right),
    \end{align*}
    where $b$ is the degree of smoothness in H\"older norm; $w_s$ and $w_a$ represent the dimensions of state and action, respectively; $\mathcal{T}(\mathbf{s}^{\star}_{\mathrm{c}}, \mathbf{a}^{\star})$ is the distribution coefficient.
\end{theorem}
\begin{proof}
    See Appendix~\ref{appendix:proof_wass_CSR_ADM}.
\end{proof}
As $n\rightarrow\infty$, the distribution estimation measured by the Wasserstein-1 distance converges, i.e., $$\mathbb{E}_{\left\{\mathbf{o}_t, \mathbf{a}_t, r_{t+1}, \mathbf{o}_{t+1}\right\}}\!\!\left[W_1(P(\hat{\mathbf{x}}^0|\hat{\mathbf{s}}_{t}, \mathbf{a}_t), \hat{P}({\mathbf{x}}^{k_0}|\hat{\mathbf{s}}_{t}, \mathbf{a}_t))\right] \rightarrow 0,$$ showing the effective distribution estimation of ADM.

\textit{Step 4: Wasserstein value difference bound under ADM:}
The bisimulation metric learning can be achieved, e.g., by an RNN, whose convergence rate of $ \mathcal{E}_\zeta$ is $\mathcal{O}\left(n^{-\frac{2p_R}{2p_R +w_s+1}}(\log n)^6\right)$ with the model size $p_R$ \cite{kohler2023rate}. Substituting the approximation error in Theorem~\ref{thm:wass_CSR_ADM} with the respective error terms, $\mathcal{E}_\phi$ and $\mathcal{E}_\theta$, introduced in Theorems~\ref{thm:valboundmodelerror}, we establish the theoretical guarantee of CaDiff in P$^2$OMDPs, as follows.

\begin{theorem}[Value difference bound with ADM]
\label{thm:bound_CSR_ADM}
Consider the conditions in Theorems~\ref{thm:valboundmodelerror}~and~\ref{thm:wass_CSR_ADM}. Let $c_b=\max\{\frac{19}{2},\frac{b+2}{2}\}$. $\forall \mathbf{s} \in \mathcal{S}$, 
\begin{align}\label{eq:final}
    &\mathbb{E}\! \left[\left| V^\pi(\mathbf{s}_{\mathrm{c}}) \!-\! V^\pi(\hat{F}( {\mathbf{o}})) \right|\right]
    \!\leq \!2\widehat{\epsilon} 
    \notag\\
    &
    + \frac{1}{C_{\mathrm{r}}(1 \!- \!\gamma)} \Big(\mathcal{O}\left(n^{-2p_R/({2p_R +w_s+1})}(\log n)^6\right)\\
    &+\frac{2C_{\mathrm{r}}+2C_{\mathrm{s}}}{1 - C_{\mathrm{s}}-C_{\mathrm{r}}} \mathcal{T}(\mathbf{s}^{\star}_{\mathrm{c}}, \mathbf{a}^{\star})\mathcal{O}\left(n^{-{b}/({2w_s+w_a+2b})} (\log n)^{c_b} \right)\Big).\notag
\end{align}
\end{theorem}

As $n \to \infty$, the approximate value function for the estimated causal state $V^\pi(\hat{F}( {\mathbf{o}}))$ in (\ref{eq:final}) converges to within $2\hat{\epsilon}$-neighborhood of the ground-truth causal state $V^\pi(\mathbf{s}_{\mathrm{c}})$.
% i.e., the neighborhood region of the ground-truth causal state $V^\pi(\mathbf{s})$ with the radius of $\hat{\epsilon}$,
% see Appendix~\ref{appendix:proof_wass_CSR_ADM}.
The parameters $(C_{\mathrm{r}}, C_{\mathrm{s}})$ ensure a trade-off between the reward approximation error and the state transition model error, while $\frac{1}{C_{\mathrm{r}}(1 - \gamma)}$ and $\frac{2C_{\mathrm{r}}+2C_{\mathrm{s}}}{1 - C_{\mathrm{s}}-C_{\mathrm{r}}}$ balance the approximation error of the bisimulation and that of the noisy distribution.

% {\color{black} \textbf{COMMENT:} The above blue sentence is confusing.}

\textbf{Computational cost:}
We evaluate the additional computational cost of CaDiff on top of RL algorithms. 
% It was analyzed in \cite{chen2024accelerating} that
From \cite{chen2024accelerating}, the computational cost of a DM is $\widetilde{\mathcal{O}}(\mathrm{poly} \log d_i)$, where $d_i$ is the dimension of the input data. 
% Considering our definition of noise intensity, 
{\color{black}The loss function of ADM includes two terms to be computed, i.e., \eqref{equ::empirical loss}, thereby doubling the computational cost of a standard DM.} 
The additional computational cost of CSR is $\widetilde{\mathcal{O}}(\mathrm{poly} \log \max\{|\mathcal{A}|, |\mathcal{O}|\})$ in CaDiff,
% . The experiments are provided 
as empirically evaluated in Section~\ref{appendix:cost}.

\section{Experiments}\label{sec:result}
To evaluate CaDiff in classical control tasks with high degrees of freedom (DoFs), we consider Roboschool tasks under standard POMDPs \cite{standard_POMDP}. There are six environments, i.e., \{Hopper, Ant, Walker\}-\{P, V\}, where ``-P" stands for observing positions and angles only, and ``-V" stands for observing velocities only.
In the no-velocities (i.e., ``-P") settings, velocity features are excluded from the raw observations; in the velocities-only (i.e., ``-V") settings, the observations contain only velocity-related information. Details for each environment are summarized in Table~\ref{table:env} with a maximum of 1,000 steps.
% For more information about environments, see Appendix~\ref{appendix:simulation_mujoco}. 
% % {\color{red} To demonstrate the robustness of CaDiff, we train CaDiff with the same hyperparameters for all six tasks.} 
{\color{black}In all experiments, both the rewards and observations are perturbed by zero-mean Gaussian noise, whose variance is scaled by a noise scale to control the perturbation intensity. Unless otherwise specified, the noise intensity is set to $\delta=2$ to evaluate the impact of the noise.}
% The hyperparameters are collated in Appendix~\ref{appendix:simulation_mujoco_para}. 
Since CaDiff can accommodate any RL algorithm, we consider SAC, a typical RL algorithm. 
% We also conduct an ablation study in Appendix~\ref{appendix:ablation}.
We evaluate all experiments with $600,000$ iterations and smooth each return. 
% The code for all experiments will be made publicly available upon publication.
The hyperparameters are summarized in Table~\ref{table:para_all}.

\begin{table}[t]
\centering
\vspace{-0.3cm}
\caption{Information of environments in this paper}
\vspace{-0.3cm}
\label{table:env}
\begin{tabular}{l|c|c}
\hline
Name                 & Dimension of observation space & DoF \\ \hline
RoboschoolAnt        & 28                       & 8   \\
RoboschoolAnt-V      & 11                       & 8   \\
RoboschoolAnt-P      & 17                       & 8   \\
RoboschoolHopper     & 15                       & 3   \\
RoboschoolHopper-V   & 6                        & 3   \\
RoboschoolHopper-P   & 9                        & 3   \\
RoboschoolWalker2d   & 22                       & 6   \\
RoboschoolWalker2d-V & 9                        & 6   \\
RoboschoolWalker2d-P & 13                       & 6   \\ \hline
\end{tabular}
\vspace{-0.2cm}
\end{table}

\begin{table}[t]
\centering
\caption{Hyperparameters for tasks}
\vspace{-0.2cm}
\label{table:para_all}
\begin{tabular}{l|l}
\hline
Description                                               & Value   \\ \hline
 Number of training iterates                               & 600    \\
 Size of replay memory                                 & $10^6$  \\
 Number of samples for each update                     & 64      \\
 Discount factor                                           & 0.99    \\
 Fraction of updating the target network per gradient step & 0.005   \\ 
 Learning rate for the policy and value networks               & 0.0003  \\ 
 Learning rate for the entropy coefficient in SAC          & 0.0003  \\ 
 Target entropy in SAC                                     & 0.2     \\
 Learning rate of the asynchronous diffusion model   & 0.0003  \\ 
 Learning rate of the bisimulation metric learning   & 0.0003  \\ 
 Network for the asynchronous diffusion model        & UNet    \\ 
 Total diffusion step                            & 500     \\ 
 Beta schedule                                   & linear  \\ 
 Noise intensity of observation and reward           & 2       \\ 
 % Variance of Gaussian noise                  & 0.5       \\ 
 \hline
\end{tabular}

\vspace{-0.2cm}
\end{table}

% For each algorithm, we plot the average return of 5 independent trials as the solid curve and plot the standard deviation using the transparent shaded region. 
% By looking at tasks that typically occlude some part of the observation, we adopt the occlusion benchmark proposed by \cite{hanvariational} and 
% Additionally, to illustrate the applicability of CaDiff in practical scenarios, we conduct decision-making for path planning in the industrial Internet of Things (IIoT) scenario. Further details are provided in Appendix~\ref{appendix:simulation_IIoT}.
    
\begin{figure}[ht]
    \centering
    \subfigure[Ant-P]{
        \centering
        \includegraphics[width=0.225\textwidth]{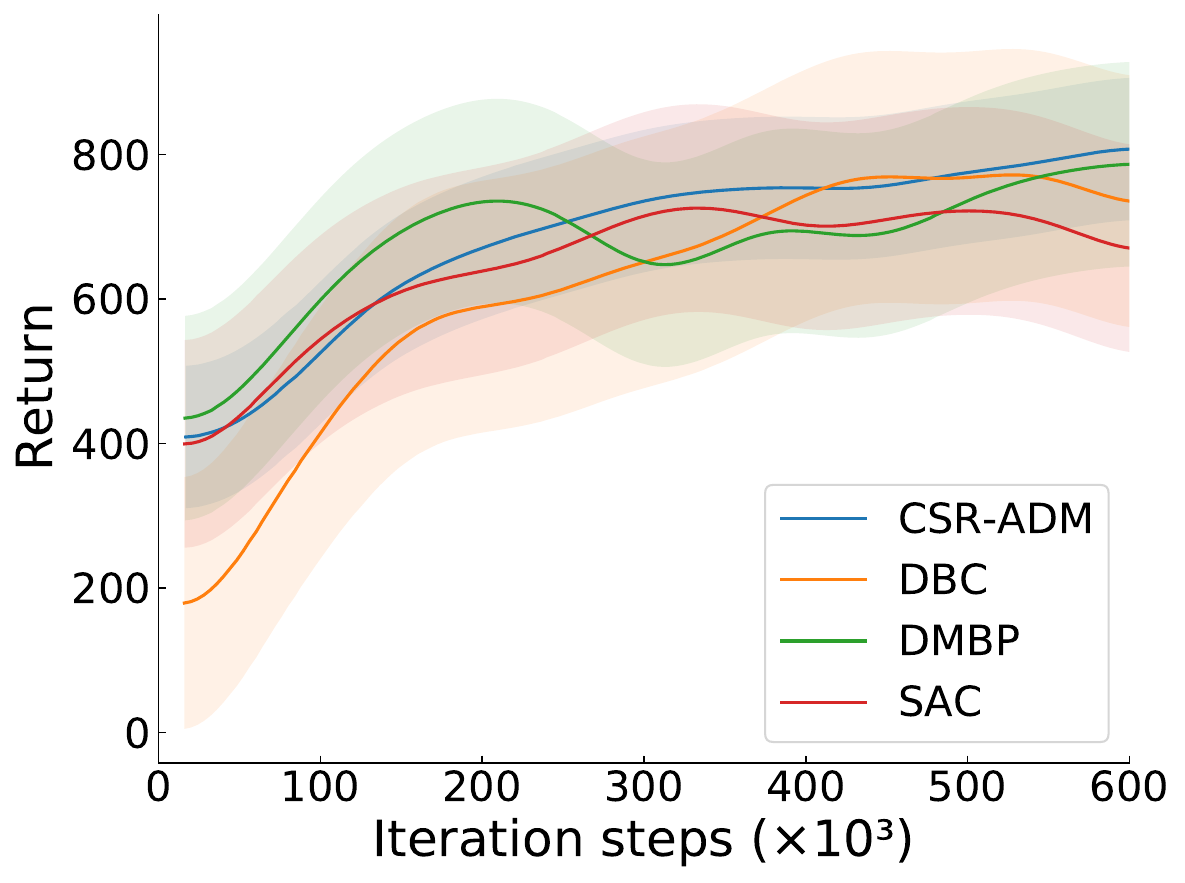}
    }
    \subfigure[Ant-V]{
        \centering
        \includegraphics[width=0.225\textwidth]{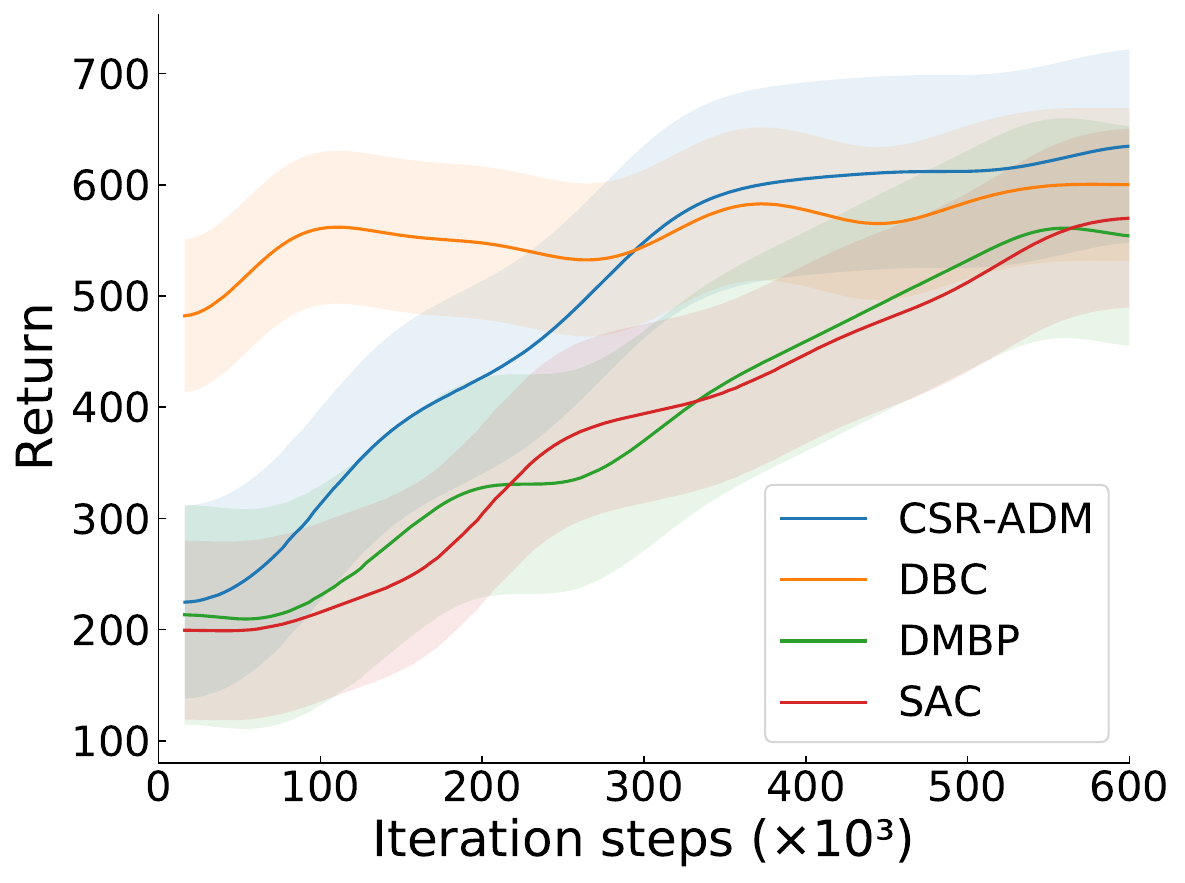}
    }
    \subfigure[Hopper-P]{
        \centering
        \includegraphics[width=0.225\textwidth]{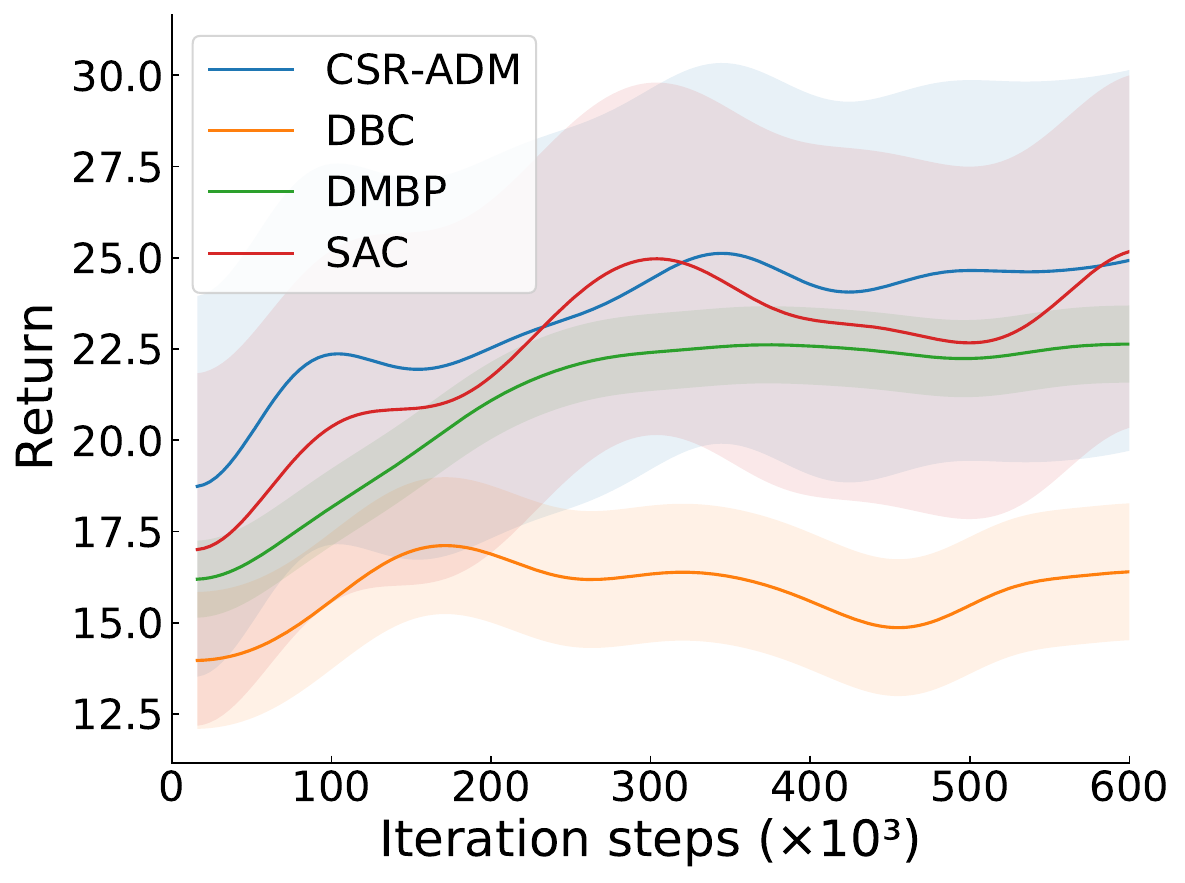}
    }
    \subfigure[Hopper-V]{
        \centering
        \includegraphics[width=0.225\textwidth]{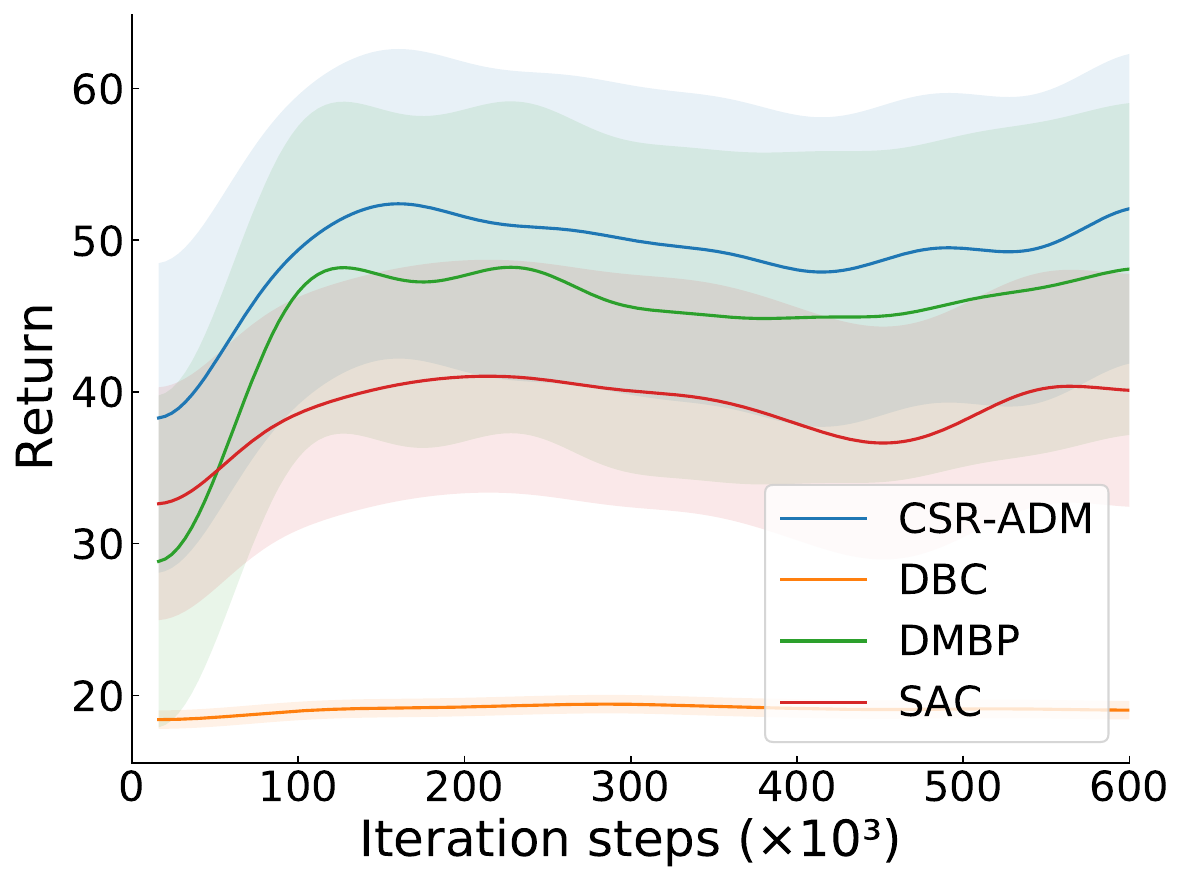}
    }
    \subfigure[Walker-P]{
        \centering
        \includegraphics[width=0.225\textwidth]{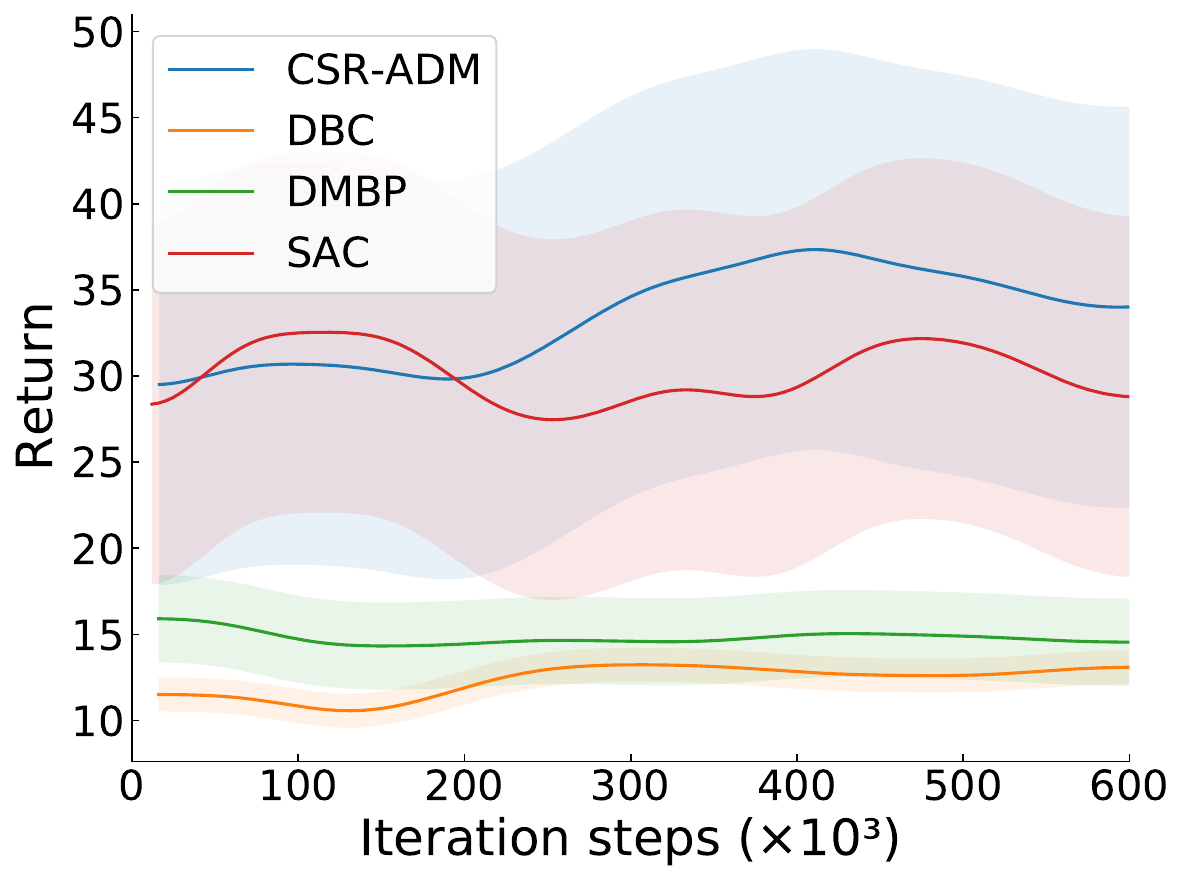}
    }
    \subfigure[Walker-V]{
        \centering
        \includegraphics[width=0.225\textwidth]{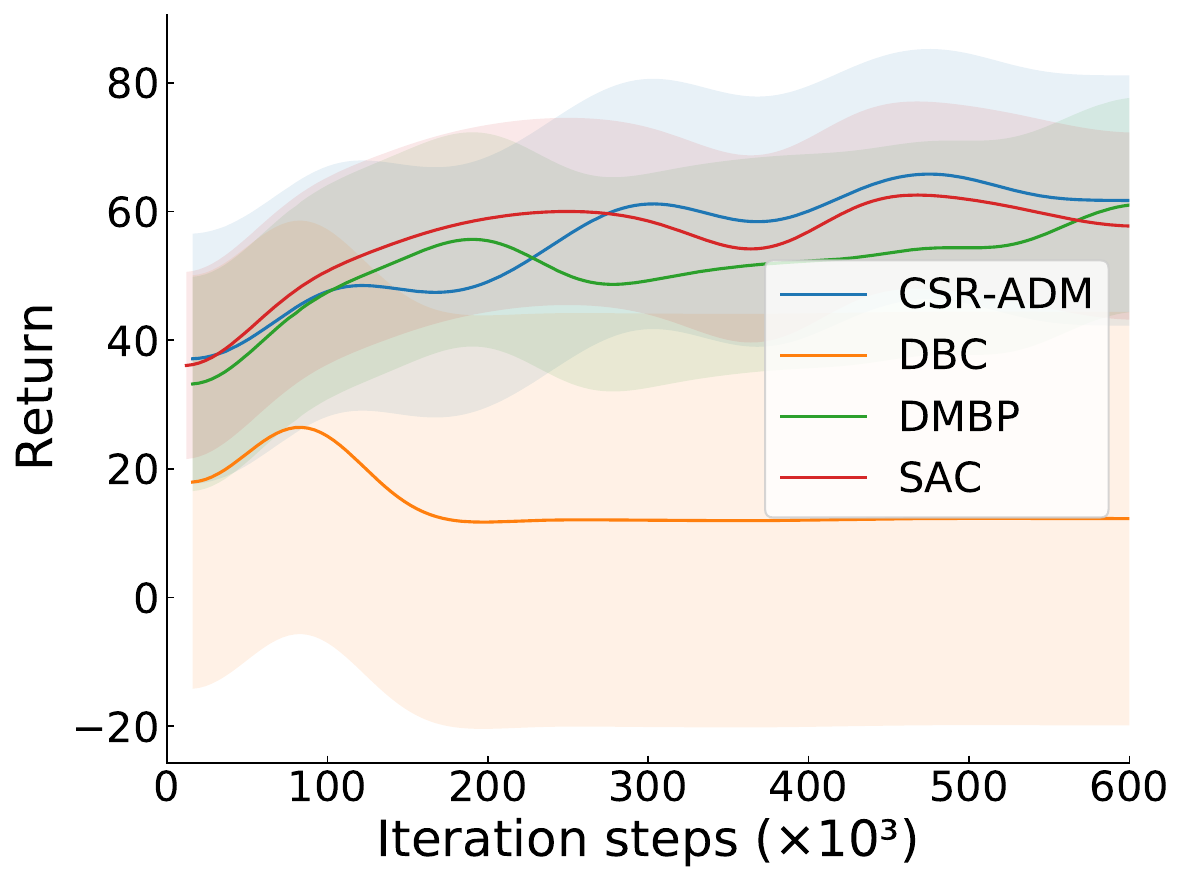}
    }
    \vspace{-0.2cm}
    \caption{Comparison of CaDiff and baselines on six environments.}
    \vspace{-0.4cm}
    \label{fig:compare}
\end{figure}

\subsubsection{Comparison with Baselines}\label{sec. baseline comparison}

Compared to classical SAC, DMBP \cite{DM_offline_3} (only denoising), and DBC \cite{bisimulation} (only considering bisimulation), we demonstrate the effectiveness and scalability of CaDiff. 
In Fig.~\ref{fig:compare}, CaDiff demonstrates superior performance across six environments. Compared to DMBP in Walker-V or DBC in Hopper-P, CaDiff exhibits superior generalization capabilities. Notably, CaDiff improves returns by at least 14.18\%, 29.42\%, and 136.63\% across the six environments when compared to SAC, DMBP, and DBC, respectively.
CaDiff achieves better performance in the early stages of training in five out of six environments.

\subsubsection{Ablation Study}\label{appendix:ablation}
    
\begin{figure}[ht]
    \centering
    \subfigure[Ant-P]{
        \centering
        \includegraphics[width=0.225\textwidth]{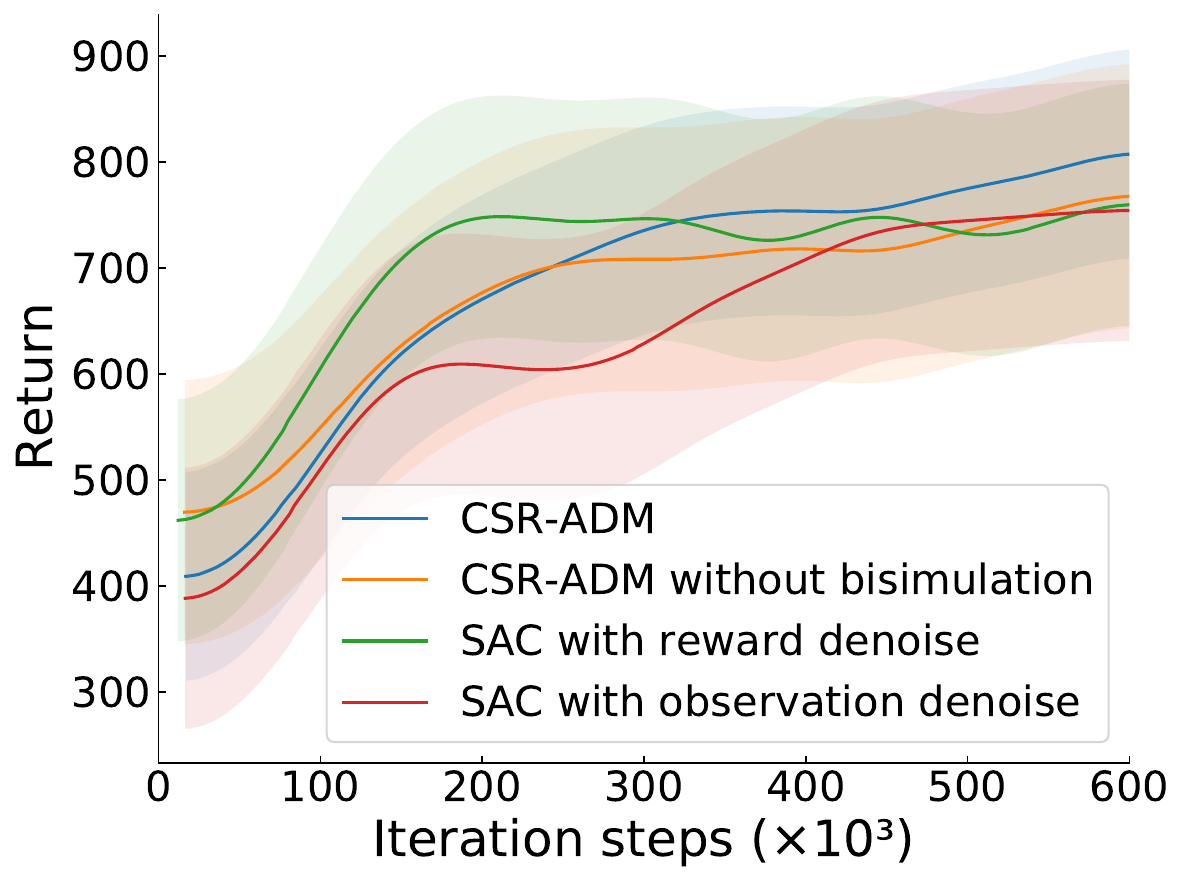}
    }
    \subfigure[Ant-V]{
        \centering
        \includegraphics[width=0.225\textwidth]{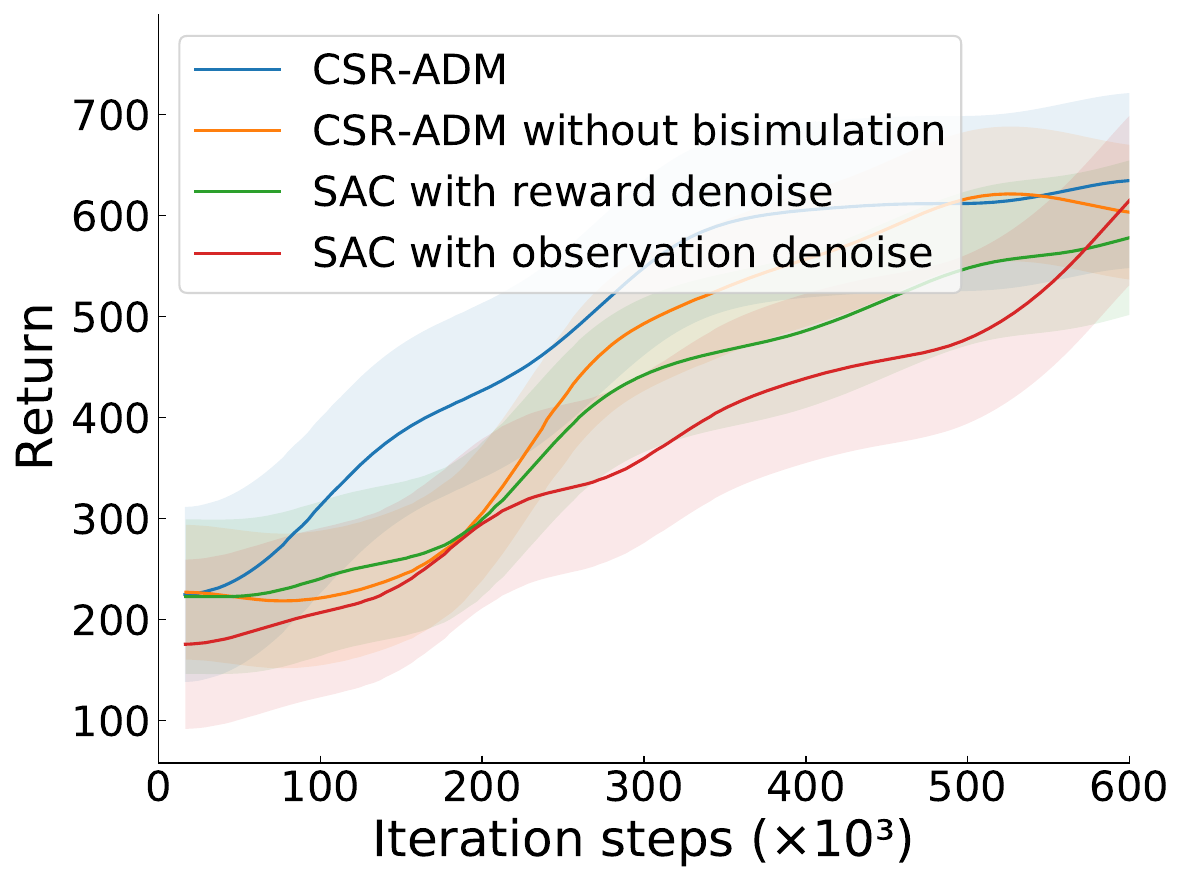}
    }
    \subfigure[Hopper-P]{
        \centering
        \includegraphics[width=0.225\textwidth]{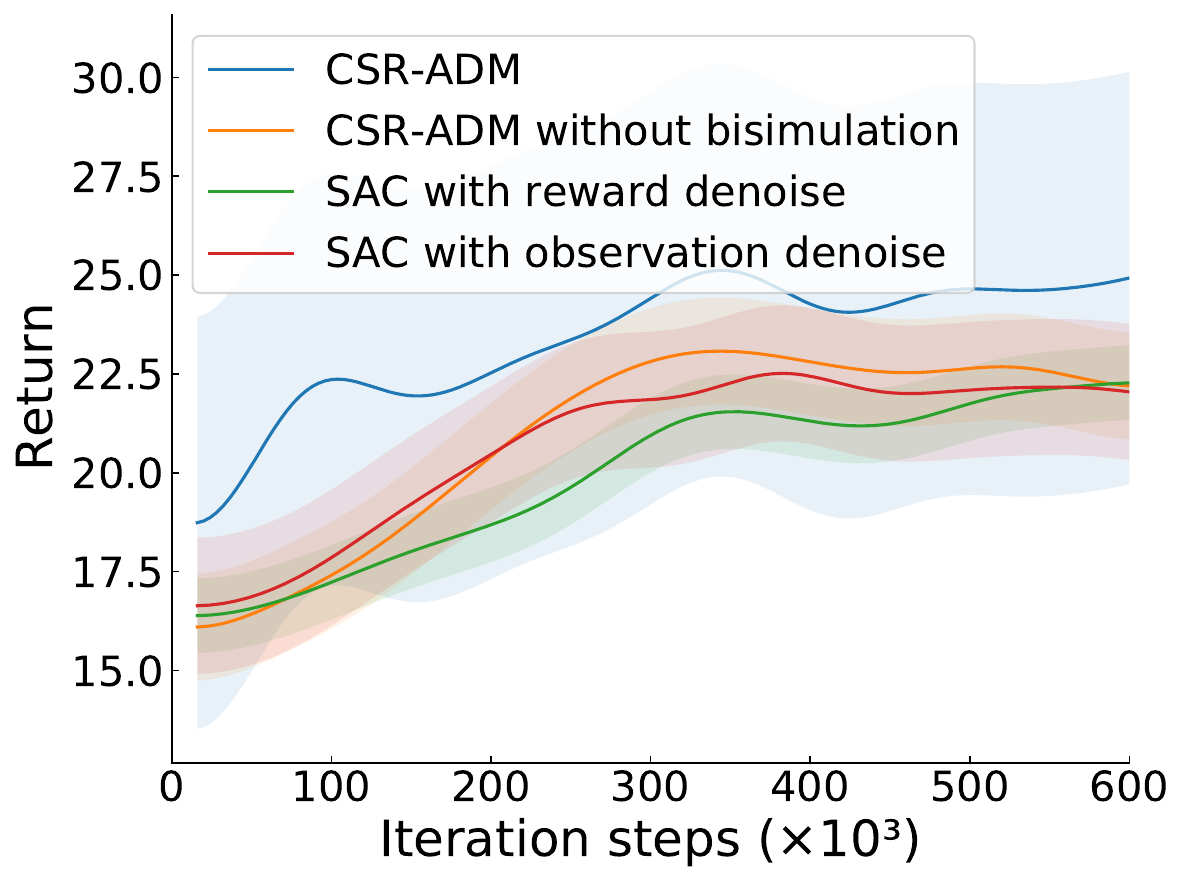}
    }
    \subfigure[Hopper-V]{
        \centering
        \includegraphics[width=0.225\textwidth]{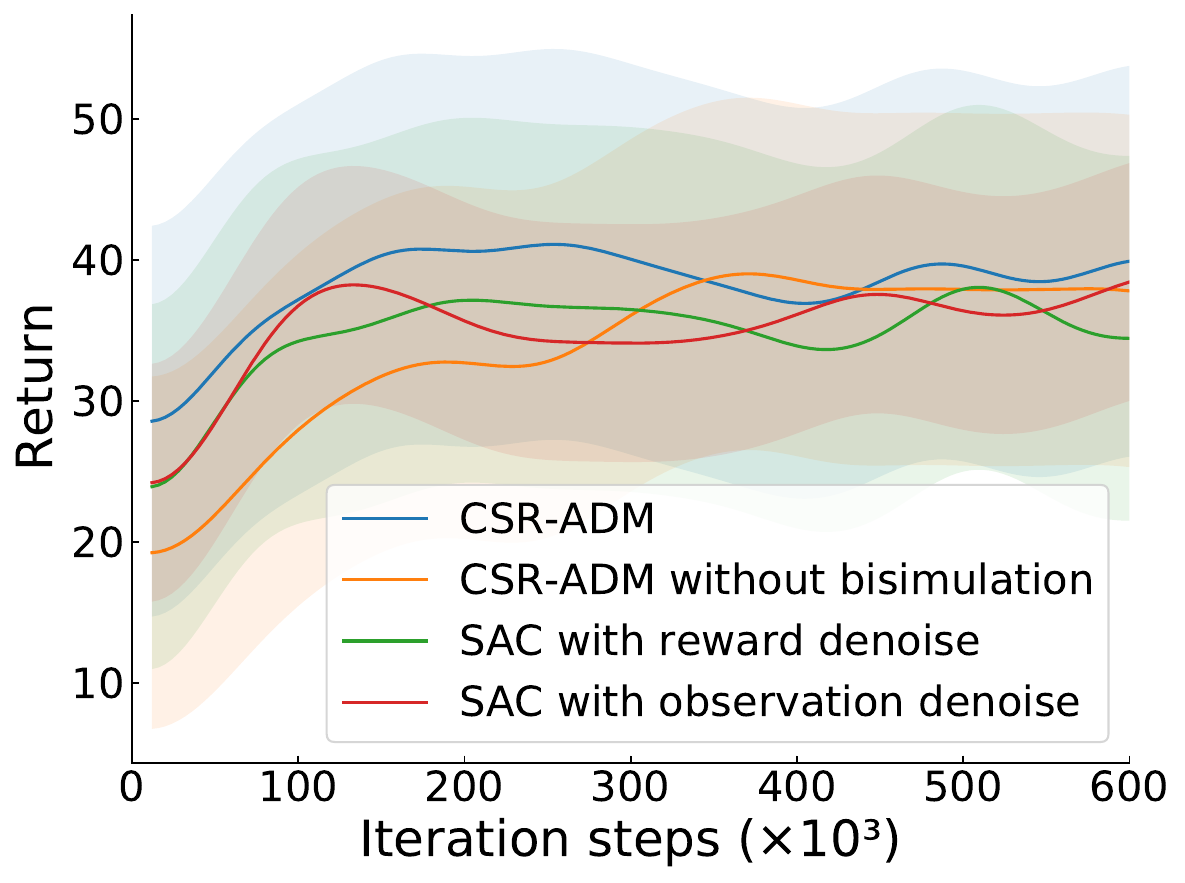}
    }
    \subfigure[Walker-P]{
        \centering
        \includegraphics[width=0.225\textwidth]{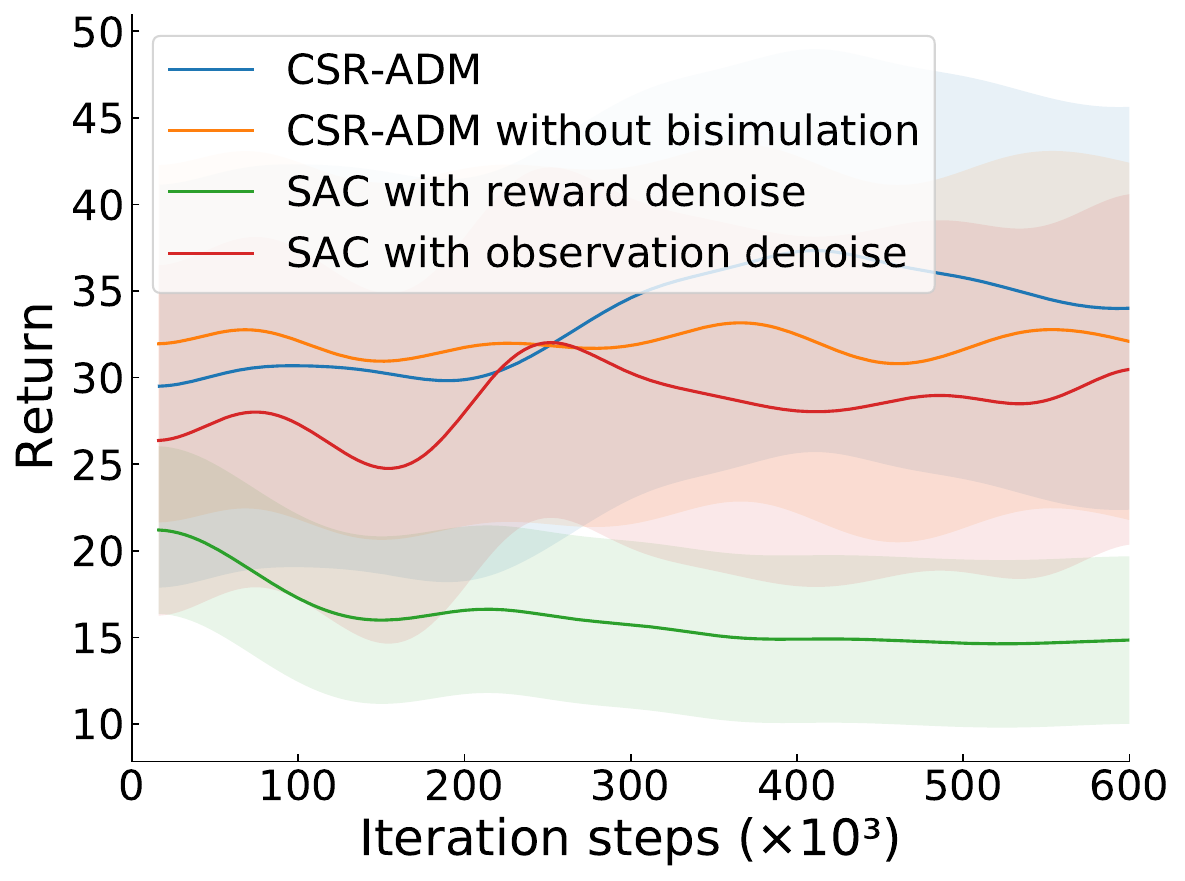}
    }
    \subfigure[Walker-V]{
        \centering
        \includegraphics[width=0.225\textwidth]{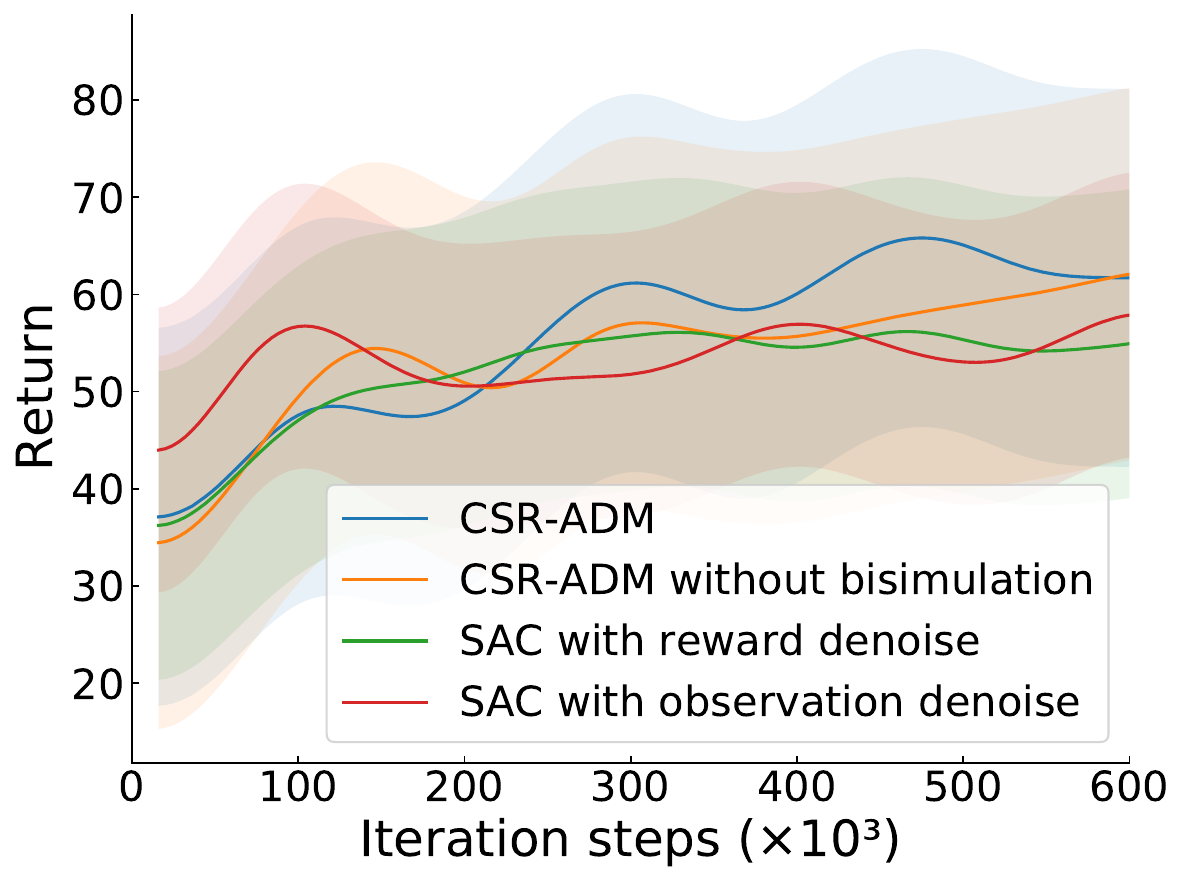}
    }
    \vspace{-0.2cm}
    \caption{Ablation studies of CaDiff on six environments.}
    \vspace{-0.4cm}
    \label{fig:ablation}
\end{figure}

We proceed to disable individual modules of CaDiff: bisimulation, reward denoising, and observation denoising. 
% The noise in environments and the noise intensity are configured 
% % the same as the experiments of comparison with the baselines.
% consistently with Section \ref{sec. baseline comparison}.
Fig.~\ref{fig:ablation} presents the ablation study of CaDiff across six environments. By comparing these three cases with our approach, it is evident that both bisimulation and ADM contribute to the return. In most environments, reward denoising is much less effective than state denoising. This can be attributed to the higher observation dimension compared to the reward, resulting in a greater impact of its noise. Additionally, the environments of Hopper-V and Walker-P exhibit higher sensitivity to noise (see Table~\ref{tab:key_para}).

\subsubsection{Influence on Key Parameters}
We evaluate CaDiff across three noise scales with varying intensities in six environments in Table~\ref{tab:key_para}. Bold numbers indicate the optimal results for each environment and noise scale. The conclusion drawn across the six environments is that for a noise scale of 0.1, the optimal noise intensity is $\delta=1$; for a noise scale of 0.5, the optimal noise intensity is $\delta=2$. {\color{black}When the noise scale increases to 1, the optimal noise intensity becomes unstable and fluctuates between 1 and 3, indicating that a larger noise intensity may be required to ensure reliable estimation under higher noise scales.}
% This indicates that noise intensity can reflect the impact of noise. 
As the noise scale increases from 0.1 to 0.5, half of the environments exhibit relatively stable returns. 
When the noise scale rises from 0.5 to 1, the returns of CaDiff across all six environments show no significant change. In this sense, CaDiff can maintain stable performance under high noise intensity. 
\begin{table}[t]
\centering
\setlength\tabcolsep{1mm}
\footnotesize
% \vspace{-0.3cm}
\caption{Returns of noise intensities with various noise scales.}
\vspace{-0.2cm}
\label{tab:key_para}
\begin{tabular}{cccccccc}
\hline
Noise scale          & $\delta$ & Ant-P & Ant-V & Hopper-P & Hopper-V & Walker-P & Walker-V \\ \hline
\multirow{3}{*}{0.1} & $1$  & \textbf{790.8} & \textbf{573.8} & \textbf{214} & \textbf{183} & \textbf{285.1} & \textbf{65.44} \\ \cline{2-8} 
                     & $2$  & 764.2 & 499.1 & 153.3 & 161 & 221.3 & 52.68 \\ \cline{2-8} 
                     & $3$  & 694 & 466 & 122.4 & 128.7 & 215.4 & 58.05 \\ \hline
\multirow{3}{*}{0.5}   & $1$  & 727.7 & 465.4 & 23.87 & 45.58 & 31.04 & 58.15 \\ \cline{2-8} 
                     & $2$  & \textbf{789.3} & \textbf{615.4} & \textbf{24.18} & \textbf{50.17} & \textbf{34.01} & \textbf{65.24} \\ \cline{2-8} 
                     & $3$  & 670.9 & 560.3 & 21.94 & 43.3 & 31.23 & 61.03 \\ \hline
\multirow{3}{*}{1}   & $1$  & 569.2 & \textbf{538.2} & 24.93 & 45.26 & 35.25 & 50.45 \\ \cline{2-8} 
                     & $2$  & 597.3 & 533.8 & \textbf{26.2} & 48.16 & \textbf{35.8} & 53.36 \\ \cline{2-8} 
                     & $3$  & \textbf{648.3} & 528.4 & 25.53 & \textbf{48.96} & 33.3 & \textbf{62.26} \\ \hline
\end{tabular}
    \vspace{-0.4cm}
\end{table}

\subsubsection{Computational Cost}\label{appendix:cost}
We provide a quantitative analysis of the additional computational cost introduced by our denoising CSR using the FLOPs metric.
Experiments are conducted on PyTorch 2.0.0 with an NVIDIA RTX 4090 24GB GPU, as shown in Table~\ref{tab:flops}. 
We also report the average runtime over 200 iterations across six environments. The runtime varies around 200 ms per update on average.
Compared to SAC, the complexity of our method is increased 
% the FLOPs 
by less than 7 GFLOPs, lower than the typical overhead of DMs; see \cite{yan2024diffusion}. 
% These new results are summarized in Table~\ref{tab:flops}.
%

\begin{table}[t]
\centering
\footnotesize
\caption{{\color{black}Computational cost in terms of FLOPs and runtime (ms).}}
\vspace{-0.2cm}
\label{tab:flops}
\begin{tabular}{cccccccc}
\hline
           &            & Ant & Hopper & Walker \\ \hline
\multirow{3}{*}{FLOPs on ``-P''} 
           & GFLOPs of $\phi$   & 3.3354 & 3.3350 & 3.3352 \\ \cline{2-5}
           & GFLOPs of $\theta$ & 3.3365 & 3.3355 & 3.3360 \\ \cline{2-5}
           & GFLOPs of $\zeta$   & 0.2785 & 0.1474 & 0.2129 \\ \hline
\multirow{3}{*}{FLOPs on ``-V''} 
           & GFLOPs of $\phi$   & 3.3352 & 3.3349 & 3.3351 \\ \cline{2-5}
           & GFLOPs of $\theta$ & 3.3359 & 3.3352 & 3.3356 \\ \cline{2-5}
           & GFLOPs of $\zeta$   & 0.1802 & 0.9830  & 0.1474 \\ \hline
\multirow{4}{*}{Runtime on ``-P''} 
           & CaDiff & 208.95 & 219.06 & 221.69 \\ \cline{2-5}
           & SAC    & 25.03  & 25.06  & 24.35  \\ \cline{2-5}
           & DMBP   & 86.64  & 87.67  & 87.11  \\ \cline{2-5}
           & DBC    & 40.29  & 41.03  & 41.27  \\ \hline
\multirow{4}{*}{Runtime on ``-V''} 
           & CaDiff & 209.09 & 211.93 & 219.99 \\ \cline{2-5}
           & SAC    & 24.60  & 24.74  & 24.17  \\ \cline{2-5}
           & DMBP   & 87.04  & 88.26  & 87.56  \\ \cline{2-5}
           & DBC    & 41.95  & 41.41  & 41.09  \\ \hline
\end{tabular}
\vspace{-0.4cm}
\end{table}

\section{Conclusion}\label{sec:con}
This paper introduces \textit{CaDiff} that effectively addresses the challenges posed by high-dimensional and noisy inputs to RL in P$^2$OMDPs. Integrating the new ADM for denoising both rewards and observations with a new bisimulation metric, CaDiff captures essential CSRs, which are crucial for decision-making tasks. Our analysis provides rigorous guarantees for the approximation of value functions between noisy observation spaces and causal state spaces.
% , reinforcing the framework's robustness. 
Experiments corroborate CaDiff's superiority, enhancing the performance and robustness of RL under various noise conditions. 
This research lays the groundwork for studies on CSR in noisy environments.

\appendix

% \section{Theoretical guarantee}
% In this section, we provide the proof of theoretical guarantee for our framework, including additional notation and basic facts, proofs of key theorems, and auxiliary proofs supporting the main results.

% \section{Additional notation and basic facts}
% In this part, we introduce the key facts about Wasserstein distances, H\"older norm, and notation about the asynchronous diffusion model, laying the foundation for our theoretical analysis.

\subsection{Proof of Remark~\ref{remark:rem1}}\label{appendix:remark}
We first prove that the update $d(\hat{\mathbf{s}}_i, \hat{\mathbf{s}}_j)$ satisfies the contraction property. Applying the Banach fixed-point theorem then ensures the existence and uniqueness of the resulting metric.
Start with $p=1$ under the Cauchy–Schwarz inequality as:
{\small\begin{align}\label{eq:app_4}
    &d(\hat{\mathbf{s}}_i, \hat{\mathbf{s}}_j) - d^\prime(\hat{\mathbf{s}}_i, \hat{\mathbf{s}}_j)\notag\\
    \le& \max\nolimits_{\mathbf{a}\in\mathcal{A}}\Big(C_{\mathrm{r}}W_1\left\|d-d'\right\|_\infty\left(P\left(r_{i+1}\mid \hat{\mathbf{s}}_{i}, \mathbf{a}\right), P\left(r_{j+1}\mid \hat{\mathbf{s}}_{j}, \mathbf{a}\right)\right)\notag\\
    &+C_{\mathrm{s}}W_1\left\|d-d'\right\|_\infty\left(P\left(\mathbf{s}'\mid \hat{\mathbf{s}}_{i}, \mathbf{a}\right), P\left(\mathbf{s}'\mid \hat{\mathbf{s}}_{j}, \mathbf{a}\right)\right)\Big)\notag\\
    \le& \left(C_{\mathrm{r}}+C_{\mathrm{s}}\right)\left\|d-d'\right\|_\infty, ~\forall(\hat{\mathbf{s}}_i, \hat{\mathbf{s}}_j) \in \mathcal{S}_{\mathrm{c}} \times \mathcal{S}_{\mathrm{c}}.
\end{align}}
% \vspace{-0.5cm}
% \end{figure*}
For $C_{\mathrm{r}} + C_{\mathrm{s}} \in [0, 1)$, the Banach fixed-point theorem~\cite{bharucha1976fixed} guarantees the existence and uniqueness of the fixed point.

We turn to the deterministic setting, where both $P$ and $\pi$ are deterministic, implying $P$ being delta distribution.
As noted in Definition~\ref{def:wass-primal}, the Wasserstein distance satisfies $W_p(d)(\Delta(\hat{\mathbf{s}}_i), \Delta(\hat{\mathbf{s}}_j)) = d(\mathbf{s}_i, \mathbf{s}_j)$ for point masses, where $\Delta(\cdot)$ denotes the Dirac delta function.
% of the Wasserstein metric.
% After obtaining 
Under the fact that
{\small{\begin{align}\label{eq:app_5}
    &d(\hat{\mathbf{s}}_i, \hat{\mathbf{s}}_j) - d^\prime(\hat{\mathbf{s}}_i, \hat{\mathbf{s}}_j)
    % = \max_{\mathbf{a}\in\mathcal{A}}\left(C_{\mathrm{r}}W_1(d)\left(P\left(r_{i+1}\mid \hat{\mathbf{s}}_{i}, \mathbf{a}\right), P\left(r_{j+1}\mid \hat{\mathbf{s}}_{j}, \mathbf{a}\right)\right)+C_{\mathrm{s}} W_1(d)\left(P\left(\mathbf{s}'\mid \hat{\mathbf{s}}_{i}, \mathbf{a}\right), P\left(\mathbf{s}'\mid \hat{\mathbf{s}}_{j}, \mathbf{a}\right)\right)\right)\notag\\
    % &\qquad\qquad\qquad\qquad\qquad\qquad-\max_{\mathbf{a}\in\mathcal{A}}\left(C_{\mathrm{r}}W_1(d')\left(P\left(r_{i+1}\mid \hat{\mathbf{s}}_{i}, \mathbf{a}\right), P\left(r_{j+1}\mid \hat{\mathbf{s}}_{j}, \mathbf{a}\right)\right)+C_{\mathrm{s}} W_1(d')\left(P\left(\mathbf{s}'\mid \hat{\mathbf{s}}_{i}, \mathbf{a}\right), P\left(\mathbf{s}'\mid \hat{\mathbf{s}}_{j}, \mathbf{a}\right)\right)\right)\notag\\
    =\max\nolimits_{\mathbf{a}\in\mathcal{A}}\Big(C_{\mathrm{r}}\left(d\left(r_{i'}, r_{j'}\right)-d'\left(r_{i'}, r_{j'}\right)\right)\notag\\
    &+C_{\mathrm{s}}\left(d\left(r_{i'}, r_{j'}\right)-d'\left(r_{i'}, r_{j'}\right)\right)\Big)\notag
    \leq\left(C_{\mathrm{r}}+C_{\mathrm{s}}\right)\left\|d-d^\prime\right\|_\infty,
\end{align}}}
the convergence of the update $d^{(n+1)}(\hat{\mathbf{s}}_i, \hat{\mathbf{s}}_j)$ from $d^{(n)}(\hat{\mathbf{s}}_i, \hat{\mathbf{s}}_j)$ is ensured when the underlying POMDP is finite, $\forall\hat{\mathbf{s}}_i, \hat{\mathbf{s}}_j$.

\subsection{Proof of Theorem~\ref{thm:generalized-ct}}\label{appendix:proofs_state}
We prove (\ref{eq:generalized-ct}) by induction with the update rules:
{\small{\begin{align}
    V^{(t+1)}(\hat{\mathbf{s}}_i) &= \max_{\mathbf{a}\in\mathcal{A}}\Big(\int_{r\in \mathcal{R}}r\left(\hat{\mathbf{s}}_i,\mathbf{a}\right)P\left(r_{i+1}\mid \hat{\mathbf{s}}_{i}, \mathbf{a}\right)dr\notag\\
    &+ \gamma \int_{\mathbf{s}^\prime \in \mathcal{S}}P\left(\mathbf{s}'\mid \hat{\mathbf{s}}_{i}, \mathbf{a}\right)V^{(t)}(\mathbf{s}^\prime)d\mathbf{s}^\prime\Big)\\
    d^{(t+1)}(\hat{\mathbf{s}}_i, \hat{\mathbf{s}}_j) &= \max_{\mathbf{a}\in\mathcal{A}}(C_{\mathrm{r}}W_p(d^{(t)})\left(P\left(r_{i+1}\mid \hat{\mathbf{s}}_{i}, \mathbf{a}\right), P\left(r_{j+1}\mid \hat{\mathbf{s}}_{j}, \mathbf{a}\right)\right)\notag\\
    &\!\!\!\!\!\!\!\!\!\!\!\!\!\! +C_{\mathrm{s}} W_p(d^{(t)})\left(P\left(\mathbf{s}'\mid \hat{\mathbf{s}}_{i}, \mathbf{a}\right), P\left(\mathbf{s}'\mid \hat{\mathbf{s}}_{j}, \mathbf{a}\right)\right)).
\end{align}}}
We proceed to demonstrate that, $\forall t \in \mathbb{N}$,
{\small{\begin{align}\label{eq:base_Vn}
        C_{\mathrm{r}}\!\left|V^{(t)}(\hat{\mathbf{s}}_i) \!- \!V^{(t)}(\hat{\mathbf{s}}_j)\right| \!\leq\! d^{(t)}(\hat{\mathbf{s}}_i, \hat{\mathbf{s}}_j), \forall(\hat{\mathbf{s}}_i, \hat{\mathbf{s}}_j) \!\in\! \mathcal{S}_\mathrm{c} \!\times\! \mathcal{S}_\mathrm{c}.
\end{align}}}
Then, (\ref{eq:generalized-ct}) holds as $t \rightarrow \infty$.
With mathematical induction, the base case $t=0$ satisfies because
{\small{\begin{align*}
    \left|V^{(0)}(\hat{\mathbf{s}}_i) - V^{(0)}(\hat{\mathbf{s}}_j)\right|=d^{(0)}(\hat{\mathbf{s}}_i, \hat{\mathbf{s}}_j)=0, ~\forall (\hat{\mathbf{s}}_i, \hat{\mathbf{s}}_j) \in \mathcal{S}_\mathrm{c} \times \mathcal{S}_\mathrm{c}.
\end{align*}}}
Assume (\ref{eq:base_Vn}) holds $\forall t$. In the case of $t+1$, we have \eqref{eq:app_1} under triangle inequality.
\begin{figure*}[ht]
    \centering
    {\small{\begin{align}\label{eq:app_1}
        &C_{\mathrm{r}}|V^{(t+1)}(\hat{\mathbf{s}}_i) - V^{(t+1)}(\hat{\mathbf{s}}_j)|\\
        % =& C_{\mathrm{r}} \Big|\max_{\mathbf{a}\in\mathcal{A}}\left(\int_{r\in \mathcal{R}}r\left(\hat{\mathbf{s}}_i,\mathbf{a}\right)P\left(r_{i+1}\mid \hat{\mathbf{s}}_{i}, \mathbf{a}\right)dr + \gamma \int_{\mathbf{s}^\prime \in \mathcal{S}}P\left(\mathbf{s}'\mid \hat{\mathbf{s}}_{i}, \mathbf{a}\right)V^{(t)}(\mathbf{s}^\prime)d\mathbf{s}^\prime\right)\notag\\
        % &-\max_{\mathbf{a}\in\mathcal{A}}\left(\int_{r\in \mathcal{R}}r\left(\hat{\mathbf{s}}_j,\mathbf{a}\right)P\left(r_{j+1}\mid \hat{\mathbf{s}}_{j}, \mathbf{a}\right)dr + \gamma \int_{\mathbf{s}^\prime \in \mathcal{S}}P\left(\mathbf{s}'\mid \hat{\mathbf{s}}_{j}, \mathbf{a}\right)V^{(t)}(\mathbf{s}^\prime)d\mathbf{s}^\prime\right)\Big|\\
        \le& C_{\mathrm{r}} \Big|\max_{\mathbf{a}\in\mathcal{A}}\Big(\int_{r\in \mathcal{R}}\!\!r\left(\hat{\mathbf{s}}_i,\mathbf{a}\right)P\left(r_{i+1}\mid \hat{\mathbf{s}}_{i}, \mathbf{a}\right)dr\!-\!\int_{r\in \mathcal{R}}\!\!r\left(\hat{\mathbf{s}}_j,\mathbf{a}\right)P\left(r_{j+1}\mid \hat{\mathbf{s}}_{j}, \mathbf{a}\right)dr\!+\! \gamma \int_{\mathbf{s}^\prime \in \mathcal{S}}\!\!\left(P\left(\mathbf{s}'\mid \hat{\mathbf{s}}_{i}, \mathbf{a}\right)\!-\!P\left(\mathbf{s}'\mid \hat{\mathbf{s}}_{j}, \mathbf{a}\right)\right)V^{(t)}(\mathbf{s}^\prime)d\mathbf{s}^\prime\Big)\Big|\notag\\
        % \leq& C_{\mathrm{r}}\!\max_{\mathbf{a}\in\mathcal{A}}\!\left|\!\int_{r\in \mathcal{R}}\!\!\!r\left(\hat{\mathbf{s}}_i,\mathbf{a}\right)\!P\left(r\!\mid\! \hat{\mathbf{s}}_{i}, \mathbf{a}\right)\!dr\!-\!\int_{r\in \mathcal{R}}\!\!\!r\left(\hat{\mathbf{s}}_j,\mathbf{a}\right)\!P\left(r\!\mid\! \hat{\mathbf{s}}_{j}, \mathbf{a}\right)dr\right|\!+\! C_{\mathrm{r}} \gamma \max_{\mathbf{a}\in\mathcal{A}}\!\left|\!\int_{\mathbf{s}^\prime \in \mathcal{S}}\!\!\left(P\left(\mathbf{s}'\!\mid\! \hat{\mathbf{s}}_{i}, \mathbf{a}\right)\!-\!P\left(\mathbf{s}'\!\mid\! \hat{\mathbf{s}}_{j}, \mathbf{a}\right)\right)V^{(t)}(\mathbf{s}^\prime)d\mathbf{s}^\prime\right| \notag\\
        \leq& C_{\mathrm{r}}\!\max_{\mathbf{a}\in\mathcal{A}}\!\left|\!\int_{r\in \mathcal{R}}\!\!\!r\left(\hat{\mathbf{s}}_i,\mathbf{a}\right)\!P\left(r_{i+1}\!\mid\! \hat{\mathbf{s}}_{i}, \mathbf{a}\right)\!dr\!-\!\int_{r\in \mathcal{R}}\!\!\!r\left(\hat{\mathbf{s}}_j,\mathbf{a}\right)\!P\left(r_{j+1}\!\mid\! \hat{\mathbf{s}}_{j}, \mathbf{a}\right)dr\right|\!+\! C_{\mathrm{s}} \gamma \max_{\mathbf{a}\in\mathcal{A}}\!\left|\!\int_{\mathbf{s}^\prime \in \mathcal{S}}\!\!\left(P\left(\mathbf{s}'\!\mid\! \hat{\mathbf{s}}_{i}, \mathbf{a}\right)\!-\!P\left(\mathbf{s}'\!\mid\! \hat{\mathbf{s}}_{j}, \mathbf{a}\right)\right)\!\frac{C_{\mathrm{r}} \gamma}{C_{\mathrm{s}}}\!V^{(t)}(\mathbf{s}^\prime)d\mathbf{s}^\prime\right| \notag
    \end{align}}}
\vspace{-0.5cm}
\end{figure*}
By the induction hypothesis, the function $C_{\mathrm{r}} V^{(t)}(\hat{\mathbf{s}})$ is 1-Lipschitz w.r.t. the metric $d^{(t)}$, that is, $C_{\mathrm{r}} V^{(t)}(\hat{\mathbf{s}}) \in \mathrm{Lip}_{1, d^{(t)}}$. Since $\gamma \leq C_{\mathrm{s}}$ by assumption, $\frac{C_{\mathrm{r}} \gamma}{C_{\mathrm{s}}}V^{(t)}(\hat{\mathbf{s}})$ is 1-Lipschitz. Under the assumption of $r\left(\hat{\mathbf{s}},\mathbf{a}\right) \in \mathrm{Lip}_{1, d^{(t)}}$, the dual form of the $W_1$ metric in (\ref{eq:wass-dual}) leads to \eqref{eq:app_12},
\begin{figure*}[ht]
    \centering
{\small{\begin{align}\label{eq:app_12}
    &C_{\mathrm{r}}|V^{(t+1)}(\hat{\mathbf{s}}_i) \!-\! V^{(t+1)}(\hat{\mathbf{s}}_j)| 
    \leq \!C_{\mathrm{r}}\!\max\nolimits_{\mathbf{a}\in\mathcal{A}}\!\left(W_1(d^{(t)})\!\left(P\left(r_{i+1}\!\!\mid\! \hat{\mathbf{s}}_{i}, \mathbf{a}\right)\!,\!P\left(r_{j+1}\!\!\mid \hat{\mathbf{s}}_{j}, \mathbf{a}\right)\right)\right) \!\!+\!\! C_{\mathrm{s}} \max\nolimits_{\mathbf{a}\in\mathcal{A}}\!\left(W_1(d^{(t)})\!\left(P\left(\mathbf{s}'\!\mid\! \hat{\mathbf{s}}_{i}, \mathbf{a}\right)\!,\!P\left(\mathbf{s}'\!\mid\! \hat{\mathbf{s}}_{j}, \mathbf{a}\right)\right)\!\right)\notag\\
    &\qquad\leq C_{\mathrm{r}}\!\max\nolimits_{\mathbf{a}\in\mathcal{A}}\left(W_p(d^{(t)})\!\left(P\left(r_{i+1}\!\mid\! \hat{\mathbf{s}}_{i}, \mathbf{a}\right)\!,\!P\left(r_{j+1}\mid \hat{\mathbf{s}}_{j}, \mathbf{a}\right)\right)\right) \!+\! C_{\mathrm{s}} \max\nolimits_{\mathbf{a}\in\mathcal{A}}\left(W_p(d^{(t)})\!\left(P\left(\mathbf{s}'\!\mid\! \hat{\mathbf{s}}_{i}, \mathbf{a}\right)\!,\!P\left(\mathbf{s}'\!\mid\! \hat{\mathbf{s}}_{j}, \mathbf{a}\right)\right)\right)
    = d^{(t+1)},
\end{align}}}
\vspace{-0.5cm}
\end{figure*}
where the second inequality holds due to Lemma \ref{lemma:wass-lemma}.

\subsection{Proof of Theorem~\ref{thm:valboundmodelerror}}
\label{appendix:theorem:valboundmodelerror}
% To prove Theorem~\ref{thm:valboundmodelerror}, 
We first introduce two useful lemmas.
\begin{lemma}[Bound for convergence]
\label{thm:boundedness}
Let $\mathcal{S}_\mathrm{c}$ be compact. If the support $\mathcal{S}_{\mathrm{c}}^{\prime} = \mathrm{supp}(\widehat{P})\subset\mathcal{S}_{\mathrm{c}}$ of the fitted dynamics $\widehat{P}$ is closed, \eqref{eq:bisimulation} admits a unique bounded bisimulation metric $\widehat{d}$:
{\small{\begin{align*}
    \mathrm{supp}(\widehat{{P}}) \subseteq \mathcal{S}_\mathrm{c} \!\Rightarrow\! \mathrm{diam}(\mathcal{S}_\mathrm{c}; \widehat{d}) \!\leq\! {C_{\mathrm{r}}}(r_{\mathrm{max}} - r_{\mathrm{min}})/(1 - C_{\mathrm{s}}).
\end{align*}}}
\end{lemma}
\begin{proof}
    See Appendix~\ref{thm:boundedness_proof}.
\end{proof}

\begin{lemma}[Bisimulation distance approximation error]
\label{lemma:distanceerror}
Define $C_{\mathrm{r}} \geq 0$ and $C_{\mathrm{s}} \in [0,1)$. Suppose $\mathrm{supp}(\widehat{P}) \subseteq \mathcal{S}_\mathrm{c}$
and $1 - (C_{\mathrm{r}}+C_{\mathrm{s}})a_p > 0$. 
With $a_p = 2^{(p-1)/p}$ and 
$\mathrm{diam}(\mathcal{S}_\mathrm{c};{d}) \leq  \frac{C_{\mathrm{r}}}{1-C_{\mathrm{s}}}(r_{\mathrm{max}}-r_{\mathrm{min}})$, based on Lemma~\ref{thm:boundedness}, there is
{\small{\begin{align}
\left\|{d  -  \widehat{d}}\right\|_\infty \leq& 
[{ (C_{\mathrm{r}}+C_{\mathrm{s}})(a_p - 1) }
\mathrm{diam}(\mathcal{S}_\mathrm{c};{d}) \notag\\
&+ 
{2C_{\mathrm{r}}} \mathcal{E}_{\phi} + {2C_{\mathrm{s}}} \mathcal{E}_{\theta}]/[{1 - (C_{\mathrm{r}}+C_{\mathrm{s}}) a_p}].
\end{align}}}
\end{lemma}
\begin{proof}
    See Appendix~\ref{lemma:distanceerror_proof}.
\end{proof}

If $p=1$, there is $a_p = a_1 = 1$. Based on Lemma~\ref{lemma:distanceerror}, we obtain the following Corollary.
\begin{corollary}[Bisimulation distance approximation error for $p=1$]
\label{corollary:distanceerrorpequal1}
On the remaining conditions in Lemma \ref{lemma:distanceerror}, for $p=1$,
{\small{\begin{equation}\label{eq:coro_1}
\left\|{d - \widehat{d}}\right\|_\infty
\leq 
\frac{2C_{\mathrm{r}}}{1-C_{\mathrm{r}}-C_{\mathrm{s}}} \mathcal{E}_{\phi} + \frac{2C_{\mathrm{s}}}{1-C_{\mathrm{r}}-C_{\mathrm{s}}} \mathcal{E}_{\theta}.
\end{equation}}}
\end{corollary}
% \begin{proof}
%     When $p=1$, we have $a_p = a_1 = 1$. Based on Lemma~\ref{lemma:distanceerror}, we obtain \eqref{eq:coro_1}.
% \end{proof}

Corollary~\ref{corollary:distanceerrorpequal1} establishes a bound on the deviation between the true bisimulation distance and its optimal \textit{approximate} counterpart, reflecting the achievable distance under encoder and ADM approximation errors. 
We express the VFA error w.r.t. $\widehat{d}_{\zeta}$ (not $\widehat{d}$),  incorporating the effect of encoder and ADM.
% We first bound the true bisimulation distance.
Using Corollary \ref{corollary:distanceerrorpequal1} and the definition of bisimulation encoder, 
{\small{\begin{align}\label{lemma:bisimboundencerr}
    \left\|{d - \widehat{d}_{\zeta}}\right\|_\infty &\leq \left\|{d - \widehat{d}}\right\|_\infty + \left\|{\widehat{d}_{\zeta} - \widehat{d}}\right\|_\infty\notag\\
    &\leq \frac{2C_{\mathrm{r}}}{1-C_{\mathrm{r}}-C_{\mathrm{s}}} \mathcal{E}_{\phi} + \frac{2C_{\mathrm{s}}}{1-C_{\mathrm{r}}-C_{\mathrm{s}}} \mathcal{E}_{\theta} + \mathcal{E}_\zeta.
\end{align}}}
Once $d$ relates to VFA, a similar relationship can be established for $\widehat{d}_{\zeta}$ as a function of the model approximation error.
% Then, we bound the state VFA error using the \textit{approximate} bisimulation distance induced by the learned encoder.

Let $\widehat{\epsilon}$ denote the aggregation radius in $\zeta$-space, such that the $\widehat{d}_{\zeta}$-diameter of each equivalence class (or partition subset) is bounded by $2\widehat{\epsilon}$. Thus,
% {\small{\begin{equation*}
    $\sup\nolimits_{{\mathbf{o}}_i, {\mathbf{o}}_j \in \mathcal{O}}
    \| \zeta({\mathbf{o}}_i) - \zeta({\mathbf{o}}_j) \|_q ~\leq~ 2 \widehat{\epsilon}$.
% \end{equation*}}}
Note that $\widehat{\epsilon}$ provides an upper bound on the maximum diameter of the partition cells w.r.t. the learned distance induced by $\zeta$, not the true bisimulation distance.

\begin{lemma}[Value difference bound with causal state]
\label{thm:generalized-ct-vfa}
Let $\zeta: {\mathcal{O}} \rightarrow \tilde{\mathcal{S}}_{\mathrm{c}}$ be a function mapping denoised observations to causal states such that $\zeta({\mathbf{o}}_i) = \zeta({\mathbf{o}}_j)$ is equivalent to $d(\tilde{\mathbf{s}}_i, \tilde{\mathbf{s}}_j) \leq 2\epsilon$. For $C_{\mathrm{r}}, C_{\mathrm{s}} \in [0, 1)$ and $C_{\mathrm{r}}+C_{\mathrm{s}}<1$:
{\small{\begin{align}
    |V^\pi(\mathbf{s}_{\mathrm{c}}) - V^\pi(\hat{F}({\mathbf{o}}))| ~\leq~ {2\epsilon}/{[C_{\mathrm{r}}(1-\gamma)]}, ~\forall\mathbf{s}_{\mathrm{c}} \in \mathcal{S}_{\mathrm{c}}.
\end{align}}}
\end{lemma}
\begin{proof}
    See Appendix~\ref{appendix:generalized-ct-vfa}.
\end{proof}

From Lemma~\ref{thm:generalized-ct-vfa}, it readily follows that
{\small{\begin{align*}
    &(1-\gamma)|V(\mathbf{s}_{\mathrm{c}}) - V(\hat{F}({\mathbf{o}}))|
    \leq \frac{C_{\mathrm{r}}^{-1}}{\xi(\hat{F}({\mathbf{o}}))}\int_{\mathbf{z} \in \hat{F}({\mathbf{o}})}d(\mathbf{s}_{\mathrm{c}}, \mathbf{z})d\xi(\mathbf{z}) \\
    \leq & \frac{C_{\mathrm{r}}^{-1}}{\xi(\hat{F}({\mathbf{o}}))}\int_{\mathbf{z} \in \hat{F}({\mathbf{o}})} \widehat{d}_{\zeta}(\mathbf{s}_{\mathrm{c}}, \mathbf{z}) + | d(\mathbf{s}_{\mathrm{c}},\mathbf{z}) - \widehat{d}_{\zeta}(\mathbf{s}_{\mathrm{c}}, \mathbf{z}) |_\infty d\xi(\mathbf{z}) \\
    % \leq & \frac{C_{\mathrm{r}}^{-1}}{\xi(\hat{F}({\mathbf{o}}))}\int_{\mathbf{z} \in \hat{F}({\mathbf{o}})} 2\widehat{\epsilon} + A_1 d\xi(\mathbf{z}) \\
    % =& C_{\mathrm{r}}^{-1}( 2\widehat{\epsilon} + A_2)\\
    \leq & \frac{1}{C_{\mathrm{r}}}\left( 2\widehat{\epsilon} + \mathcal{E}_\zeta + \frac{2C_{\mathrm{r}}}{1-C_{\mathrm{r}}-C_{\mathrm{s}}} \mathcal{E}_{\phi} + \frac{2C_{\mathrm{s}}}{1-C_{\mathrm{r}}-C_{\mathrm{s}}} \mathcal{E}_{\theta} \right),
\end{align*}}}

\noindent where the last inequality is due to (\ref{lemma:bisimboundencerr}). Theorem~\ref{thm:valboundmodelerror} is proved.

\subsection{Proof of Theorem~\ref{thm:wass_CSR_ADM}}\label{appendix:proof_wass_CSR_ADM}
% For conciseness, we denote $\mathbf{y}=(\hat{\mathbf{s}}_{t}, \mathbf{a}_t)$.
Notice that we have the following decomposition:
{\small{\begin{align}\label{eq:Appendix-Estimation-W1-1}
     &{W_1(P(\mathbf{x}|\mathbf{y}), \hat{P}({\mathbf{x}}^{k_0}|\mathbf{y}))} \leq
     {W_1(P(\mathbf{x}|\mathbf{y}), P({\mathbf{x}}^{k_0}|\mathbf{y}))}\\
     &+{W_1(P({\mathbf{x}}^{k_0}|\mathbf{y}), P^\prime({\mathbf{x}}^{k_0}|\mathbf{y}))}
     +{W_1(P^\prime({\mathbf{x}}^{k_0}|\mathbf{y}),\hat{P}({\mathbf{x}}^{k_0}|\mathbf{y}))},\notag
\end{align}}}

\noindent where ${W_1(P(\mathbf{x}|\mathbf{y}), P({\mathbf{x}}^{k_0}|\mathbf{y}))}$ follows the correspondence between the forward and backward processes; ${W_1(P({\mathbf{x}}^{k_0}|\mathbf{y}), P^\prime({\mathbf{x}}^{k_0}|\mathbf{y}))}$ follows the definitions of $\mathbf{x}$ and $\mathbf{x}'$ (with the only difference in initial distribution), with the latter being the result obtained by the true distribution.

We use another backward process as a transition term between ${\mathbf{x}}'_k$ and $\overline{\mathbf{x}}'_k$, which is defined as: with $\overline{\mathbf{x}}'_0 \sim  \mathcal{N}(0, I)$,
{\small{\begin{equation}
    d\overline{\mathbf{x}}'_k = \left[ 0.5\overline{\mathbf{x}}'_k + \nabla \log P_{K-k}(\overline{\mathbf{x}}'_k | \mathbf{y})\right] dk + d\hat{\mathbf{w}}_k.
\end{equation}}}
Let $P^{\prime}_{K-k}(\cdot|\mathbf{y})$ denote the conditional distribution of $\overline{\mathbf{x}}'_k$ on $\mathbf{y}$.
We bound the three terms on the RHS of (\ref{eq:Appendix-Estimation-W1-1}), as follows.

\textit{Bound ${W_1(P(\mathbf{x}|\mathbf{y}), P({\mathbf{x}}^{k_0}|\mathbf{y}))}$:}
Let $X \sim P(\mathbf{x}|\mathbf{y})$ and $Z \sim \mathcal{N}(0,I)$. Then,
{\small{\begin{align}
W_1(P(\mathbf{x}|\mathbf{y}), &P({\mathbf{x}}^{k_0}|\mathbf{y}))  \leq \mathbb{E}[\|X-\sqrt{\alpha_{k_0}}X + \sigma_{k_0}Z\|]\nonumber\\
   &\leq (1-\sqrt{\alpha_{k_0}}) \mathbb{E}[\|X\|] + \sigma_{k_0}\mathbb{E}[\|Z\|]
  \\ &  \leq (1-\sqrt{\alpha_{k_0}}) \sqrt{d} + \sigma_{k_0}\sqrt{d} \lesssim \sqrt{k_0},\notag 
\end{align}}}
where the last holds since ${\sigma_{k}}/{\sqrt{\alpha_{k}}}=\mathcal{O}(\sqrt{k})$ if $k=o(1)$.

\textit{Bound ${W_1(P({\mathbf{x}}^{k_0}|\mathbf{y}), P^\prime({\mathbf{x}}^{k_0}|\mathbf{y}))}$:}
Because $\overline{\mathbf{x}}'_k$ and $\overline{\mathbf{x}}_k$ evolve under an identical backward  stochastic differential equality (SDE) while differing only in their initial distributions, following \cite[Lemma 2]{canonne2022short}, we have
{\small{\begin{align}
    W_1(P({\mathbf{x}}^{k_0}|\mathbf{y}), &P^\prime({\mathbf{x}}^{k_0}|\mathbf{y}))\lesssim{\rm TV}(P({\mathbf{x}}^{k_0}|\mathbf{y}),P^{\prime}({\mathbf{x}}^{k_0}|\mathbf{y}))\nonumber \\
    % &\lesssim \sqrt{{\rm KL} (P({\mathbf{x}}^{k_0}|\mathbf{y})||P^{\prime}({\mathbf{x}}^{k_0}|\mathbf{y}) )}\nonumber\\
    & \lesssim \sqrt{{\rm KL} (P({\mathbf{x}}^{K}|\mathbf{y})||\mathcal{N}(0, I) )}\\
    &\lesssim \sqrt{{\rm KL} (P({\mathbf{x}}|\mathbf{y})||\mathcal{N}(0, I) )}\exp(-K)
% \end{align*}
% which leads to
% \begin{align}
%     W_1(P({\mathbf{x}}^{k_0}|\mathbf{y}), P^\prime({\mathbf{x}}^{k_0}|\mathbf{y}))
    \lesssim \exp(-K).\notag
\end{align}}}

\textit{Bound ${W_1(P^\prime({\mathbf{x}}^{k_0}|\mathbf{y}),\hat{P}({\mathbf{x}}^{k_0}|\mathbf{y}))}$:}
Although Assumption~\ref{assump:distri_true} does not ensure Novikov's condition according to \cite{chen2023sampling}, Girsanov’s theorem bounds the KL divergence between any two distributions produced from the same SDE, provided the score estimation error possesses a bounded second moment and the KL divergence w.r.t. the standard Gaussian remains finite.

% \begin{lemma}[{\cite[Lemma D.4]{DM_offline_9}}]\label{lemma:Girsanov}
%     Let $P_0$ be a probability distribution, and $Y=\left\{Y_k\right\}_{k\in[0,K]}$ and $Y'=\left\{Y'_k\right\}_{k\in[0,K]}$ be two stochastic processes that satisfy the following SDEs:
%     \begin{align*}
%         dY_k&=s(Y_k,k) dt+ d W_k,~~~Y_0\sim P_0; \\
%         dY'_k&=s'(Y'_k,k) dk+ d W_k,~~~Y'_0\sim P_0.
%     \end{align*}
%     Further define the distributions of $Y_k$ and $Y'_k$ as $P_k$ and $P^\prime_k$, respectively. Suppose that
%     \begin{align}\label{equ::weak condition}
%         \int_{\mathbf{x}}P_k(\mathbf{x})\left\|{(s-s')(\mathbf{x},k)}\right\|^2 d\mathbf{x} \le C, \qquad \forall k\in[0,K].
%     \end{align}
%     We have
%     \begin{align*}
%         \text{KL} \left(P_K|P^\prime_K\right)\le \int_{0}^{K}\frac{1}{2}\int_{\mathbf{x}}P_k(\mathbf{x})\left\|{(s-s')(\mathbf{x},k)}\right\|^2 d\mathbf{x} dk.
%     \end{align*}
% \end{lemma}

According to {\cite[Lemma D.4]{DM_offline_9}}, we obtain
{\small{\begin{align*}
    &W_1(P^\prime({\mathbf{x}}^{k_0}|\mathbf{y}),\hat{P}({\mathbf{x}}^{k_0}|\mathbf{y}))\notag\\
    \lesssim& {\rm TV}(P^\prime({\mathbf{x}}^{k_0}|\mathbf{y}),\hat{P}({\mathbf{x}}^{k_0}|\mathbf{y}))\lesssim \sqrt{{\rm KL}(P^\prime({\mathbf{x}}^{k_0}|\mathbf{y}),\hat{P}({\mathbf{x}}^{k_0}|\mathbf{y}))}\nonumber\\
    \lesssim& \sqrt{\int_{k_0}^{K}\frac{1}{2}\int_{{\mathbf{x}}^{k}}P_k({\mathbf{x}}^{k}|\mathbf{y})\left\|{{\varphi}({\mathbf{x}}^{k},\mathbf{y},k)-\nabla \log P({\mathbf{x}}^{k}|\mathbf{y})}\right\|^2 d{\mathbf{x}}^{k} dk}.
\end{align*}}}
For a state-action pair $\mathbf{y}^{\star}=(\mathbf{s}^{\star}_{\mathrm{c}}, \mathbf{a}^{\star})$, the estimated conditional distribution ${P}(\mathbf{x}^{k_0}|\mathbf{s}^{\star}_{\mathrm{c}}, \mathbf{a}^{\star})$ is obtained via the reverse process of ADM, arriving at \eqref{eq:app_3},
\begin{figure*}[ht]
    \centering
{\small{\begin{align}\label{eq:app_3}
    W_1(P^\prime({\mathbf{x}}^{k_0}|\mathbf{y}),\hat{P}({\mathbf{x}}^{k_0}|\mathbf{y}))
    \lesssim &\sqrt{\int_{k_0}^{K}\frac{1}{2}\!\int_{\mathbf{x}^k}P(\mathbf{x}^k|\mathbf{s}^{\star}_{\mathrm{c}}, \mathbf{a}^{\star})\left\|{{\varphi}(\mathbf{x}^k,\mathbf{s}^{\star}_{\mathrm{c}}, \mathbf{a}^{\star},k)-\nabla \log P(\mathbf{x}^k|\mathbf{s}^{\star}_{\mathrm{c}}, \mathbf{a}^{\star})}\right\|^2\!d \mathbf{x} d k}\\
    = &\sqrt{\frac{\int_{k_0}^{K}\mathbb{E}_{\mathbf{x}^k}\left[\left\|{{\varphi}(\mathbf{x}^k,\mathbf{s}^{\star}_{\mathrm{c}}, \mathbf{a}^{\star},k)\!-\!\nabla \log P(\mathbf{x}^k|\mathbf{s}^{\star}_{\mathrm{c}}, \mathbf{a}^{\star})}\right\|^2\right] d k}{\int_{k_0}^{K}\mathbb{E}_{\mathbf{x}^k,\mathbf{s},\mathbf{a}}\left[\left\|{{\varphi}(\mathbf{x}^k,\mathbf{s},\mathbf{a},k)\!-\!\nabla \log P(\mathbf{x}^k|\mathbf{s},\mathbf{a})}\right\|^2\right] d k}} \!\cdot\! \sqrt{\frac{K}{2}\mathcal{R}({\varphi})}
    \le \mathcal{T}(\mathbf{s}^{\star}_{\mathrm{c}}, \mathbf{a}^{\star})\sqrt{\frac{K}{2}\mathcal{R}({\varphi})}.\nonumber
\end{align}}}
\vspace{-0.5cm}
\end{figure*}
Besides, the distribution coefficient $\mathcal{T}(\mathbf{s}^{\star}_{\mathrm{c}}, \mathbf{a}^{\star})$ is aligned with the concentrability coefficient, i.e., $L_\infty$ density ratio in RL \cite{fan2020theoretical}. Owing to the smoothing property of the score network $\mathcal{F}$ that mitigates discrepancies between $(\mathbf{s}^{\star}_{\mathrm{c}}, \mathbf{a}^{\star})$ and the training data, $\mathcal{T}(\mathbf{s}^{\star}_{\mathrm{c}}, \mathbf{a}^{\star})$ is persistently lower than the standard concentrability coefficient.

We investigate the theoretical guarantees under the case of approximating the conditional score with a ReLU NN.
We derive the bounded covering number of ReLU NN functions by defining a truncated loss function $\ell^{\mathrm{tr}}(\mathbf{x}, \mathbf{y}; \varphi)$ as
{\small{\begin{equation*}
    \ell^{\mathrm{tr}}(\mathbf{x},\mathbf{y}; \varphi)\coloneqq\ell(\mathbf{x},\mathbf{y}; \varphi)\mathbb{I}\left\{\left\|{\mathbf{x}}\right\|_{\infty}\le R\right\}.
\end{equation*}}}
Following this, we define the truncated domain of the score function as $\mathcal{D}=[-R, R]^{w_x}\times \mathbb{R}^{w_y}\cup\left\{{\varnothing}\right\}$, and the associated truncated loss function can be written as
{\small{\begin{equation}
    \mathcal{S}(R)=\{{\ell(\cdot,\cdot; \varphi):\mathcal{D} \rightarrow \mathbb{R}|\mathbf{s} \in \mathcal{F}}\}.
\end{equation}}}

To derive the approximation guarantee with statistical theory, the loss function class's covering number $\mathcal{S}(R)$ is defined:
\begin{definition}
Let $\mathcal{N}(\varrho, \mathcal{F}, \left\|{\cdot}\right\|)$ denote the $\varrho$-covering number of the function class $\mathcal{F}$ under the norm $\left\|{\cdot}\right\|$, that is, 
\begin{align*}
    \mathcal{N}(\varrho, \mathcal{F},\left\|{\cdot}\right\|)=&\min \{N: \exists \left\{{f_{t}}\right\}_{t\in\mathcal{B}} \subseteq \mathcal{F}, \\
    & \text{s.t.}~ \forall f \in \mathcal{F}, \exists t \in [N],~ \left\|{f_{t}-f}\right\|\le \varrho\}.
\end{align*}
\end{definition}
We present the covering number associated with $\mathcal{S}(R)$ as:
\begin{lemma}\label{lemma::covering number S}
For any $\varrho>0$, under the condition $|\mathbf{x}|_{\infty} \le R$, the $\varrho$-covering number of $\mathcal{S}(R)$ under $\left\|{\cdot}\right\|_{L_{\infty}\mathcal{D}}$ is bounded as
    
{\small{\begin{equation*}
    \mathcal{N}\left(\varrho, \!\mathcal{S}(R), \left\|{\cdot}\right\|_{L_{\infty}\mathcal{D}}\right) \!\!\lesssim\!\! \left(\!\frac{2L^2(W\!\max(R,K)\!\!+\!\!2)\kappa^LW^{L\!+\!1}\!\log \!N}{\varrho}\!\right)^{\!\!2P}\!\!\!,
\end{equation*}}}
where the norm $\left\|{\cdot}\right\|_{L_{\infty}\mathcal{D}}$ is given by
{\small{\begin{align*}
    \left\|{f(\cdot,\cdot)}\right\|_{L_{\infty}\mathcal{D}}=\max\nolimits_{\mathbf{x} \in [-R,R]^{w_x},\mathbf{y} \in \mathbb{R}^{w_y}\cup\left\{{\varnothing}\right\}}\left|{f(\mathbf{x},\mathbf{y})}\right|.
\end{align*}}}
\end{lemma}
\begin{proof}
    See Appendix \ref{sec::proof of lemma::covering number S}.
\end{proof}

Following from {\cite[Thm 3.4]{DM_offline_9}} established upon{\cite[Assumption 3.1]{DM_offline_9}}, we have the key Lemmas as follows.
\begin{lemma}\label{lemma:approx_prob}
    Under Assumption~\ref{assump:distri_true}, for positive constant $C_{\alpha}$ and large enough $N$, by setting the terminal step $K=C_{\alpha}\log N$, $\mathbf{s} \in \mathcal{F}(M_t, W, \kappa, L, P)$ exists and satisfies that, for $\mathbf{y} \in \mathbb{R}^{w_y}$ and $k \in [0,K]$, one has
    {\small{\begin{align}
        &\int_{\mathbf{x}^{k_1}} \|{\zeta(\mathbf{x}^{k_1},\mathbf{y},{k_1})-\nabla\log P_{k_1}(\mathbf{x}^{k_1}|\mathbf{y})\|}^2_2 \cdot P_{k_1}(\mathbf{x}^{k_1}|\mathbf{y}) d\mathbf{x}^{k_1}\nonumber\\
        &+\int_{\mathbf{x}^{k_2}} \|{\zeta(\mathbf{x}^{k_2},\mathbf{y},{k_2})-\nabla\log p_{k_2}(\mathbf{x}^{k_2}|\mathbf{y})\|}^2_2 \cdot p_{k_2}(\mathbf{x}^{k}|\mathbf{y}) d\mathbf{x}^{k_2}\notag\\
    &=\mathcal{O} \left( \frac{B^2}{\sigma_k^2}\cdot N^{-\frac{2b}{w_x+w_y}} \cdot (\log N)^{b+1}\right),
    \end{align}}}
    where $\delta\le k_1\le K$ and $0\le k_2\le K$.
    % The hyperparameters in the ReLU neural network class $\mathcal{F}$.
    The ReLU NN class $\mathcal{F}$ is parameterized by hyperparameters that fulfill
    \begin{align*}
        & \hspace{0.4in} M_t = \mathcal{O}\left(\sqrt{\log N}/\sigma_t\right),~
         W = {\mathcal{O}}\left(N\log^7 N\right),\\
         & \kappa =\exp \left({\mathcal{O}}(\log^4 N)\right),~
        L = {\mathcal{O}}(\log^4 N),~
         P= {\mathcal{O}}\left(N\log^9 N\right).
    \end{align*}
\end{lemma}
\begin{proof}
    See Appendix~\ref{appendix:approx_error}.
\end{proof}

\begin{lemma}\label{proof_Rs}
    Under Assumption~\ref{assump:distri_true}, Lemmas~\ref{lemma::covering number S} and \ref{lemma:approx_prob}, given ReLU NN $\mathcal{F}(M_t, W, \kappa, L, P)$ as Lemma~\ref{lemma:approx_prob}, by setting the early-stopping step $k_0=n^{-\mathcal{O}(1)}$, the network size parameter $N=n^{\frac{1}{w_x+w_y+2b}}$, and terminal step $K=\mathcal{O}(\log n)$, the empirical loss has
    {\small{\begin{align}
        \label{equ::exp generalization unbounded y}
        \!\!\!\E\nolimits_{\left\{{\mathbf{x}_{t\!+\!1},\mathbf{y}_t}\right\}}\left[ \mathcal{R}({\varphi})\right] \!\!=\!\!\mathcal{O}\big(\!\log\! \frac{1}{k_0} n^{-\frac{2b}{2w_s\!+\!w_a\!+\!2b}}(\log n)^{\max(17,b) }\big).
    \end{align}}}
\end{lemma}
\begin{proof}
    See Appendix \ref{appendix:proof_Rs}. 
\end{proof}
Taking expectations w.r.t. $n$ samples $\left\{{\mathbf{x}_t,\mathbf{s}_t,\mathbf{a}_t}\right\}$ and applying (\ref{equ::exp generalization unbounded y}), we have 
% _{\left\{{\mathbf{x}_t\!,\hat{\mathbf{s}}_t\!,\mathbf{a}_t}\right\}_{t=1}^{n}}\!\!\!
{\small{\begin{align*}
    &\mathbb{E}\left[W_1(P^\prime({\mathbf{x}}_{t+1}^{k_0}|\hat{\mathbf{s}}_t,\mathbf{a}_t),\hat{P}({\mathbf{x}}_{t+1}^{k_0}|\hat{\mathbf{s}}_t,\mathbf{a}_t))\right]\\
    &\lesssim \mathcal{T}(\mathbf{s}^{\star}_{\mathrm{c}}, \mathbf{a}^{\star})\sqrt{\frac{K}{2}\frac{1}{k_0}n^{-\frac{2b}{2w_s+w_a+2b}}(\log n)^{\max(17,b)}}.
\end{align*}}}
Taking $k_0 = n^{-\frac{4b}{2w_s + w_a + 2b} - 1}$ and $K = \frac{2\beta}{2w_s + w_a + 2b}\log n$, the expected total variation can be bounded as
{\small{\begin{align*}
&\mathbb{E}\left[W_1(P^\prime({\mathbf{x}}_{t+1}^{k_0}|\hat{\mathbf{s}}_t,\mathbf{a}_t),\hat{P}({\mathbf{x}}_{t+1}^{k_0}|\hat{\mathbf{s}}_t,\mathbf{a}_t))\right]\\
&=\mathcal{T}(\mathbf{s}^{\star}_{\mathrm{c}}, \mathbf{a}^{\star})\mathcal{O}\left(n^{-\frac{2b}{2w_s+w_a+2b}}(\log n)^{\max(19/2,(b+2)/2)}\right).
\end{align*}}}
% _{\left\{{\mathbf{x}_t,\hat{\mathbf{s}}_t,\mathbf{a}_t}\right\}_{t=1}^{n}}

% \paragraph{Putting all this together}
Finally, the divergence between the estimated conditional distribution $\hat{P}(\mathbf{x}^{k_0}|\mathbf{y})$ and the true conditional data distribution $P(\mathbf{x}|\mathbf{y})$ can be bounded as
{\small{\begin{align*}
    &\mathbb{E}\left[W_1(P(\mathbf{x}|\mathbf{y}), \hat{P}({\mathbf{x}}^{k_0}|\mathbf{y}))\right]\\
    &\leq \mathcal{T}(\mathbf{s}^{\star}_{\mathrm{c}}, \mathbf{a}^{\star})O\left(n^{-\frac{2b}{2w_s+w_a+2b}}(\log n)^{\max(19/2,(b+2)/2)}\right).
\end{align*}}}
% _{\left\{{\mathbf{x}_t,\hat{\mathbf{s}}_t,\mathbf{a}_t}\right\}_{t=1}^{n}}
% This proof is complete.

% \subsubsection{Statement of Overall Convergence}\label{appendix:analysis_convergence}
The convergence of $\mathbb{E}\! \left[\left| V^\pi(\mathbf{s}_{\mathrm{c}}) \!-\! V^\pi(\hat{F}( {\mathbf{o}})) \right|\right]$ relies on the convergence of $\frac{\ln^{c_1} n}{n^{c_2}}$, since $c_1=6$ and $c_2=\frac{2p_R}{2p_R +w_s+1}>0$ in $\mathcal{O}(n^{-\frac{2p_R}{2p_R +w_s+1}}(\log n)^6)$ and $c_1 = \max\{\frac{19}{2},\frac{b+2}{2}\}>0$ and $c_2=\frac{b}{2w_s+w_a+2b}>0$ in $\frac{2C_{\mathrm{r}}+2C_{\mathrm{s}}}{1 - C_{\mathrm{s}}-C_{\mathrm{r}}} \mathcal{T}(\mathbf{s}^{\star}_{\mathrm{c}}, \mathbf{a}^{\star})\mathcal{O}\left(n^{-\frac{b}{2w_s+w_a+2b}} (\log n)^{\max(19/2,(b+2)/2)} \right)$. 
With L'Hôpital's rule, $\lim\limits_{n \to \infty} \frac{\ln^{c_1} n}{n^{c_2}} = \lim\limits_{n \to \infty} \frac{c_1\ln n}{c_2 n^{c_2}} = \lim\limits_{n \to \infty} \frac{c_1}{c_2^2n^{c_2}} = 0$, $\forall c_1,c_2>0$. Both  $\mathcal{O}(n^{-\frac{2p_R}{2p_R +w_s+1}}(\log n)^6)$ and $\frac{2C_{\mathrm{r}}+2C_{\mathrm{s}}}{1 - C_{\mathrm{s}}-C_{\mathrm{r}}} \mathcal{T}(\mathbf{s}^{\star}_{\mathrm{c}}, \mathbf{a}^{\star})\mathcal{O}\left(n^{-\frac{b}{2w_s+w_a+2b}} (\log n)^{\max\{19/2,(b+2)/2\}} \right)$  approach zero, as $n\rightarrow \infty$.
The value funtion of the estimated causal state $V^\pi( \zeta( \mathbf{s} ) )$ in \eqref{eq:final} converges to within $2\hat{\epsilon}$-neighborhood of the ground-truth causal state $V^\pi(\mathbf{s})$, i.e., the neighborhood region of $V^\pi(\mathbf{s})$ with the radius of $\hat{\epsilon}$.
% In other words, the asymptotic convergence of the proposed CaDiff algorithm is established, as $n \to \infty$.

\subsection{Proof of Lemma~\ref{thm:boundedness}}\label{thm:boundedness_proof}
The existence proof follows that of Remark~\ref{remark:rem1} in Appendix~\ref{appendix:remark}, only with ${P}$ replaced by $\widehat{{P}}$. Given that $\mathcal{S}_{\mathrm{c}}$ is compact by assumption, it follows that $\mathrm{supp}(\widehat{{P}}) \subseteq \mathcal{S}_{\mathrm{c}}$ is compact as \eqref{eq:app_8}.
\begin{figure*}[ht]
    \centering
{\small{\begin{align}\label{eq:app_8}
    d(\mathbf{s}_{\mathrm{c}, i}, \mathbf{s}_{\mathrm{c}, j}) -& d^\prime(\mathbf{s}_{\mathrm{c}, i}, \mathbf{s}_{\mathrm{c}, j})
    % = \max\nolimits_{\mathbf{a}\in\mathcal{A}}\left[C_{\mathrm{r}}W_1(d)\left(\widehat{{P}}\left(r\mid \mathbf{s}_{\mathrm{c}, i}, \mathbf{a}\right), \widehat{{P}}\left(r\mid \mathbf{s}_{\mathrm{c}, j}, \mathbf{a}\right)\right)+C_{\mathrm{s}} W_1(d)\left(\widehat{{P}}\left(\mathbf{s}'\mid \mathbf{s}_{\mathrm{c}, i}, \mathbf{a}\right), \widehat{{P}}\left(\mathbf{s}'\mid \mathbf{s}_{\mathrm{c}, j}, \mathbf{a}\right)\right)\right]\notag\\
    % -&\max\nolimits_{\mathbf{a}\in\mathcal{A}}\left[C_{\mathrm{r}}W_1(d')\left(\widehat{{P}}\left(r\mid \mathbf{s}_{\mathrm{c}, i}, \mathbf{a}\right), \widehat{{P}}\left(r\mid \mathbf{s}_{\mathrm{c}, j}, \mathbf{a}\right)\right)+C_{\mathrm{s}} W_1(d')\left(\widehat{{P}}\left(\mathbf{s}'\mid \mathbf{s}_{\mathrm{c}, i}, \mathbf{a}\right), \widehat{{P}}\left(\mathbf{s}'\mid \mathbf{s}_{\mathrm{c}, j}, \mathbf{a}\right)\right)\right]\notag\\
    \le\max\nolimits_{\mathbf{a}\in\mathcal{A}}C_{\mathrm{r}}\left[W_1(d)\left(\widehat{{P}}\left(r_{i+1}\mid \mathbf{s}_{\mathrm{c}, i}, \mathbf{a}\right), \widehat{{P}}\left(r_{j+1}\mid \mathbf{s}_{\mathrm{c}, j}, \mathbf{a}\right)\right)-W_1(d')\left(\widehat{{P}}\left(r_{i+1}\mid \mathbf{s}_{\mathrm{c}, i}, \mathbf{a}\right), \widehat{{P}}\left(r_{j+1}\mid \mathbf{s}_{\mathrm{c}, j}, \mathbf{a}\right)\right)\right]\notag\\
    &+\max\nolimits_{\mathbf{a}\in\mathcal{A}}C_{\mathrm{s}}\left[W_1(d)\left(\widehat{{P}}\left(\mathbf{s}'\mid \mathbf{s}_{\mathrm{c}, i}, \mathbf{a}\right), \widehat{{P}}\left(\mathbf{s}'\mid \mathbf{s}_{\mathrm{c}, j}, \mathbf{a}\right)\right)-W_1(d')\left(\widehat{{P}}\left(\mathbf{s}'\mid \mathbf{s}_{\mathrm{c}, i}, \mathbf{a}\right), \widehat{{P}}\left(\mathbf{s}'\mid \mathbf{s}_{\mathrm{c}, j}, \mathbf{a}\right)\right)\right]\notag\\
    % =&\max\nolimits_{\mathbf{a}\in\mathcal{A}}C_{\mathrm{r}}\left[W_1(d\!-\!d'\!+\!d')\!\left(\widehat{{P}}\left(r\!\mid\! \mathbf{s}_{\mathrm{c}, i}, \mathbf{a}\right), \widehat{{P}}\left(r\!\mid\! \mathbf{s}_{\mathrm{c}, j}, \mathbf{a}\right)\right)\!-\!W_1(d')\!\left(\widehat{{P}}\left(r\!\mid\! \mathbf{s}_{\mathrm{c}, i}, \mathbf{a}\right), \widehat{{P}}\left(r\!\mid\! \mathbf{s}_{\mathrm{c}, j}, \mathbf{a}\right)\right)\right]\notag\\
    % +&\max\nolimits_{\mathbf{a}\in\mathcal{A}}C_{\mathrm{s}}\left[W_1(d\!-\!d'\!+\!d')\!\left(\widehat{{P}}\left(\mathbf{s}'\!\mid\! \mathbf{s}_{\mathrm{c}, i}, \mathbf{a}\right)\!,\! \widehat{{P}}\left(\mathbf{s}'\!\mid\! \mathbf{s}_{\mathrm{c}, j}, \mathbf{a}\right)\right)\!-\!W_1(d')\!\left(\widehat{{P}}\left(\mathbf{s}'\!\mid\! \mathbf{s}_{\mathrm{c}, i}, \mathbf{a}\right)\!,\! \widehat{{P}}\left(\mathbf{s}'\!\mid\! \mathbf{s}_{\mathrm{c}, j}, \mathbf{a}\right)\right)\right]\notag\\
    \le& \max\nolimits_{\mathbf{a}\in\mathcal{A}}\!\Big[\!C_{\mathrm{r}}W_1\left\|d\!-\!d'\right\|_\infty\!\left(\!\widehat{{P}}\left(r_{i+1}\!\mid \mathbf{s}_{\mathrm{c}, i}, \mathbf{a}\right), \widehat{{P}}\left(r_{j+1}\!\mid \mathbf{s}_{\mathrm{c}, j}, \mathbf{a}\right)\!\right)\!+\!C_{\mathrm{s}}W_1\left\|d\!-\!d'\right\|_\infty\left(\widehat{{P}}\left(\mathbf{s}'\!\mid \mathbf{s}_{\mathrm{c}, i}, \mathbf{a}\right), \widehat{{P}}\left(\mathbf{s}'\!\mid \mathbf{s}_{\mathrm{c}, j}, \mathbf{a}\right)\!\right)\!\Big]\notag\\
    \le& \left(C_{\mathrm{r}}+C_{\mathrm{s}}\right)\left\|d-d'\right\|_\infty, ~\forall(\mathbf{s}_{\mathrm{c}, i}, \mathbf{s}_{\mathrm{c}, j}) \in \mathcal{S}_{\mathrm{c}} \times \mathcal{S}_{\mathrm{c}}.
\end{align}}}
\vspace{-0.5cm}
\end{figure*}
Hence, $\mathcal{F}$ satisfies the $(C_{\mathrm{r}} + C_{\mathrm{s}})$-contraction property.
Next, we prove the boundedness of the distance. With Lemma \ref{lemma:wass-diam} and $\mathrm{supp}(\widehat{{P}}) \subseteq \mathcal{S}_{\mathrm{c}}$, we have, $\forall p \geq 1$,
{\small{\begin{align*}
\sup\limits_{\mathbf{s}_{\mathrm{c}, i}, \mathbf{s}_{\mathrm{c}, j} \in \mathcal{S}_{\mathrm{c}} \times \mathcal{S}_{\mathrm{c}}} W_p(\widehat{d})(\widehat{{P}}^\pi(\cdot | \mathbf{s}_{\mathrm{c}, i}, \mathbf{a}), \widehat{{P}}^\pi(\cdot | \mathbf{s}_{\mathrm{c}, j}, \mathbf{a})) \leq \mathrm{diam}(\mathcal{S}_{\mathrm{c}}; \widehat{d}).
\end{align*}}}
% Following \cite{Ferns2011Bisimulation}, we have:
In turn, $\forall(\mathbf{s}_{\mathrm{c}, i}, \mathbf{s}_{\mathrm{c}, j}) \in \mathcal{S}_{\mathrm{c}} \times \mathcal{S}_{\mathrm{c}}$,
{\small{\begin{align*}
    {d}(\mathbf{s}_{\mathrm{c}, i}, \mathbf{s}_{\mathrm{c}, j}) =& \max_{\mathbf{a}\in\mathcal{A}}\Big(C_{\mathrm{r}}W_p(d)\left(P\left(r_{i+1}\mid \mathbf{s}_{\mathrm{c}, i}, \mathbf{a}\right), P\left(r_{j+1}\mid \mathbf{s}_{\mathrm{c}, j}, \mathbf{a}\right)\right)\\
    &+C_{\mathrm{s}} W_p(d)\left(P\left(\mathbf{s}'\mid \mathbf{s}_{\mathrm{c}, i}, \mathbf{a}\right), P\left(\mathbf{s}'\mid \mathbf{s}_{\mathrm{c}, j}, \mathbf{a}\right)\right)\Big)\\
     \leq& C_{\mathrm{r}}(r_{\mathrm{max}}-r_{\mathrm{min}}) + C_{\mathrm{s}} \mathrm{diam}(\mathcal{S}_\mathrm{c};{d})\\
     &\leq {C_{\mathrm{r}}}(r_{\mathrm{max}}-r_{\mathrm{min}})/({1-C_{\mathrm{s}}}),
\end{align*}}}

\noindent  which follows from Lemma~\ref{lemma:wass-diam} and gives the upper bound as $p \to \infty$. Subsequently,
% \begin{align*}
    $\mathrm{diam}(\mathcal{S}_\mathrm{c};{d}) 
    % & \leq C_{\mathrm{r}}(r_{\mathrm{max}}-r_{\mathrm{min}}) + C_{\mathrm{s}} \mathrm{diam}(\mathcal{S}_\mathrm{c};{d})\\
    \leq {C_{\mathrm{r}}}(r_{\mathrm{max}}-r_{\mathrm{min}})/({1-C_{\mathrm{s}}})$.
% \end{align*}
Likewise, $\forall(\mathbf{s}_{\mathrm{c}, i}, \mathbf{s}_{\mathrm{c}, j}) \in \mathcal{S}_{\mathrm{c}} \times \mathcal{S}_{\mathrm{c}}$,
% % \begin{align*}
% %     \widehat{d}(\mathbf{s}_{\mathrm{c}, i}, \mathbf{s}_{\mathrm{c}, j}) =& \max_{\mathbf{a}\in\mathcal{A}}\Big(C_{\mathrm{r}}\!W_p(\widehat{d})\left(\widehat{{P}}\left(r\mid \mathbf{s}_{\mathrm{c}, i}, \mathbf{a}\right), \widehat{{P}}\left(r\mid \mathbf{s}_{\mathrm{c}, j}, \mathbf{a}\right)\right)\\
% %     &+C_{\mathrm{s}} W_p(\widehat{d})\left(\widehat{{P}}\left(\mathbf{s}'\mid\mathbf{s}_{\mathrm{c}, i}, \mathbf{a}\right), \widehat{{P}}\left(\mathbf{s}'\mid \mathbf{s}_{\mathrm{c}, j}, \mathbf{a}\right)\right)\Big)\\
% %      \leq &C_{\mathrm{r}}(r_{\mathrm{max}}-r_{\mathrm{min}}) + C_{\mathrm{s}} \mathrm{diam}(\mathcal{S}; \widehat{d}),
% % \end{align*}
% % which implies that
% \begin{align*}
   $ \mathrm{diam}(\mathcal{S}_{\mathrm{c}}; \widehat{d})
    % &\leq C_{\mathrm{r}}(r_{\mathrm{max}}-r_{\mathrm{min}}) + C_{\mathrm{s}} \mathrm{diam}(\mathcal{S}_{\mathrm{c}}; \widehat{d})\\ 
    \leq {C_{\mathrm{r}}}(r_{\mathrm{max}}-r_{\mathrm{min}})/({1-C_{\mathrm{s}}})$.
% \end{align*}}

\subsection{Proof of Lemma~\ref{lemma:distanceerror}}\label{lemma:distanceerror_proof}
By the Wasserstein triangle inequality \cite{clement2008elementary}, we define the respective differences for rewards and transitions in \eqref{eq:supp:9:triineq}.
\begin{figure*}[ht]
    \centering
    {\small{\begin{subequations}\label{eq:supp:9:triineq}
\begin{align}
&\left|W_p({d})\left({{P}}\left(r_{i+1}\mid \mathbf{s}_{\mathrm{c}, i}, \mathbf{a}\right), {{P}}\left(r_{j+1}\mid \mathbf{s}_{\mathrm{c}, j}, \mathbf{a}\right)\right)-W_p({d})\left(\widehat{{P}}\left(r_{i+1}\mid {\mathbf{s}}_{\mathrm{c}, i}, \mathbf{a}\right), \widehat{{P}}\left(r_{j+1}\mid {\mathbf{s}_{\mathrm{c}, j}}, \mathbf{a}\right)\right)\right| 
\leq 2\mathcal{E}_{\phi};\\
&\left|W_p({d})\left({{P}}\left(\mathbf{s}'\mid \mathbf{s}_{\mathrm{c}, i}, \mathbf{a}\right), {{P}}\left(\mathbf{s}'\mid \mathbf{s}_{\mathrm{c}, j}, \mathbf{a}\right)\right)-W_p({d})\left(\widehat{{P}}\left(\mathbf{s}'\mid {\mathbf{s}_{\mathrm{c}, i}}, \mathbf{a}\right), \widehat{{P}}\left(\mathbf{s}'\mid {\mathbf{s}_{\mathrm{c}, j}}, \mathbf{a}\right)\right)\right| 
\leq 2\mathcal{E}_{\theta}
\end{align} 
    \end{subequations}}}
\vspace{-0.5cm}
\end{figure*}
Since $d^p$ is convex, we obtain
{\small{\begin{align}
    &W_p( || d - \widehat{d} ||_\infty + d)(\widehat{{P}}\left(\mathbf{s}'\mid {\mathbf{s}_{\mathrm{c}, i}}, \mathbf{a}\right), \widehat{{P}}\left(\mathbf{s}'\mid {\mathbf{s}_{\mathrm{c}, j}}, \mathbf{a}\right))\label{eq:supp:9:convexity}\\
    % =& (\inf\nolimits_{\omega \in \Omega}\mathbb{E}_{(\mathbf{s}_{\mathrm{c}, i}, \mathbf{s}_{\mathrm{c}, j}) \sim \omega}[( || d - \widehat{d} ||_\infty + d(\mathbf{s}_{\mathrm{c}, i}, \mathbf{s}_{\mathrm{c}, j}))^p]  )^{{1}/{p}} \nonumber\\
    \leq& (\inf\nolimits_{\omega \in \Omega} 2^{p-1} \mathbb{E}_{(\mathbf{s}_{\mathrm{c}, i}, \mathbf{s}_{\mathrm{c}, j}) \sim \omega}[( || d - \widehat{d} ||_\infty^p + d(\mathbf{s}_{\mathrm{c}, i}, \mathbf{s}_{\mathrm{c}, j})^p]  )^{{1}/{p}} \nonumber\\
    \leq& a_p ( || d - \widehat{d} ||_\infty^p \!+\! W_p^p( d)(\widehat{{P}}\left(\mathbf{s}'\mid \mathbf{s}_{\mathrm{c}, i}, \mathbf{a}\right), \widehat{{P}}\left(\mathbf{s}'\mid \mathbf{s}_{\mathrm{c}, j}, \mathbf{a}\right)) )^{{1}/{p}} \nonumber\\
    \leq& a_p ( [|| d - \widehat{d} ||_\infty \!+\! W_p( d)(\widehat{{P}}\left(\mathbf{s}'\mid \mathbf{s}_{\mathrm{c}, i}, \mathbf{a}\right), \widehat{{P}}\left(\mathbf{s}'\mid \mathbf{s}_{\mathrm{c}, j}, \mathbf{a}\right)) ]^p )^{1/p} \nonumber\\
    =& a_p (|| d - \widehat{d} ||_\infty + W_p(d)(\widehat{{P}}\left(\mathbf{s}'\mid \mathbf{s}_{\mathrm{c}, i}, \mathbf{a}\right), \widehat{{P}}\left(\mathbf{s}'\mid \mathbf{s}_{\mathrm{c}, j}, \mathbf{a}\right)) ).\notag 
\end{align}}}
Similarly, we obtain
{\small\begin{align*}
    &W_p( || d - \widehat{d} ||_\infty + d)(\widehat{{P}}\left(r_{i+1}\mid \mathbf{s}_{\mathrm{c}, i}, \mathbf{a}\right), \widehat{{P}}\left(r_{j+1}\mid \mathbf{s}_{\mathrm{c}, j}, \mathbf{a}\right))\nonumber\\
    % =&(\inf\nolimits_{\omega \in \Omega}\mathbb{E}_{(\mathbf{s}_{\mathrm{c}, i}, \mathbf{s}_{\mathrm{c}, j}) \sim \omega}[( || d - \widehat{d} ||_\infty + d(\mathbf{s}_{\mathrm{c}, i}, \mathbf{s}_{\mathrm{c}, j}))^p]  )^{{1}/{p}} \nonumber\\
    % \leq& (\inf\nolimits_{\omega \in \Omega} 2^{p-1} \mathbb{E}_{(\mathbf{s}_{\mathrm{c}, i}, \mathbf{s}_{\mathrm{c}, j}) \sim \omega}[( || d - \widehat{d} ||_\infty^p + d(\mathbf{s}_{\mathrm{c}, i}, \mathbf{s}_{\mathrm{c}, j})^p]  )^{{1}/{p}} \nonumber\\
    % \leq& a_p ( || d - \widehat{d} ||_\infty^p + W_p^p( d)(\widehat{{P}}\left(r\mid \mathbf{s}_{\mathrm{c}, i}, \mathbf{a}\right), \widehat{{P}}\left(r\mid \mathbf{s}_{\mathrm{c}, j}, \mathbf{a}\right)) )^{{1}/{p}} \nonumber\\
    % \leq& a_p ( [|| d - \widehat{d} ||_\infty + W_p( d)\left(\widehat{{P}}\left(r\mid \mathbf{s}_{\mathrm{c}, i}, \mathbf{a}\right), \widehat{{P}}\left(r\mid \mathbf{s}_{\mathrm{c}, j}, \mathbf{a}\right)\right) ]^p )^{1/p} \nonumber\\
    \le& a_p (|| d - \widehat{d} ||_\infty + W_p(d)(\widehat{{P}}\left(r_{i+1}\mid \mathbf{s}_{\mathrm{c}, i}, \mathbf{a}\right), \widehat{{P}}\left(r_{j+1}\mid \mathbf{s}_{\mathrm{c}, j}, \mathbf{a}\right)) ). 
\end{align*}}
Due to Lemma \ref{lemma:wass-diam}, we have:
{\small{\begin{subequations}\label{eq:supp:9:Wdiam}
\begin{align}
    &W_p(d)(\widehat{{P}}\left(\mathbf{s}'\mid \mathbf{s}_{\mathrm{c}, i}, \mathbf{a}\right), \widehat{{P}}\left(\mathbf{s}'\mid \mathbf{s}_{\mathrm{c}, j}, \mathbf{a}\right)) \leq \mathrm{diam}(\mathcal{S}_\mathrm{c};{d})\\
    &W_p(d)(\widehat{{P}}\left(r_{i+1}\mid \mathbf{s}_{\mathrm{c}, i}, \mathbf{a}\right), \widehat{{P}}\left(r_{j+1}\mid \mathbf{s}_{\mathrm{c}, j}, \mathbf{a}\right)) \leq \mathrm{diam}(\mathcal{S}_\mathrm{c};{d}).
\end{align}
\end{subequations}}}
The difference in distance can be bounded in \eqref{eq:sup1},
\begin{figure*}[ht]
    \centering
{\small{\begin{align}\label{eq:sup1}
    \bigl| W_p(d)&\left({{P}}\left(\mathbf{s}'\mid \mathbf{s}_{\mathrm{c}, i}, \mathbf{a}\right), {{P}}\left(\mathbf{s}'\mid \mathbf{s}_{\mathrm{c}, j}, \mathbf{a}\right)\right) - W_p( \widehat{d})\left(\widehat{{P}}\left(\mathbf{s}'\mid \mathbf{s}_{\mathrm{c}, i}, \mathbf{a}\right), \widehat{{P}}\left(\mathbf{s}'\mid \mathbf{s}_{\mathrm{c}, j}, \mathbf{a}\right)\right) \bigr|\notag\\
    \leq& \bigl| W_p( \widehat{d})\left(\widehat{{P}}\left(\mathbf{s}'\mid \mathbf{s}_{\mathrm{c}, i}, \mathbf{a}\right), \widehat{{P}}\left(\mathbf{s}'\mid \mathbf{s}_{\mathrm{c}, j}, \mathbf{a}\right)\right) - W_p( {d})\left(\widehat{{P}}\left(\mathbf{s}'\mid \mathbf{s}_{\mathrm{c}, i}, \mathbf{a}\right), \widehat{{P}}\left(\mathbf{s}'\mid \mathbf{s}_{\mathrm{c}, j}, \mathbf{a}\right)\right) \bigr|\\
    & + \bigl| W_p(d)\left({{P}}\left(\mathbf{s}'\mid \mathbf{s}_{\mathrm{c}, i}, \mathbf{a}\right), {{P}}\left(\mathbf{s}'\mid \mathbf{s}_{\mathrm{c}, j}, \mathbf{a}\right)\right) - W_p( {d})\left(\widehat{{P}}\left(\mathbf{s}'\mid \mathbf{s}_{\mathrm{c}, i}, \mathbf{a}\right), \widehat{{P}}\left(\mathbf{s}'\mid \mathbf{s}_{\mathrm{c}, j}, \mathbf{a}\right)\right) \bigr| \notag\\
    \leq& \bigl| W_p( \widehat{d})\left(\widehat{{P}}\left(\mathbf{s}'\mid \mathbf{s}_{\mathrm{c}, i}, \mathbf{a}\right), \widehat{{P}}\left(\mathbf{s}'\mid \mathbf{s}_{\mathrm{c}, j}, \mathbf{a}\right)\right) - W_p( {d})\left(\widehat{{P}}\left(\mathbf{s}'\mid \mathbf{s}_{\mathrm{c}, i}, \mathbf{a}\right), \widehat{{P}}\left(\mathbf{s}'\mid \mathbf{s}_{\mathrm{c}, j}, \mathbf{a}\right)\right) \bigr| + 2 \mathcal{E}_{\theta} \notag\\
    % =& \bigl| W_p( \widehat{d}-d+d)\left(\widehat{{P}}\left(\mathbf{s}'\mid \mathbf{s}_{\mathrm{c}, i}, \mathbf{a}\right), \widehat{{P}}\left(\mathbf{s}'\mid \mathbf{s}_{\mathrm{c}, j}, \mathbf{a}\right)\right) - W_p( {d})\left(\widehat{{P}}\left(\mathbf{s}'\mid \mathbf{s}_{\mathrm{c}, i}, \mathbf{a}\right), \widehat{{P}}\left(\mathbf{s}'\mid \mathbf{s}_{\mathrm{c}, j}, \mathbf{a}\right)\right) \bigr| + 2 \mathcal{E}_{\theta} \notag\\
    \leq& \bigl| W_p( \|\widehat{d}-d\|_\infty +d)\left(\widehat{{P}}\left(\mathbf{s}'\mid \mathbf{s}_{\mathrm{c}, i}, \mathbf{a}\right), \widehat{{P}}\left(\mathbf{s}'\mid \mathbf{s}_{\mathrm{c}, j}, \mathbf{a}\right)\right) - W_p( {d})\left(\widehat{{P}}\left(\mathbf{s}'\mid \mathbf{s}_{\mathrm{c}, i}, \mathbf{a}\right), \widehat{{P}}\left(\mathbf{s}'\mid \mathbf{s}_{\mathrm{c}, j}, \mathbf{a}\right)\right) \bigr| + 2 \mathcal{E}_{\theta} \notag\\
    % =& \bigl| W_p( \|d-\widehat{d}\|_\infty +d)\left(\widehat{{P}}\left(\mathbf{s}'\mid \mathbf{s}_{\mathrm{c}, i}, \mathbf{a}\right), \widehat{{P}}\left(\mathbf{s}'\mid \mathbf{s}_{\mathrm{c}, j}, \mathbf{a}\right)\right)- W_p( {d})\left(\widehat{{P}}\left(\mathbf{s}'\mid \mathbf{s}_{\mathrm{c}, i}, \mathbf{a}\right), \widehat{{P}}\left(\mathbf{s}'\mid \mathbf{s}_{\mathrm{c}, j}, \mathbf{a}\right)\right) \bigr| + 2 \mathcal{E}_{\theta} \notag\\
    \leq & \bigl| a_p \| d - \widehat{d} \|_\infty + a_p W_p( {d})\left(\widehat{{P}}\left(\mathbf{s}'\mid \mathbf{s}_{\mathrm{c}, i}, \mathbf{a}\right), \widehat{{P}}\left(\mathbf{s}'\mid \mathbf{s}_{\mathrm{c}, j}, \mathbf{a}\right)\right)- W_p( {d})\left(\widehat{{P}}\left(\mathbf{s}'\mid \mathbf{s}_{\mathrm{c}, i}, \mathbf{a}\right), \widehat{{P}}\left(\mathbf{s}'\mid \mathbf{s}_{\mathrm{c}, j}, \mathbf{a}\right)\right) \bigr| + 2 \mathcal{E}_{\theta}\notag\\
    \leq & a_p \| d - \widehat{d} \|_\infty  + (a_p - 1)\mathrm{diam}(\mathcal{S}_\mathrm{c};{d}) + 2 \mathcal{E}_{\theta}
\end{align}}}
\vspace{-0.5cm}
\end{figure*}
where the second inequality holds due to (\ref{eq:supp:9:triineq}), the penultimate inequality inherits from (\ref{eq:supp:9:convexity}), and the last one comes from (\ref{eq:supp:9:Wdiam}).
Similarly, 
\begin{align}\label{eq:sup2}
    &\bigl| W_p(d)\left({{P}}\left(r_{i+1}\mid \mathbf{s}_{\mathrm{c}, i}, \mathbf{a}\right)\!, {{P}}\left(r_{j+1}\mid \mathbf{s}_{\mathrm{c}, j}, \mathbf{a}\right)\right) \notag\\
    &\qquad\qquad- W_p( \widehat{d})(\widehat{{P}}\left(r_{i+1}\mid \mathbf{s}_{\mathrm{c}, i}, \mathbf{a}\right)\!, \widehat{{P}}\left(r_{j+1}\mid \mathbf{s}_{\mathrm{c}, j}, \mathbf{a}\right)) \bigr|\notag\\
    \leq& a_p \| d - \widehat{d} \|_\infty  + (a_p - 1)\mathrm{diam}(\mathcal{S}_\mathrm{c};{d}) + 2 \mathcal{E}_{\phi}.
\end{align} 
Plugging (\ref{eq:sup1}) and (\ref{eq:sup2}) into the difference between the approximate and the true bisimulation distances yields 
{\small{\begin{align*}
    &|d(\mathbf{s}_{\mathrm{c}, i}, \mathbf{s}_{\mathrm{c}, j})  -  \widehat{d}(\mathbf{s}_{\mathrm{c}, i}, \mathbf{s}_{\mathrm{c}, j})|\\
    \leq& \max\nolimits_{\mathbf{a}\in\mathcal{A}}\Big(C_{\mathrm{r}}\Big|W_p({d})\left({{P}}\left(r_{i+1}\mid \mathbf{s}_{\mathrm{c}, i}, \mathbf{a}\right), {{P}}\left(r_{j+1}\mid \mathbf{s}_{\mathrm{c}, j}, \mathbf{a}\right)\right)\\
    &-W_p(\widehat{d})\left(\widehat{{P}}\left(r_{i+1}\mid \mathbf{s}_{\mathrm{c}, i}, \mathbf{a}\right), \widehat{{P}}\left(r_{j+1}\mid \mathbf{s}_{\mathrm{c}, j}, \mathbf{a}\right)\right)\Big|\Big)\\
    &+\max\nolimits_{\mathbf{a}\in\mathcal{A}}\Big(C_{\mathrm{s}} \Big|W_p({d})\left({{P}}\left(\mathbf{s}'\mid \mathbf{s}_{\mathrm{c}, i}, \mathbf{a}\right), {{P}}\left(\mathbf{s}'\mid \mathbf{s}_{\mathrm{c}, j}, \mathbf{a}\right)\right)\\
    &-W_p(\widehat{d})\left(\widehat{{P}}\left(\mathbf{s}'\mid \mathbf{s}_{\mathrm{c}, i}, \mathbf{a}\right), \widehat{{P}}\left(\mathbf{s}'\mid \mathbf{s}_{\mathrm{c}, j}, \mathbf{a}\right)\right)\Big|\Big) \\
    \leq& C_{\mathrm{r}} \left| a_p \| d - \widehat{d} \|_\infty  + (a_p - 1) \mathrm{diam}(\mathcal{S}_\mathrm{c};{d})  + 2 \mathcal{E}_{\phi} \right|\\
    &+ C_{\mathrm{s}}\left| a_p \| d - \widehat{d} \|_\infty  + (a_p - 1) \mathrm{diam}(\mathcal{S}_\mathrm{c};{d})  + 2 \mathcal{E}_{\theta} \right|.
\end{align*}}}
In other words, we have
{\small{\begin{align*}
    \| d - \widehat{d} \|_\infty &\leq 2C_{\mathrm{r}} \mathcal{E}_{\phi} + {2C_{\mathrm{s}}} \mathcal{E}_{\theta} + (C_{\mathrm{r}}+C_{\mathrm{s}}) a_p \| d - \widehat{d} \|_\infty + \\
    &(C_{\mathrm{r}}+C_{\mathrm{s}})(a_p - 1) \mathrm{diam}(\mathcal{S}_\mathrm{c};{d}).
\end{align*}}}
Applying the supremum over states, we derive
{\small{\begin{align*}
    \| d - \widehat{d} \|_\infty &\leq [{ (C_{\mathrm{r}}+C_{\mathrm{s}})(a_p - 1) }\mathrm{diam}(\mathcal{S}_\mathrm{c};{d}) \\
    &+ {2C_{\mathrm{r}}} \mathcal{E}_{\phi} + {2C_{\mathrm{s}}} \mathcal{E}_{\theta}]/[{1 - (C_{\mathrm{r}}+C_{\mathrm{s}})a_p}] .
\end{align*}}}

\subsection{Proof of Lemma~\ref{thm:generalized-ct-vfa}}
\label{appendix:generalized-ct-vfa}
Define $\xi$ as the measure over $\mathcal{S}$. Consider a partition $\hat{F}(\mathbf{o}) \in \mathcal{S}_{\mathrm{c}}$, representing a subset of $\mathcal{S}_{\mathrm{c}}$ formed by clustering points within an $\epsilon$-neighborhood such that $\xi(\hat{F}(\mathbf{o})) > 0$.
In analogy with the $\xi$-average finite MDP formulation in~\cite[Theorem~3.21]{Ferns2011Bisimulation}, we define the reward and dynamics of the $\xi$-average finite POMDP as follows:
{\small{\begin{align*}
    \widetilde{P}(r|\hat{F}({\mathbf{o}}), \mathbf{a}) &= \frac{1}{\xi(\hat{F}({\mathbf{o}}))}\int\nolimits_{\mathbf{z} \in \hat{F}({\mathbf{o}})} {P}(r| \mathbf{z}, \mathbf{a})d\xi(\mathbf{z}); \\
    \widetilde{P}(\hat{F}({\mathbf{o}}^\prime)|\hat{F}({\mathbf{o}}), \mathbf{a}) &= \frac{1}{\xi(\hat{F}({\mathbf{o}}))}\int\nolimits_{\mathbf{z} \in \hat{F}({\mathbf{o}})} {P}(\hat{F}({\mathbf{o}^\prime})| \mathbf{z}, \mathbf{a})d\xi(\mathbf{z}).
\end{align*}}}
Then, we obtain \eqref{eq:app_6},
\begin{figure*}[ht]
    \centering
{\small{\begin{align}\label{eq:app_6}
        |V(\mathbf{s}_{\mathrm{c}}) - V(\hat{F}({\mathbf{o}}))|
        % =& \bigg| \max_{\mathbf{a}\in\mathcal{A}}\left(\int_{r\in \mathcal{R}}r\left(\mathbf{s}_{\mathrm{c}},\mathbf{a}\right)P\left(r\mid \mathbf{s}_{\mathrm{c}}, \mathbf{a}\right)dr + \gamma \int_{\mathbf{s}_{\mathrm{c}}^\prime \in \mathcal{S}_{\mathrm{c}}}P\left(\mathbf{s}'\mid \mathbf{s}_{\mathrm{c}}, \mathbf{a}\right)V(\mathbf{s}^\prime)d\mathbf{s}^\prime\right)\notag\\
        % &-\max_{\mathbf{a}\in\mathcal{A}}\bigg(\int_{r\in \mathcal{R}}r\left(\hat{F}({\mathbf{o}}),\mathbf{a}\right)\widetilde{P}(r|\hat{F}({\mathbf{o}}), \mathbf{a})dr+ \gamma \int_{\hat{F}({\mathbf{o}}^\prime) \in {\mathcal{S}_{\mathrm{c}}}}\widetilde{P}(\hat{F}({\mathbf{o}}^\prime)|\hat{F}({\mathbf{o}}), \mathbf{a})V(\hat{F}({\mathbf{o}}^\prime))d\hat{F}({\mathbf{o}}^\prime)\bigg)\bigg| \nonumber\\
    \le & \bigg| \max_{\mathbf{a}\in\mathcal{A}}\Big[\int_{r\in \mathcal{R}}\left(r\left(\mathbf{s}_{\mathrm{c}},\mathbf{a}\right)P\left(r\mid \mathbf{s}_{\mathrm{c}}, \mathbf{a}\right)-r\left(\hat{F}({\mathbf{o}}),\mathbf{a}\right)\widetilde{P}(r|\hat{F}({\mathbf{o}}), \mathbf{a})\right)dr\notag\\
    &+ \gamma \Big(\int_{\mathbf{s}^\prime \in \mathcal{S}}P\left(\mathbf{s}'\mid \mathbf{s}_{\mathrm{c}}, \mathbf{a}\right)V(\mathbf{s}^\prime)d\mathbf{s}^\prime-\int_{\hat{F}({\mathbf{o}}^\prime) \in {\mathcal{S}}_{\mathrm{c}}}\widetilde{P}(\hat{F}({\mathbf{o}}^\prime)|\hat{F}({\mathbf{o}}), \mathbf{a})V(\hat{F}({\mathbf{o}}^\prime))d\hat{F}({\mathbf{o}}^\prime)\Big)\Big]\bigg| \nonumber\\
    \le & \max_{\mathbf{a}\in\mathcal{A}}\left| \int_{r\in \mathcal{R}}\left(r\left(\mathbf{s}_{\mathrm{c}},\mathbf{a}\right)P\left(r\mid \mathbf{s}_{\mathrm{c}}, \mathbf{a}\right)-r\left(\hat{F}({\mathbf{o}}),\mathbf{a}\right)\widetilde{P}(r|\hat{F}({\mathbf{o}}), \mathbf{a})\right)dr\right|\notag\\
    &+ \max_{\mathbf{a}\in\mathcal{A}}\gamma \left|\int_{\mathbf{s}^\prime \in \mathcal{S}}P\left(\mathbf{s}'\mid \mathbf{s}_{\mathrm{c}}, \mathbf{a}\right)V(\mathbf{s}^\prime)d\mathbf{s}^\prime-\int_{\hat{F}({\mathbf{o}}^\prime) \in {\mathcal{S}_{\mathrm{c}}}}\widetilde{P}(\hat{F}({\mathbf{o}}^\prime)|\hat{F}({\mathbf{o}}), \mathbf{a})V(\hat{F}({\mathbf{o}}^\prime))d\hat{F}({\mathbf{o}}^\prime)\right|
\end{align}}}
\vspace{-0.5cm}
\end{figure*}
% and, in turn, \eqref{eq:app_7},
\begin{figure*}[ht]
    \centering
{\small{\begin{align}\label{eq:app_7}
    A_1 =& \left| \int_{r\in \mathcal{R}}\left(r\left(\mathbf{s}_{\mathrm{c}},\mathbf{a}\right)P\left(r\mid \mathbf{s}_{\mathrm{c}}, \mathbf{a}\right)-r\left(\hat{F}({\mathbf{o}}),\mathbf{a}\right)\widetilde{P}(r|\hat{F}({\mathbf{o}}), \mathbf{a})\right)dr\right|\nonumber\\
    \le & \frac{1}{\xi(\hat{F}({\mathbf{o}}))}\int_{\mathbf{z} \in \hat{F}({\mathbf{o}})}\left| \int_{r\in \mathcal{R}}\left(r\left(\mathbf{s}_{\mathrm{c}},\mathbf{a}\right)P\left(r\mid \mathbf{s}_{\mathrm{c}}, \mathbf{a}\right)-r\left(\hat{F}({\mathbf{o}}),\mathbf{a}\right)\widetilde{P}(r|\hat{F}({\mathbf{o}}), \mathbf{a})\right)dr\right|d\xi(\mathbf{z})\\
    \le & \frac{C_{\mathrm{r}}^{-1}}{\xi(\hat{F}({\mathbf{o}}))}\int_{\mathbf{z} \in \hat{F}({\mathbf{o}})} C_{\mathrm{r}} W_1(d)\left(P\left(r\mid \mathbf{s}_{\mathrm{c}}, \mathbf{a}\right),{P}\left(r|\mathbf{z}, \mathbf{a}\right)\right) d\xi(\mathbf{z})
    \le \frac{C_{\mathrm{r}}^{-1}}{\xi(\hat{F}({\mathbf{o}}))}\int_{\mathbf{z} \in \hat{F}({\mathbf{o}})} C_{\mathrm{r}} W_p(d)\left(P\left(r\mid \mathbf{s}_{\mathrm{c}}, \mathbf{a}\right),{P}\left(r|\mathbf{z}, \mathbf{a}\right)\right) d\xi(\mathbf{z})\nonumber
\end{align}}}
\vspace{-0.5cm}
\end{figure*}
where the penultimate inequality holds because $r(\mathbf{s}_{\mathrm{c}}, \mathbf{a})$ is 1-Lipschitz and the $W_1$ metric has dual form; the last line is due to Lemma~\ref{lemma:wass-lemma}. Similarly, we obtain \eqref{eq:app_7}.
\begin{figure*}[ht]
    \centering
{\small{\begin{align}
    A_2 =&  \gamma \left|\int_{\mathbf{s}^\prime \in \mathcal{S}}P\left(\mathbf{s}'\mid \mathbf{s}_{\mathrm{c}}, \mathbf{a}\right)V(\mathbf{s}^\prime)d\mathbf{s}^\prime-\int_{\hat{F}({\mathbf{o}}^\prime) \in {\mathcal{S}_{\mathrm{c}}}}\widetilde{P}(\hat{F}({\mathbf{o}}^\prime)|\hat{F}({\mathbf{o}}), \mathbf{a})V(\hat{F}({\mathbf{o}}^\prime))d\hat{F}({\mathbf{o}}^\prime)\right|\nonumber\\
    % \le & \frac{\gamma }{\xi(\hat{F}({\mathbf{o}}))}\int_{\mathbf{z} \in \hat{F}({\mathbf{o}})}\Bigg|\int_{\mathbf{s}^\prime \in \mathcal{S}}P\left(\mathbf{s}'\mid \mathbf{s}_{\mathrm{c}}, \mathbf{a}\right)V(\mathbf{s}^\prime)d\mathbf{s}^\prime-\int_{\hat{F}({\mathbf{o}}^\prime) \in {\mathcal{S}_{\mathrm{c}}}}{P}(\hat{F}({\mathbf{o}}^\prime)|\mathbf{z}, \mathbf{a})V(\hat{F}({\mathbf{o}}^\prime))d\hat{F}({\mathbf{o}}^\prime)\Bigg|d\xi(\mathbf{z})\nonumber\\
    % \le & \frac{\gamma }{\xi(\hat{F}({\mathbf{o}}))}\int\limits_{\mathbf{z} \in \hat{F}({\mathbf{o}})}\left|\int_{\mathbf{s}^\prime \in \mathcal{S}}\left(P\left(\mathbf{s}'\mid \mathbf{s}_{\mathrm{c}}, \mathbf{a}\right)V(\mathbf{s}^\prime)-{P}(\hat{F}({\mathbf{o}}^\prime)|\mathbf{z}, \mathbf{a})V(\hat{F}({\mathbf{o}}^\prime))\right)d\mathbf{s}^\prime\right|d\xi(\mathbf{z})\nonumber\\
    \le & \frac{\gamma }{\xi(\hat{F}({\mathbf{o}}))}\int\limits_{\mathbf{z} \in \hat{F}({\mathbf{o}})}\left|\int_{\mathbf{s}^\prime \in \mathcal{S}}\left(P\left(\mathbf{s}'\mid \mathbf{s}_{\mathrm{c}}, \mathbf{a}\right)V(\mathbf{s}^\prime)-{P}(\mathbf{s}^\prime|\mathbf{z}, \mathbf{a})V(\mathbf{s}^\prime)\right)d\mathbf{s}^\prime\right|d\xi(\mathbf{z})\nonumber\\
    & + \frac{\gamma }{\xi(\hat{F}({\mathbf{o}}))}\int\limits_{\mathbf{z} \in \hat{F}({\mathbf{o}})}\left|\int_{\mathbf{s}^\prime \in \mathcal{S}}\left(P\left(\hat{F}({\mathbf{o}}^\prime)\mid \mathbf{z}, \mathbf{a}\right)\left(V(\mathbf{s}^\prime)-V(\hat{F}({\mathbf{o}}^\prime))\right)\right)d\mathbf{s}^\prime\right|d\xi(\mathbf{z})\\
    \le& \frac{\gamma }{\xi(\hat{F}({\mathbf{o}}))}\int_{\mathbf{z} \in \hat{F}({\mathbf{o}})}\left|\int_{\mathbf{s}^\prime \in \mathcal{S}}\left(P\left(\mathbf{s}'\mid \mathbf{s}_{\mathrm{c}}, \mathbf{a}\right)-{P}\left(\mathbf{s}^\prime|\mathbf{z}, \mathbf{a}\right)\right)V(\mathbf{s}^\prime)d\mathbf{s}^\prime\right|d\xi(\mathbf{z})+ \left\|V-V\right\|_\infty\nonumber\\
    \le & \frac{C_{\mathrm{r}}^{-1}}{\xi(\hat{F}({\mathbf{o}}))}\int_{\mathbf{z} \in \hat{F}({\mathbf{o}})} C_{\mathrm{s}} \left| \int_{\mathbf{s}^\prime \in \mathcal{S}}  \left(P\left(\mathbf{s}'\mid \mathbf{s}_{\mathrm{c}}, \mathbf{a}\right)-{P}\left(\mathbf{s}^\prime|\mathbf{z}, \mathbf{a}\right)\right)\frac{C_{\mathrm{r}}\gamma}{C_{\mathrm{s}}}V(\mathbf{s}^\prime) d\mathbf{s}^\prime \right| d\xi(\mathbf{z}) \!+\! \gamma \left\|V - V\right\|_\infty\nonumber\\
    \le & \frac{C_{\mathrm{r}}^{-1}}{\xi(\hat{F}({\mathbf{o}}))}\int_{\mathbf{z} \in \hat{F}({\mathbf{o}})} C_{\mathrm{s}} W_1(d)\left(P\left(\mathbf{s}'\mid \mathbf{s}_{\mathrm{c}}, \mathbf{a}\right),{P}\left(\mathbf{s}^\prime|\mathbf{z}, \mathbf{a}\right)\right) d\xi(\mathbf{z}) + \gamma \left\|V - V\right\|_\infty\nonumber\\
    \le & \frac{C_{\mathrm{r}}^{-1}}{\xi(\hat{F}({\mathbf{o}}))}\int_{\mathbf{z} \in \hat{F}({\mathbf{o}})} C_{\mathrm{s}} W_p(d)\left(P\left(\mathbf{s}'\mid \mathbf{s}_{\mathrm{c}}, \mathbf{a}\right),{P}\left(\mathbf{s}^\prime|\mathbf{z}, \mathbf{a}\right)\right) d\xi(\mathbf{z}) + \gamma \left\|V - V\right\|_\infty,
\end{align}}}
\vspace{-0.5cm}
\end{figure*}
where the second inequality holds since $\left\|\cdot\right\|_\infty$ defined the supremum norm over $\mathcal{S}$, the third inequality holds because $\frac{C_{\mathrm{r}} \gamma}{C_{\mathrm{s}}}V(\mathbf{s}_{\mathrm{c}})$ is 1-Lipschitz, together with the dual form of the $W_1$ metric, and the last inequality is due to Lemma~\ref{lemma:wass-lemma}. Hence,
{\small{\begin{align}
    &|V(\mathbf{s}_{\mathrm{c}}) - V(\hat{F}({\mathbf{o}}))| \le \max_{\mathbf{a}\in\mathcal{A}}(A_1+A_2)\nonumber\\
    \leq & \frac{C_{\mathrm{r}}^{-1}}{\xi(\hat{F}({\mathbf{o}}))}\int_{\mathbf{z} \in \hat{F}({\mathbf{o}})}d(\mathbf{s}_{\mathrm{c}}, \mathbf{z})d\xi(\mathbf{z}) + \gamma \left\|V - V\right\|_\infty \notag\\
    \leq & C_{\mathrm{r}}^{-1} 2\epsilon + \gamma\left\|V - V\right\|_\infty.
\end{align}}}
Hence, applying the supremum over states:
{\small{\begin{align}
    |V(\mathbf{s}_{\mathrm{c}}) - V(\hat{F}({\mathbf{o}}))| \leq {2\epsilon}/{(C_{\mathrm{r}}(1-\gamma))}, ~\forall \mathbf{s} \in \mathcal{S}.
\end{align}}}

\subsection{Proof of Lemma \ref{lemma::covering number S}}\label{sec::proof of lemma::covering number S}

% Based on {\color{red}Lemma}~\ref{lemma:DM_offline_9}, the input $(\mathbf{x},\mathbf{y},t)$ can be uniformly bounded by $\mathcal{O}(\log N)$ \cite{DM_offline_9}.
% Now, we begin by proving Lemma \ref{lemma::covering number S}.
For any two ReLU NNs, $\varphi_1$ and $\varphi_2$, with $\left\|{\varphi_1-\varphi_2}\right\|_{L_{\infty}\mathcal{D}}\le \epsilon$, the $L_{\infty}$ error is bounded between $\ell(\cdot,\cdot,\varphi_1)$ and $\ell(\cdot,\cdot,\varphi_2)$. For any $(\mathbf{x},\mathbf{y})\in \mathcal{D}$, we have \eqref{equ:: lip loss},
\begin{figure*}[ht]
\centering
{\small{\begin{align}\label{equ:: lip loss}
    &\left|{\ell(\mathbf{x},\mathbf{y},\varphi_1)\!-\!\ell(\mathbf{x},\mathbf{y},\varphi_2)}\right|
    % \nonumber\\
    \le \int_{k_0}^{K}\!\!\!\!\frac{1}{K\!\!-\!k_0}\E\limits_{\tau,\mathbf{x}^k|\mathbf{x}^0=\hat{\mathbf{x}}^0}\Big[\left(\varphi_1(\mathbf{x}^k,\tau\mathbf{y},k)\!\!-\!\!\varphi_2(\mathbf{x}^k,\tau\mathbf{y},k)\right)^{\top} \Big(\varphi_1(\mathbf{x}^k,\tau\mathbf{y},k)\!\!+\!\!\varphi_2(\mathbf{x}^k,\tau\mathbf{y},k)\!\!-\!\!2P(\mathbf{x}^k|\hat{\mathbf{x}}^0)\Big) \Big]d k \nonumber\\
    &\qquad\qquad\qquad\qquad+\int_{\delta}^{K}\frac{1}{K-\delta}\E\limits_{\tau,\mathbf{x}^k|\mathbf{x}^\delta={\mathbf{x}}}\Big[\left(\varphi_1(\mathbf{x}^k,\tau\mathbf{y},k)-\varphi_2(\mathbf{x}^k,\tau\mathbf{y},k)\right)^{\top} \Big(\varphi_1(\mathbf{x}^k,\tau\mathbf{y},k)+\varphi_2(\mathbf{x}^k,\tau\mathbf{y},k)-2P(\mathbf{x}^k|{\mathbf{x}^\delta})\Big) \Big]d k \nonumber\\
    &\qquad\qquad\qquad\lesssim \epsilon\int_{k_0}^{K}\frac{1}{K-k_0}\E\limits_{\tau,\mathbf{x}^k|\mathbf{x}^0=\hat{\mathbf{x}}^0}\left[\left\|\varphi_1(\mathbf{x}^k,\tau\mathbf{y},k)+\varphi_2(\mathbf{x}^k,\tau\mathbf{y},k)-2P(\mathbf{x}^k|\hat{\mathbf{x}}^0)\right\|\right]d k \nonumber\\
    &\qquad\qquad\qquad\qquad+\epsilon\int_{\delta}^{K}\frac{1}{K-\delta}\E\limits_{\tau,\mathbf{x}^k|\mathbf{x}^\delta={\mathbf{x}}}\left[\left\|\varphi_1(\mathbf{x}^k,\tau\mathbf{y},k)+\varphi_2(\mathbf{x}^k,\tau\mathbf{y},k)-2P(\mathbf{x}^k|{\mathbf{x}^\delta})\right\| \right]d k \nonumber\\
   &\qquad\qquad\qquad\lesssim\epsilon\int_{k_0}^{K}\frac{1}{K-k_0}\E\limits_{\tau,\mathbf{x}^k|\mathbf{x}^0=\hat{\mathbf{x}}^0}\left[\left\|m_k\sqrt{\log N}+P(\mathbf{x}^k|\hat{\mathbf{x}}^0)\right\|\right]d k +\epsilon\int_{\delta}^{K}\frac{1}{K-\delta}\E\limits_{\tau,\mathbf{x}^k|\mathbf{x}^\delta={\mathbf{x}}}\left[\left\|m_k\sqrt{\log N}+P(\mathbf{x}^k|{\mathbf{x}^\delta})\right\| \right]d k \nonumber\\
    &\qquad\qquad\qquad\lesssim\frac{\epsilon}{K-k_0}\left(\sqrt{\log N} \int_{k_0}^{K} m_k d k +\int_{k_0}^{K} \frac{1}{\sigma_k} d k \right)+\frac{\epsilon}{K-\delta}\left(\sqrt{\log N} \int_{\delta}^{K} m_k d k +\int_{\delta}^{K} \frac{1}{\sigma_k} d k \right)
    \lesssim \epsilon  \log  N
\end{align}}}
\vspace{-0.5cm}
\end{figure*}
where $\hat{\mathbf{x}}^0 = \frac{1}{\sqrt{\bar{\alpha}_\delta}}\left(\mathbf{x}^{\delta} - \sqrt{1 - \bar{\alpha}_\delta}\epsilon\right)$ and $\epsilon$ follows a normal distribution.
For the second inequality, we invoke $\left|{\varphi(\mathbf{x}^k,\tau\mathbf{y},k)}\right|\le m_k \sqrt{\log N}$. 
In the last inequality, we invoke $m_k\le\frac{1}{\sigma^2_k} \le O\left(\frac{1}{k}\right)$ when  $k=o(1)$ and $m_k=\mathcal{O}(1)$ when $k\gg 1$,
and have
% \begin{equation*}
    ${\log N} \lesssim {K-\delta} $ and $ \delta\ge k_0$.
% \end{equation*}
Since $\mathcal{F}$ concatenates two ReLU NNs of the identical architecture,  the input variable $\mathbf{z} = (\mathbf{x}, \mathbf{y}, k)$ (or $\mathbf{z} = (\mathbf{x}, k)$ in the unconditional setting) lies within the bounded domain $\|(\mathbf{x}, \mathbf{y}, k)\|_{\infty} \le \max(R, K)$.
% \begin{lemma}[{\cite[Lemma~7]{chen2022nonparametric}}]\label{lemma::covering number}
% Suppose $\varrho>0$ and the input $\mathbf{z}$ satisfies $\left\|{\mathbf{z}}\right\|_\infty \le R$,  the $\varrho-$covering number of the NN class $\mathcal{F}(W,\kappa,L,P)$ w.r.t. $\left\|{\cdot}\right\|_{L_{\infty}}$ satisfies
% \begin{align*}
%     \mathcal{N}\left(\varrho, \mathcal{F}(W,\kappa,L,P), \left\|{\cdot}\right\|_{L_{\infty}}\right) \le \Big(\frac{2L^2(WR+2)\kappa^LW^{L+1}}{\varrho}\Big)^P.
% \end{align*}
% \end{lemma}
By {\cite[Lemma~7]{chen2022nonparametric}}, the covering number of $\mathcal{F}$ is bounded as
{\small{\begin{equation*}
     \mathcal{N}\left(\varrho, \mathcal{F}, \left\|{\cdot}\right\|_{L_{\infty}\mathcal{D}}\right) \!\lesssim\! \left(\!\frac{2L^2(W\max(R,K)\!+\!2)\kappa^LW^{L+1}}{\varrho}\!\right)^{\!\!2P}\!\!.
\end{equation*}}}
Together with (\ref{equ:: lip loss}), we bound the covering number of~$\mathcal{S}(R)$. 
% The proof is complete.

\subsection{Proof of Lemma~\ref{lemma:approx_prob}}\label{appendix:approx_error}
The proposed loss function can be divided as follows:
{\small{\begin{subequations}
    \begin{align}
        l_1 &\!=\!\!\! \int_{\mathbf{x}^{k}}\!\! \|{\varphi(\mathbf{x}^{k},\mathbf{y},{k})\!-\!\nabla\log P_{k}(\mathbf{x}^{k}|\mathbf{y})\|}^2_2 P_{k}(\mathbf{x}^{k}|\mathbf{y}) d\mathbf{x}^{k},\label{eq:obj_delta}\\
        &\qquad\qquad\qquad\qquad \mathbf{x}^{\delta}\sim P_{\text{data}}(\mathbf{x}^{\delta}|\mathbf{y}), \quad\forall\delta\le k\le K, \notag\\
        l_2 &\!=\!\!\! \int_{\mathbf{x}^{k}}\!\! \|{\varphi(\mathbf{x}^{k},\mathbf{y},{k})\!-\!\nabla\log P_{k}(\mathbf{x}^{k}|\mathbf{y})\|}^2_2 P_{k}(\mathbf{x}^{k}|\mathbf{y}) d\mathbf{x}^{k},\label{eq:obj_0}\\
        &\qquad\qquad\qquad\qquad\mathbf{x}^{0} = (\mathbf{x}^{\delta} - \sqrt{1 - \bar{\alpha}_\delta}\epsilon)/{\sqrt{\bar{\alpha}_\delta}}\quad\forall 0\le k\le K.\notag
    \end{align}
\end{subequations}}}
% where $\mathbf{x}^{\delta}\sim P_{\text{data}}(\mathbf{x}^{\delta}|\mathbf{y})$ for \eqref{eq:obj_delta} and 
% $\mathbf{x}^{0} = \frac{1}{\sqrt{\bar{\alpha}_\delta}}\left(\mathbf{x}^{\delta} - \sqrt{1 - \bar{\alpha}_\delta}\epsilon\right)$ for \eqref{eq:obj_0}.
Next, we analyze $l_1$ and $l_2$ separately.

\textit{Proof of $l_1$:}
For $l_1$, each conditional distribution in the forward process satisfies {\cite[Assumption 3.1]{DM_offline_9}}. We obtain \eqref{eq:app_10}.
\begin{figure*}[ht]
\centering
{\small{\begin{align}\label{eq:app_10}
    p_t(\mathbf{x}^k|\mathbf{y})
    =&\int_{\mathbb{R}^{d}}P(\mathbf{x}_{\mathrm{tr}}|\mathbf{y})\frac{1}{\sigma_k^{d}(2\pi)^{d/2}} \exp\left(-\frac{\left\|\sqrt{\alpha_k}\mathbf{x}_{\mathrm{tr}}-\mathbf{x}^{k}\right\|^2}{2\sigma_k^2}\right) d\mathbf{x}_{\mathrm{tr}}\\
    % =&\int_{\mathbb{R}^{d}}\exp(-\frac{C_1\left\|\mathbf{x}_{\mathrm{tr}}\right\|^2_2}{2}) \cdot f(\mathbf{x}_{\mathrm{tr}}, \mathbf{y})\frac{1}{\sigma_k^{d}(2\pi)^{d/2}} \exp\left(-\frac{\left\|\sqrt{\alpha_k}\mathbf{x}_{\mathrm{tr}}-\mathbf{x}^{k}\right\|^2}{2\sigma_k^2}\right) d\mathbf{x}_{\mathrm{tr}}\\
    % % % =&\int_{\mathbb{R}^{d}} f(\mathbf{x}_{\mathrm{tr}}, \mathbf{y})\frac{1}{\sigma_k^{d}(2\pi)^{d/2}} \exp\left(-\frac{\left\|\sqrt{\alpha_k}\mathbf{x}_{\mathrm{tr}}-\mathbf{x}^{k}\right\|^2+C_1\sigma_k^2\left\|\mathbf{x}_{\mathrm{tr}}\right\|^2_2}{2\sigma_k^2}\right) d\mathbf{x}_{\mathrm{tr}}\nonumber\\
    % % =&\int_{\mathbb{R}^{d}} f(\mathbf{x}_{\mathrm{tr}}, \mathbf{y})\frac{1}{\sigma_k^{d}(2\pi)^{d/2}} \exp\left(-\frac{\left\|(\alpha_k+C_1\sigma_k^2)\mathbf{x}_{\mathrm{tr}}-\sqrt{\alpha_k}\mathbf{x}^k\right\|^2+C_1\sigma_k^2\left\|\mathbf{x}^k\right\|^2}{2\sigma_k^2(\alpha_k+C_1\sigma_k^2)}\right) d\mathbf{x}_{\mathrm{tr}}\nonumber\\
    % % % =&\exp\left(-\frac{C_1\left\|\mathbf{x}^k\right\|^2_2}{2(\alpha_k+C_1\sigma_k^2)}\right)\int_{\mathbb{R}^{d}} \frac{f(\mathbf{x}_{\mathrm{tr}},\mathbf{y})}{(2\pi)^{d/2}\sigma_k^d}\exp\left(-\frac{\left\|(\alpha_k+C_1\sigma_k^2)\mathbf{x}_{\mathrm{tr}}-\sqrt{\alpha_k}\mathbf{x}^k\right\|^2}{2\sigma_{k}^2(\alpha_k+C_1\sigma_k^2)}\right)d\mathbf{x}_{\mathrm{tr}}\nonumber\\
    =&\exp\left(\!-\!\frac{C_1\left\|\mathbf{x}^k\right\|^2_2}{2(\alpha_k\!+\!C_1\sigma_k^2)}\right)\int_{\mathbb{R}^{d}} \frac{f(\mathbf{x}_{\mathrm{tr}},\mathbf{y})}{(2\pi)^{d/2}\sigma_k^d}\exp\left(-\frac{\left\|\mathbf{x}_{\mathrm{tr}}-\sqrt{\alpha_k}\mathbf{x}^k/(\alpha_k\!+\!C_1\sigma_k^2)\right\|^2}{2\sigma_{k}^2/(\alpha_k\!+\!C_1\sigma_k^2)}\right)d\mathbf{x}_{\mathrm{tr}}
    \!=\!\exp\left(\!-\!\frac{C_1\left\|\mathbf{x}^k\right\|^2_2}{2(\alpha_k\!+\!C_1\sigma_k^2)}\right)f^k(\mathbf{x}^k,\mathbf{y}).\nonumber
\end{align}}}
\vspace{-0.5cm}
\end{figure*}
With $f \in \mathcal{H}^{b}(\mathbb{R}^{w_x}\times \mathbb{R}^{w_y}, B)$ and $f(\mathbf{x}_{\mathrm{tr}}, \mathbf{y}) \geq C$ in Assumption~\ref{assump:distri_true}, there are two constants $B'$ and $C'$, such that $f^k \in \mathcal{H}^{b}(\mathbb{R}^{w_x}\times \mathbb{R}^{w_y}, B')$ and $f(\mathbf{x}^k, \mathbf{y}) \geq C'$ hold. Let $C_2=C_1/(\alpha_k+C_1\sigma_k^2)$, and the conditional density function $P(\mathbf{x}^k | \mathbf{y}) = \exp(-C_2\left\|\mathbf{x}^k\right\|^2_2/2) \cdot f(\mathbf{x}^k, \mathbf{y})$, there is
{\small{\begin{align*}
    P(\mathbf{x}^k|\mathbf{y})&=\int_{\mathbb{R}^{d}}\frac{P(\mathbf{x}_{\mathrm{tr}}|\mathbf{y})}{\sigma_k^{d}(2\pi)^{d/2}} \exp\left(-\frac{\left\|\sqrt{\alpha_k}\mathbf{x}_{\mathrm{tr}}-\mathbf{x}^{k}\right\|^2}{2\sigma_k^2}\right) d\mathbf{x}_{\mathrm{tr}}.
\end{align*}}}
Therefore, {\cite[Assumption 3.1]{DM_offline_9}} holds $\forall k \ge 0$. 
Based on {\cite[Thm 3.4]{DM_offline_9}}, by replacing $K_0$ with $\delta$ in (\ref{eq:obj_delta}), we derive
\begin{align}\label{eq:l_1}
    l_1=\mathcal{O} \left( {B^2}\cdot N^{-\frac{2b}{w_x+w_y}} \cdot (\log N)^{b+1}/{\sigma_k^2}\right).
\end{align}
% \begin{corollary}
%     Suppose Assumption~\ref{assump:distri_true} holds. For sufficiently large $N$ and constant $C_{\alpha}>0$, by taking terminal step $K=C_{\alpha}\log N$, there exists $\mathbf{s} \in \mathcal{F}(M_t, W, \kappa, L, P)$ inherited from Lemma~\ref{lemma:DM_offline_9} such that for all $\mathbf{y} \in [0,1]^{w_y}$ and $k \in [\delta,K]$, it holds that
%     \begin{align}\label{eq:l_1}
%         l_1=\mathcal{O} \left( \frac{B^2}{\sigma_k^2}\cdot N^{-\frac{2b}{w_x+w_y}} \cdot (\log N)^{b+1}\right).
%     \end{align}
% \end{corollary}

\textit{Proof of $l_2$:}
Based on DM, we have $\mathbf{x}^{0} = \frac{1}{\sqrt{\bar{\alpha}_\delta}}\left(\mathbf{x}^{\delta} - \sqrt{1 - \bar{\alpha}_\delta}\epsilon\right)$ with $\epsilon$ following a normal distribution: 
% Thus, we only need to prove that $P(\mathbf{x}^{0}|\mathbf{y})$ satisfies {\color{red}{\cite[Assumption 3.1]{DM_offline_9}}}. 
{\small{\begin{align}\label{eq:eps}
    \epsilon = {(\mathbf{x}^{\delta}-\sqrt{\bar{\alpha}_\delta}\mathbf{x}^{0})}/{\sqrt{1 - \bar{\alpha}_\delta}}.
\end{align}}}
With
% \begin{align}\label{eq:eps_prob}
    $P(\epsilon) = \exp\left(-{\|\epsilon\|_2^2}/{2}\right)/(2\pi)^{d/2}$, 
% \end{align}
% By substituting (\ref{eq:eps}) into (\ref{eq:eps_prob}), we obtain
we have
{\small{\begin{align}
    P(\epsilon|\mathbf{x}^{0}, \mathbf{x}^{\delta}, \mathbf{y}) = \frac{1}{(2\pi)^{d/2}}\exp\left(-\frac{\|\mathbf{x}^{\delta}-\sqrt{\bar{\alpha}_\delta}\mathbf{x}^{0}\|_2^2}{2(1 - \bar{\alpha}_\delta)}\right).
\end{align}}}
With variable substitution, we have
{\small{\begin{align}\label{eq:sup4}
    P(\mathbf{x}^{0}|\mathbf{x}^{\delta}, \mathbf{y}) \!= \!\frac{1}{\sigma_\delta^d(2\pi)^{d/2}}\exp\left(\!\!-\frac{\|\mathbf{x}^{0}\!\!-\!\mathbf{x}^{\delta}/\sqrt{\bar{\alpha}_\delta}\|_2^2}{2\sigma_\delta^2}\!\right).
\end{align}}}
Therefore, $P(\mathbf{x}^{0}|\mathbf{x}^{\delta}, \mathbf{y})$ follows a distribution with mean $\mathbf{x}^{\delta}/\sqrt{\bar{\alpha}_\delta}$ and covariance $\sigma_\delta\mathbf{I}$, where $\sigma_\delta=\sqrt{(1-\bar{\alpha}_\delta)/(\bar{\alpha}_\delta)}$. Then, it readily follows that
{\small{\begin{align}
    &P(\mathbf{x}^{0}|\mathbf{y}) = \int_{\mathbb{R}^{d}}P(\mathbf{x}^{0}|\mathbf{x}^{\delta}, \mathbf{y})P(\mathbf{x}^{\delta}|\mathbf{y})d\mathbf{x}^{\delta}\\
    =&\int_{\mathbb{R}^{d}}P(\mathbf{x}^{\delta}|\mathbf{y})\frac{1}{\sigma_\delta^d(2\pi)^{d/2}}\exp\left(-\frac{\|\mathbf{x}^{0}-\mathbf{x}^{\delta}/\sqrt{\bar{\alpha}_\delta}\|_2^2}{2\sigma_\delta^2}\right)d\mathbf{x}^{\delta}.\nonumber
\end{align}}}
As a result, $P(\mathbf{x}^{0}|\mathbf{y})$ satisfies {\cite[Assumption 3.1]{DM_offline_9}} and leads to {\cite[Theorem 3.4]{DM_offline_9}}. 
% Assuming $C_\sigma$ in {\color{red}Lemma~\ref{lemma:DM_offline_9}} is sufficiently large, 
Consequently, we can replace $K_0$ with $0$ and obtain the final conclusion, that is:
{\small{\begin{align}\label{eq:l_2}
    l_2=\mathcal{O} \left( {B^2}\cdot N^{-\frac{2b}{w_x+w_y}} \cdot (\log N)^{b+1}/{\sigma_k^2}\right).
\end{align}}}
% \begin{corollary}
%     Suppose Assumption~\ref{assump:distri_true} holds. For sufficiently large $N$ and constant $C_{\alpha}>0$, by taking terminal step $K=C_{\alpha}\log N$, there exists $\mathbf{s} \in \mathcal{F}(M_t, W, \kappa, L, P)$ inherited from Lemma~\ref{lemma:DM_offline_9} such that for all $\mathbf{y} \in [0,1]^{w_y}$ and $k \in [0,K]$, it holds that
%     \begin{align}\label{eq:l_2}
%         l_2=\mathcal{O} \left( \frac{B^2}{\sigma_k^2}\cdot N^{-\frac{2b}{w_x+w_y}} \cdot (\log N)^{b+1}\right).
%     \end{align}
% \end{corollary}

\textit{By summing $l_1$ in (\ref{eq:l_1}) and $l_2$ in (\ref{eq:l_2}),}
we obtain the final NN approximation error. Lemma~\ref{lemma:approx_prob} is proved.

\subsection{Proof of Lemma~\ref{proof_Rs}}\label{appendix:proof_Rs}
{\color{black}Define $\mathcal{N}_R\coloneqq \mathcal{N}\left(\varrho, \mathcal{S}(R), \left\|{\cdot}\right\|_{L_{\infty}\mathcal{D}}\right)$ for brevity.}
Under the network configuration in Lemma~\ref{lemma:approx_prob}, we obtain the following upper bound on the logarithmic covering number:

{\small{\begin{align} 
    \log \mathcal{N}_R \lesssim& N\log^9 N \Big( \text{Poly}(\log \log N)+\log^8 N+\frac{1}{\varrho}  \nonumber\\
    &+\text{Poly}(\log \log N)\log N\log 
    R\Big)\notag\\
    \lesssim& N\log^{9} N \left( \log^8 N+\log^2 N\log R +\frac{1}{\varrho}  \right). \label{equ::logcover}
\end{align}}}
% With Lemmas \ref{lemma:approx_prob}, \ref{lemma::bound loss class}, and \ref{lemma::covering number S}, we now prove Lemma~\ref{proof_Rs}.
Let $\varphi^{\star}(\mathbf{x},\mathbf{y},k)=\nabla \log P(\mathbf{x}^k|\mathbf{y})$ be the true score if $\mathbf{y} \neq \varnothing$ and $\varphi^{\star}(\mathbf{x},\varnothing,k)=\nabla \log P(\mathbf{x}^k)$. Create $n$ i.i.d. ghost samples: 
$$(\mathbf{x}'_1,\mathbf{y}'_1),...,(\mathbf{x}'_n,\mathbf{y}'_n) \sim P_{\mathrm{data}}(\mathbf{x}^{\delta}, \mathbf{y}).$$
As $\mathcal{R}_{\star}(\varphi^{\star}) = 0$ and $\mathcal{R}_{\star}(\varphi)$ deviates from $\ell(\varphi)$ merely by a constant independent of $\varphi$, it can be readily bounded that
{\small{\begin{align}
    \mathcal{R}_{\star}({\varphi})&=\mathcal{R}_{\star}({\varphi})-\mathcal{R}_{\star}(\varphi^{\star})=\ell({\varphi})-\ell(\varphi^{\star})\\
    &=\frac{1}{n}\mathbb{E}_{{\mathbf{x}'_t,\mathbf{y}'_t}}\left[\sum\nolimits_{t}\!\left(\ell(\mathbf{x}'_t,\mathbf{y}'_t;{\varphi})\!-\!\ell(\mathbf{x}'_t,\mathbf{y}'_t;\varphi^{\star})  \right) \right].\notag
\end{align}}}
Define $\ell_1=\frac{1}{n}\sum_{t}\left(\ell(\mathbf{x}_{t+1},\mathbf{y}_t;{\varphi})-\ell(\mathbf{x}_{t+1},\mathbf{y}_t;\varphi^{\star})  \right)$, $\ell^{\mathrm{tr}}_1=\frac{1}{n}\sum_{t}\left(\ell^{\mathrm{tr}}(\mathbf{x}_{t+1},\mathbf{y}_t;{\varphi})-\ell^{\mathrm{tr}}(\mathbf{x}_{t+1},\mathbf{y}_t;\varphi^{\star})  \right)$ and $\ell_2=\frac{1}{n}\sum_{t}\left(\ell(\mathbf{x}'_t,\mathbf{y}'_t;{\varphi})-\ell(\mathbf{x}'_t,\mathbf{y}'_t;\varphi^{\star})  \right)$, $\ell^{\mathrm{tr}}_2=\frac{1}{n}\sum_{t}\left(\ell^{\mathrm{tr}}(\mathbf{x}'_t,\mathbf{y}'_t;{\varphi})-\ell^{\mathrm{tr}}(\mathbf{x}'_t,\mathbf{y}'_t;\varphi^{\star})  \right).$
We decompose  $\mathbb{E}_{{\mathbf{x}_{t+1},\mathbf{y}_t}}\left[  \mathcal{R}_{\star}({\varphi})\right]$ into
{\small{\begin{align}
    &\E\limits_{{\mathbf{x}_{t+1},\mathbf{y}_t}}\left[  \mathcal{R}_{\star}({\varphi})\right]=\E\limits_{{\mathbf{x}_{t+1},\mathbf{y}_t}}\left[\ell^{\mathrm{tr}}_1 -\ell_1 \right]+\E\limits_{{\mathbf{x}_{t+1},\mathbf{y}_t}}\left[\mathbb{E}_{\mathbf{x}'_t,\mathbf{y}'_t}\left[ \ell_2-\ell^{\mathrm{tr}}_2\right] \right]\notag\\
    &+\mathbb{E}_{{\mathbf{x}_{t+1},\mathbf{y}_t}}\left[\mathbb{E}_{{\mathbf{x}'_t,\mathbf{y}'_t}}\left[ \ell^{\mathrm{tr}}_2\right] -\ell^{\mathrm{tr}}_1\right]
    % \notag\\
    % &
    + \mathbb{E}_{{\mathbf{x}_{t+1},\mathbf{y}_t}}\left[\ell_1\right], \label{equ::score esti C}
\end{align}}}
with the four terms on the RHS denoted by $B_1$, $B_2$, $C$, and~$D$.

\textit{Bounding $B_1$ and $B_2$:}
By the definition of $\ell(\mathbf{x},\mathbf{y}; \varphi)$, \eqref{eq:app_11} holds $\forall\mathbf{x},\mathbf{y}$ and $\mathbf{s} \in \mathcal{F}$,
\begin{figure*}[ht]
    \centering
    {\small{
\begin{align}\label{eq:app_11}
    \ell(\mathbf{x},\!\mathbf{y};\! \varphi)\!&\le\! 2\!\!\int_{k_0}^{T}\!\!\!\frac{1}{K\!\!-\!k_0}\!\E\limits_{\tau,\mathbf{x}^k|\mathbf{x}^0\!=\!\hat{\mathbf{x}}^0}\!\left[\left\|{\varphi(\mathbf{x}^k\!\!,\tau\mathbf{y},k)}\right\|_2^2\!\!+\!\left\|{\nabla \log p_t(\mathbf{x}^k|\hat{\mathbf{x}}^0)}\right\|_2^2\right]  d k
    % \notag\\    & 
    \!+ \!\!2\!\!\int_{\delta}^{T}\!\!\!\!\frac{1}{K\!\!-\!\delta}\!\!\E\limits_{\tau,\mathbf{x}^k|\mathbf{x}^\delta\!=\!\mathbf{x}}\!\left[\left\|{\varphi(\mathbf{x}^k\!\!,\tau\mathbf{y},k)}\right\|_2^2\!\!+\!\!\left\|{\nabla \log p_t(\mathbf{x}^k|\mathbf{x}^\delta)}\right\|_2^2\right] \! d k\notag\\
    \lesssim & \int_{k_0}^{T}\frac{1}{K-k_0}\E\limits_{\tau,\mathbf{x}^k|\mathbf{x}^0=\hat{\mathbf{x}}^0}\left[m_t^2\log N+\left\|{\nabla \log p_t(\mathbf{x}^k|\hat{\mathbf{x}}^0)}\right\|_2^2\right]  d k + \int_{\delta}^{T}\frac{1}{K-\delta}\E\limits_{\tau,\mathbf{x}^k|\mathbf{x}^\delta=\mathbf{x}}\left[m_t^2\log N+\left\|{\nabla \log p_t(\mathbf{x}^k|\mathbf{x}^\delta)}\right\|_2^2\right]  d k\notag\\
    \lesssim & \int_{k_0}^{K} M_{k}^2 d k +\int_{k_0}^{K}\frac{1}{K-k_0}\frac{1}{\sigma^2_k}d k + \int_{\delta}^{K} M_{k}^2 d k +\int_{\delta}^{K}\frac{1}{K-\delta}\frac{1}{\sigma^2_k}d k
    \lesssim \int_{k_0}^{K} M_{k}^2 d k + \int_{\delta}^{K}M_{k}^2 d k
    \lesssim \int_{k_0}^{K} M_{k}^2 d k = M,
\end{align}}}
\vspace{-0.5cm}
\end{figure*}
where we invoke $\left|{\varphi}\right|\lesssim m_k\sqrt{\log N}$  for the first inequality and $1/\sigma_k \lesssim m_k$ for the last line.
For $\forall \varphi \in \mathcal{F} $ with $\varphi$ possibly dependent on $(\mathbf{x}, \mathbf{y})$, we obtain \eqref{equ::trunc popu loss},
\begin{figure*}[ht]
    \centering
{\small{\begin{align}\label{equ::trunc popu loss}
    &\E\limits_{\mathbf{x},\mathbf{y}}\!\left[\left| \ell(\mathbf{x},\mathbf{y}; \varphi)\!-\!\ell^{\mathrm{tr}}(\mathbf{x},\mathbf{y}; \varphi) \right|\right]
    % \!=\!\int_{k_0}^{K}\frac{1}{K\!-\!k_0}\int_{\mathbf{y}}\int_{\left\|{\mathbf{x}}\right\|>R}\E\limits_{\tau,\mathbf{x}^k|\mathbf{x}^0=\hat{\mathbf{x}}^0}\left[\left\|{\varphi(\mathbf{x}^k,\tau\mathbf{y},k)\!-\!\nabla \log \zeta(\mathbf{x}^k|\mathbf{x}^0)}\right\|_2^2\right]P(\mathbf{x}|\mathbf{y})P(\mathbf{y}) d \mathbf{x} d \mathbf{y} d k \nonumber\\
    % &\qquad+\int_{\delta}^{K}\frac{1}{K-\delta}\int_{\mathbf{y}}\int_{\left\|{\mathbf{x}}\right\|>R}\mathbb{E}_{\tau,\mathbf{x}^k|\mathbf{x}^\delta=\mathbf{x}}\left[\left\|{\varphi(\mathbf{x}^k,\tau\mathbf{y},k)-\nabla \log \zeta(\mathbf{x}^k|\mathbf{x}^\delta)}\right\|_2^2\right]P(\mathbf{x}|\mathbf{y})P(\mathbf{y}) d \mathbf{x} d \mathbf{y} d k \nonumber\\
    \!\le\! 2\!\!\int_{k_0}^{K}\!\!\!\frac{1}{K\!-\!k_0}\!\int_{\mathbf{y}}\!\int_{\left\|{\mathbf{x}}\right\|>R}\!\E\limits_{\tau,\mathbf{x}^k|\mathbf{x}^0=\hat{\mathbf{x}}^0}\!\left[\left\|\varphi(\mathbf{x}^k\!,\!\tau\mathbf{y}\!,\!k)\right\|_2^2\!+\!\left\|\nabla \log \zeta(\mathbf{x}^k|\mathbf{x}^0)\right\|_2^2\right]\!P(\mathbf{x}|\mathbf{y})P(\mathbf{y}) d \mathbf{x} d \mathbf{y} d k\nonumber\\
    &\qquad+ \!2\!\int_{\delta}^{K}\!\!\!\frac{1}{K\!-\!\delta}\int_{\mathbf{y}}\!\int_{\left\|{\mathbf{x}}\right\|>R}\!\!\mathbb{E}_{\tau,\mathbf{x}^k|\mathbf{x}^\delta=\mathbf{x}}\!\left[\left\|\varphi(\mathbf{x}^k,\tau\mathbf{y},k)\right\|_2^2\!+\!\left\|\nabla \log \zeta(\mathbf{x}^k|\mathbf{x}^\delta)\right\|_2^2\right]P(\mathbf{x}|\mathbf{y})P(\mathbf{y}) d \mathbf{x} d \mathbf{y} d k \nonumber\\
    \lesssim& \int_{k_0}^{K}\frac{1}{\log N}\int_{\left\|{\mathbf{x}}\right\|>R}\mathbb{E}_{\tau,\mathbf{x}^k|\mathbf{x}^0=\hat{\mathbf{x}}^0}\left[m_k^2\log N+\left\|{\nabla \log \zeta(\mathbf{x}^k|\mathbf{x}^0)}\right\|_2^2\right]\exp(-C_1 \left\|{\mathbf{x}}\right\|_2^2/2) d \mathbf{x} d t\nonumber\\
    &+ \int_{\delta}^{K}\frac{1}{\log N}\int_{\left\|{\mathbf{x}}\right\|>R}\mathbb{E}_{\tau,\mathbf{x}^k|\mathbf{x}^\delta=\mathbf{x}}\left[m_k^2\log N+\left\|{\nabla \log \zeta(\mathbf{x}^k|\mathbf{x}^0)}\right\|_2^2\right]\exp(-C_1' \left\|{\mathbf{x}}\right\|_2^2/2) d \mathbf{x} d t\\
    \lesssim & \exp\!\left(\!-C_1R^2\right)\!R\!\!\int_{k_0}^{K}\!\! m_{k}^2 d k \!+\!\exp\!\left(-C_1R^2\right)\!\!\int_{k_0}^{K}\!\!\frac{1}{\sigma^2_k}d k\!+\!\exp\!\left(-C_1'R^2\right)\!R\!\int_{\delta}^{K}\!\! m_{k}^2 d k \!+\!\exp\!\left(-C_1'R^2\right)\!\!\int_{\delta}^{K}\!\!\frac{1}{\sigma^2_k}d k
    \!\lesssim \!\exp\!\left(-C_1R^2\right)\!RM\nonumber
\end{align}}}
\vspace{-0.5cm}
\end{figure*}
where the middle line follows from the sub-Gaussian property of $P(\mathbf{x}|\mathbf{y})$ with Assumption~\ref{assump:distri_true}, and the third inequality invokes $\mathbb{E}_{\mathbf{x}^k|\mathbf{x}^0=\mathbf{x}}\left[\left\|{\nabla \log P(\mathbf{x}^k|\mathbf{x}^0)}\right\|_2^2 \right]=1/\sigma_k^2$ and $\mathbb{E}_{\mathbf{x}^k|\mathbf{x}^\delta=\mathbf{x}}\left[\left\|{\nabla \log P(\mathbf{x}^k|\mathbf{x}^\delta)}\right\|_2^2 \right]=1/\sigma_k^2$. Thus, both  $B_1$ and $B_2$ are bounded by $\mathcal{O}\left( \exp\left(-C_1R^2\right)RM\right)$.

\textit{Bounding $C$:}
For conciseness, we take $\mathbf{z}=(\mathbf{x},\mathbf{y})$, and write $\ell^{\mathrm{tr}}(\mathbf{x},\mathbf{y};{\varphi})$ as $\hat{\ell}(\mathbf{z})$ and $\ell^{\mathrm{tr}}(\mathbf{x},\mathbf{y};\varphi^{\star})$ as $\ell^{\star}(\mathbf{z})$.
For $\varrho>0$ to be specified subsequently, consider $\mathcal{J}=\{\ell_1,\ldots,\ell_{\mathcal{N}_R}\}$ as a minimal $\varrho$-covering of $\mathcal{S}(R)$ in the $L^{\infty}$ metric over $\mathcal{D}$.
Let $J$ denote a random index such that $|\hat{\ell}-\ell_J|_{\infty}\le\varrho$. Set $u_j=\max \left\{A,\sqrt{\mathbb{E}_{\mathbf{z}}\left[ \ell_j(\mathbf{z})-\ell^{\star}(\mathbf{z})\right]}\right\}$ with $\mathbf{z}\sim P_{\mathbf{x},\mathbf{y}}$ independent of ${(\mathbf{z}_t,\mathbf{z}'_t)}_{t\in\mathcal{B}}$. Moreover, we define
{\small{\begin{align*}
    E=\max_{1\le j\le \mathcal{N}_R}\left| \sum\nolimits_{t} {[\left(\ell_j(\mathbf{z}_t)-\ell^{\star}(\mathbf{z}_t)\right) -\left(\ell_j(\mathbf{z}'_t)-\ell^{\star}(\mathbf{z}'_t)\right)]}/{u_j}\right|.
\end{align*}}}
We can bound $C$ as
{\small{\begin{align}\label{equ:: excess bound}
    \left|{C}\right|&\!=\!\left|\E\limits_{{\mathbf{z}_{t}}} \left[\frac{1}{n}\sum_{t} \left(\hat{\ell}(\mathbf{z}_t)-\ell^{\star}(\mathbf{z}_t)\right)  -\E\limits_{{\mathbf{z}'_{t}}} \left[\sum_{t} \left(\hat{\ell}(\mathbf{z}'_t)-\ell^{\star}(\mathbf{z}'_t)\right) \right]\right]\right|\nonumber\\
    % &=\left|\frac{1}{n} \mathbb{E}_{{\mathbf{z}_t,\mathbf{z}'_t}} \left[ \sum_{i\in\mathcal{B}} \left(\left(\hat{\ell}(\mathbf{z}_t)-\ell^{\star}(\mathbf{z}_t)\right) -\left(\hat{\ell}(\mathbf{z}'_t)-\ell^{\star}(\mathbf{z}'_t)\right)\right)\right]\right|\nonumber\\ 
    &\le \left| \frac{1}{n} \E\limits_{{\mathbf{z}_t,\mathbf{z}'_t}} \left[ \sum_{t\in\mathcal{B}} \left(\left(\ell_J(\mathbf{z}_t)\!-\!\ell^{\star}(\mathbf{z}_t)\right) \!-\!\left(\ell_J(\mathbf{z}'_t)\!-\!\ell^{\star}(\mathbf{z}'_t)\right)\right)\right]\right| \!+\!2\varrho \nonumber\\
    &\le 
   \frac{1}{n} \E\limits_{{\mathbf{z}_t,\mathbf{z}'_t}} \left[ u_JE \right] \!+\!2\varrho \!\le\! \frac{1}{2} \E\limits_{{\mathbf{z}_t,\mathbf{z}'_t}} \left[ u_J^2 \right] \!+\!\frac{1}{2n^2}  \E\limits_{{\mathbf{z}_t,\mathbf{z}'_t}} \left[ E^2 \right] \!+\!2\varrho.\!\!
\end{align}}}
Let $h_j(\mathbf{z})=\ell_j(\mathbf{z})-\ell^{\star}(\mathbf{z})$ and $\hat{h}(\mathbf{z})=\hat{\ell}(\mathbf{z})-\ell^{\star}(\mathbf{z})$; define the truncated population loss as $\mathcal{R}_{\star}^{\mathrm{tr}}({\varphi})=\mathbb{E}_{\mathbf{z}}\left[ \hat{h}\right]$ and the truncated empirical loss as $\hat{\mathcal{R}}_{\star}^{\text{tr}}({\varphi})=\frac{1}{n}\sum_{t} \hat{h}(\mathbf{z}_t) $. From
(\ref{equ::trunc popu loss}), it follows that $\left|{\mathcal{R}_{\star}^{\mathrm{tr}}({\varphi})-\mathcal{R}_{\star}({\varphi})}\right|\lesssim \exp\left(-C_1R^2\right)RM $. 
% Now we bound $ \mathbb{E}_{{\mathbf{z}_t,\mathbf{z}'_t}} \left[ u_J^2 \right]$ and $\mathbb{E}_{{\mathbf{z}_t,\mathbf{z}'_t}} \left[ E^2 \right]$ separately.

From the definition of $u_J$, it follows that
{\small{\begin{align}
    \mathbb{E}_{{\mathbf{z}_t,\mathbf{z}'_t}} \left[ u_J^2 \right]&\le A^2+ \mathbb{E}_{{\mathbf{z}_t,\mathbf{z}'_t}}\left[\mathbb{E}_{\mathbf{z}}\left[h_J(\mathbf{z}) \right] \right] \nonumber\\
    &\le A^2+ \mathbb{E}_{{\mathbf{z}_t,\mathbf{z}'_t}}\left[\mathbb{E}_{\mathbf{z}}\left[ \hat{h}(\mathbf{z})\right]\right] +2\varrho  \nonumber\\
    &=A^2+\mathbb{E}_{{\mathbf{z}_t,\mathbf{z}'_t}}\left[\mathcal{R}_{\star}^{\mathrm{tr}}({\varphi})\right] +2\varrho. \label{equ:: bound uj}
\end{align}}}
Denote $g_j = \sum_{t} {[h_j(\mathbf{z}_t)-h_j(\mathbf{z}'_t)]}/{u_j}$. Notice that $\mathbb{E}_{\mathbf{z}_t,\mathbf{z}'_t} \left[ {[h_j(\mathbf{z}_t)-h_j(\mathbf{z}'_t)]}/{u_j} \right]=0,\forall t,j$. With the independence of $\left\{g_j\right\}_{j=1}^{\mathcal{N}_R}$, we have
{\small{\begin{align*}
    &\mathbb{E}_{{\mathbf{z}_t,\mathbf{z}'_t}} \left[\sum\nolimits_{t} \left[{(h_j(\mathbf{z}_t)-h_j(\mathbf{z}'_t))}/{u_j}\right]^2\right]\\
    &\le  \sum\nolimits_{t} \mathbb{E}_{\mathbf{z}_t,\mathbf{z}'_t} \left[\left({h_j(\mathbf{z}_t)}/{u_j}\right)^2+\left({h_j(\mathbf{z}'_t)}/{u_j}\right)^2\right]\\
    &\le M \sum\nolimits_{t} \mathbb{E}_{\mathbf{z}_t,\mathbf{z}'_t} \left[{h_j(\mathbf{z}_t)}/{u^2_j}+{h_j(\mathbf{z}'_t)}/{u^2_j}\right]\le 2nM.
\end{align*}}}
Since $\left|{(h_j(\mathbf{z}_t)-h_j(\mathbf{z}'_t))}/{u_j}\right| \le {M}/{A}$ and $g_j$ has zero mean, applying Bernstein’s inequality yields that, $\forall j$, there exists
{\small{\begin{align*}
    P\left[g_j^2\ge h\right]&=2P\left[\sum\nolimits_{t}{(h_j(\mathbf{z}_t)-h_j(\mathbf{z}'_t))}/{u_j} \ge \sqrt{h}\right]\\
    &\le 2\exp ( -{(h/2)}/{(M(2n+{\sqrt{h}}/{3A}))} ).
\end{align*}}}
{\color{black}Thus, it follows that
{\small{\begin{align*}
    P\left[E^2\!\ge\! h\right]\!\le\!\sum\nolimits_{j=1}^{\mathcal{N}_R} P\left[g_j^2\ge h\right]\!\le\!  2\mathcal{N}_R\!\exp \left( \!\!-\frac{h/2}{M(2n\!+\!\frac{\sqrt{h}}{3A})} \!\right)\!.
\end{align*}}}
Thus, we have: $\forall h_0>0$, 
{\small{\begin{align*}
   &\mathbb{E}_{{\mathbf{z}_t,\mathbf{z}'_t}} \left[ E^2 \right]
   =\int_{0}^{h_0} \!\!\!P\!\left[E^2\ge h\right]d h 
  \!+\!\!\int_{h_0}^\infty \!\!\!P\!\left[E^2\ge h\right]d h\\ &\le h_0+\int_{h_0}^\infty 2\mathcal{N}_R\exp \left( -{(h/2)}/{\left(M\left(2n+\frac{\sqrt{h}}{3A}\right)\right)} \right) d h \\
   &\le h_0 +2\mathcal{N}_R \int_{h_0}^\infty \left[\exp\left(-\frac{h}{8Mn}\right)+\exp\left(-\frac{3A\sqrt{h}}{4M}\right)\right]d h\\
   &\le h_0 +2\mathcal{N}_R[8Mn\exp\left(-{h_0}/{(8Mn)}\right)\\
   &+\left[{8M\sqrt{h_0}}/{(3A)}+{32M}/{(9A^2)}\right]\exp(-{3A\sqrt{h_0}}/{(4M)})].
\end{align*}}}}
Taking $A=\sqrt{h_0}/6n$ and $h_0=8Mn\log\mathcal{N}_R$, we have 
{\small{\begin{align}
    \E\limits_{{\mathbf{z}_t,\mathbf{z}'_t}} [ E^2 ] &\!\le\! 8Mn\log \mathcal{N}_R \!+\!48Mn\!+\!\frac{32}{\log\mathcal{N}_R} \!\lesssim\! Mn\log \mathcal{N}_R.\label{equ:: bound G}
\end{align}}}
By substituting (\ref{equ:: bound uj}) and  (\ref{equ:: bound G}) into (\ref{equ:: excess bound}), it follows that
{\small{\begin{align*}
    &\left|\mathbb{E}_{{\mathbf{z}_t}}\left[ \hat{\mathcal{R}}_{\star}^{\text{tr}}({\varphi}) - \mathcal{R}_{\star}^{\mathrm{tr}}({\varphi})  \right] \right|\\
    &\lesssim \frac{1}{2}\left(A^2+\mathbb{E}_{{\mathbf{z}_t,\mathbf{z}'_t}}\left[\mathcal{R}_{\star}^{\mathrm{tr}}({\varphi})\right]+2\varrho\right)+ \frac{M}{n}\log \mathcal{N}_R +2\varrho\\
    &=\frac{1}{2}\mathbb{E}_{{\mathbf{z}_t}}\left[\mathcal{R}_{\star}^{\mathrm{tr}}({\varphi})\right]+\frac{M}{n}\log \mathcal{N}_R+\frac{7}{2}\varrho.
\end{align*}}}
Thus, we have
{\small{\begin{equation*}
    \mathbb{E}_{{\mathbf{z}_t}}\left[ \mathcal{R}_{\star}^{\mathrm{tr}}({\varphi})  \right] \lesssim 2\mathbb{E}_{{\mathbf{z}_t}}\left[ \hat{\mathcal{R}}_{\star}^{\text{tr}}({\varphi}) \right] +\frac{M}{n}\log \mathcal{N}_R+7\delta,
\end{equation*}}}
which means that 
{\small{\begin{align*}
      C&\lesssim \mathbb{E}_{{\mathbf{x}_{t+1},\mathbf{y}_t}}\left[\ell^{\mathrm{tr}}_1 \right]+\frac{M}{n}\log \mathcal{N}_R+7\delta\\
      &\le \mathbb{E}_{{\mathbf{x}_{t+1},\mathbf{y}_t}}\left[\ell_1 \right]+\left|{A_1}\right|+\frac{M}{n}\log \mathcal{N}_R+7\delta\\
      &\lesssim D+ \exp\left(-C_1R^2\right)RM+({M}\log \mathcal{N}_R)/{n}+7\delta.
\end{align*}}}

\textit{Bounding $D$:}
Let $\hat{\mathcal{R}}_{\star}(\varphi)=\hat{\ell}(\varphi)-\hat{\ell}(\varphi^{\star})$, $\forall\varphi$. Next, $\ell_1=\hat{\mathcal{R}}_{\star}(\hat{\mathbf{s}})$. As $\hat{\mathbf{s}}$ is the minimizer of $\hat{\ell}$, it follows $${\small{\hat{\mathcal{R}}_{\star}({\varphi})=\hat{\ell}(\hat{\mathbf{s}})-\hat{\ell}(\varphi^{\star})\le \hat{\ell}(\varphi)-\hat{\ell}(\varphi^{\star})=\hat{\mathcal{R}}_{\star}(\varphi). }}$$ Thus, we have
% {\small{\begin{align*}
    $D=\mathbb{E}_{{\mathbf{z}_t}}\left[ \hat{\mathcal{R}}_{\star}({\varphi})  \right]\le \mathbb{E}_{{\mathbf{z}_t}}\left[ \hat{\mathcal{R}}_{\star}(\varphi)\right]=\mathcal{R}_{\star}(\varphi)$. 
% \end{align*}}}
Take the minimum w.r.t. $\varphi \in \mathcal{F}$. $D\le \min_{\mathbf{s} \in \mathcal{F}} \mathcal{R}_{\star}(\varphi)$. 

\textit{Balancing error:}
Combining the bounds for $B_1$, $B_2$, $C$, and $D$ and plugging the log covering number (\ref{equ::logcover}), we derive \eqref{equ::balance estimation error},
\begin{figure*}[ht]
    \centering
{\small{
\begin{subequations}
\begin{align}\label{equ::balance estimation error}
    \mathbb{E}_{{\mathbf{z}_t}}\left[ \mathcal{R}_{\star}({\varphi})  \right] \le& 2\min_{\mathbf{s} \in \mathcal{F}}\int_{k_0}^{K}\frac{1}{K-k_0}\mathbb{E}_{\tau,\mathbf{x}^k,\mathbf{y}}\left\|{\varphi(\mathbf{x}^k,\tau\mathbf{y},k)-\nabla \log P(\mathbf{x}^k|\tau \mathbf{y})}\right\|_2^2d k\nonumber\\
    &+2\min_{\mathbf{s} \in \mathcal{F}}\int_{\delta}^{K}\frac{1}{K-\delta}\mathbb{E}_{\tau,\mathbf{x}^k,\mathbf{y}}\left\|{\varphi(\mathbf{x}^k,\tau\mathbf{y},k)-\nabla \log P(\mathbf{x}^k|\tau \mathbf{y})}\right\|_2^2d k\nonumber\\
    & +\mathcal{O}\left(\frac{M}{n}N^{w_x+w_y}\log^{9} N \left( \log^8 N+\log^2 N\log R +\frac{1}{\varrho}  \right)\right)+\mathcal{O}\left( \exp\left(-C_1R^2\right)RM\right) +7\varrho\\
    \le& 2\min_{\mathbf{s} \in \mathcal{F}}\int_{k_0}^{K}\frac{1}{K-k_0}\mathbb{E}_{\tau,\mathbf{y}}\left[{\mathbb{E}_{\mathbf{x}^k}}\left\|{\varphi(\mathbf{x}^k,\tau\mathbf{y},k)-\nabla \log P(\mathbf{x}^k|\tau\mathbf{y})}\right\|_2^2\right]d k \\
    &+2\min_{\mathbf{s} \in \mathcal{F}}\int_{\delta}^{K}\frac{1}{K\!-\!\delta}\mathbb{E}_{\tau,\mathbf{y}}\left[\mathbb{E}_{\mathbf{x}^k}\left\|{\varphi(\mathbf{x}^k,\tau\mathbf{y},k)\!-\!\nabla \log P(\mathbf{x}^k|\tau\mathbf{y})}\right\|_2^2\right]d k\!+\!\mathcal{O}\left(\frac{M}{n}N\log^{17} N \right)\!+\!\mathcal{O}\left(N^{-\frac{2b}{w_x+w_y}}\right)\label{eq:final_1}
\end{align}
\end{subequations}}}
\vspace{-0.5cm}
\end{figure*}
{\color{black}where \eqref{eq:final_1} comes from Assumption~\ref{assump:distri_true}, $R=\sqrt{\frac{(C_\sigma+2b)\log N }{C_1(w_x+w_y)}}$ and $\varrho=N^{-2b/(w_x+w_y)}$, and
% we obtain \eqref{eq:sup3}.
% \begin{figure*}[ht]
% \centering
% \small{\color{black}\begin{align}\label{eq:sup3}
%     \mathbb{E}_{{\mathbf{z}_t}}\left[ \mathcal{R}_{\star}({\varphi})  \right] \le& 2\min_{\mathbf{s} \in \mathcal{F}}\int_{k_0}^{K}\frac{1}{K-k_0}\mathbb{E}_{\tau,\mathbf{x}^k,\mathbf{y}}\left\|{\varphi(\mathbf{x}^k,\tau\mathbf{y},k)-\nabla \log P(\mathbf{x}^k|\tau\mathbf{y})}\right\|_2^2d k \notag\\
%     &+2\min_{\mathbf{s} \in \mathcal{F}}\int_{\delta}^{K}\frac{1}{K-\delta}\mathbb{E}_{\tau,\mathbf{x}^k,\mathbf{y}}\left\|{\varphi(\mathbf{x}^k,\tau\mathbf{y},k)\!-\!\nabla \log P(\mathbf{x}^k|\tau\mathbf{y})}\right\|_2^2d k\!+\!\mathcal{O}\left(\frac{M}{n}N\log^{17} N \right)\!+\!\mathcal{O}\left(MN^{-2b-C_{\sigma}}\right) \notag\\
%     \le& 2\min_{\mathbf{s} \in \mathcal{F}}\int_{k_0}^{K}\frac{1}{K-k_0}\mathbb{E}_{\tau,\mathbf{y}}\left[{\mathbb{E}_{\mathbf{x}^k}}\left\|{\varphi(\mathbf{x}^k,\tau\mathbf{y},k)-\nabla \log P(\mathbf{x}^k|\tau\mathbf{y})}\right\|_2^2\right]d k \\
%     &+2\min_{\mathbf{s} \in \mathcal{F}}\int_{\delta}^{K}\frac{1}{K\!-\!\delta}\mathbb{E}_{\tau,\mathbf{y}}\left[\mathbb{E}_{\mathbf{x}^k}\left\|{\varphi(\mathbf{x}^k,\tau\mathbf{y},k)\!-\!\nabla \log P(\mathbf{x}^k|\tau\mathbf{y})}\right\|_2^2\right]d k\!+\!\mathcal{O}\left(\frac{M}{n}N\log^{17} N \right)\!+\!\mathcal{O}\left(N^{-\frac{2b}{w_x+w_y}}\right)\notag
% \end{align}}
% \vspace{-0.5cm}
% \end{figure*}
the inequality $M\lesssim\frac{1}{\delta}\le{\frac{1}{k_0}}=N^{C_\sigma}$.}
% for the last line of (\ref{equ::balance estimation error}). 
For any $k>0$, recall that the score function approximator $\zeta(\cdot, \cdot, k)$ obeys
{\small{\begin{align*}
&\mathbb{E}_{\tau,\mathbf{x}^k,\mathbf{y}}\left\|{\varphi(\mathbf{x}^k,\tau\mathbf{y},k)-\nabla \log P(\mathbf{x}^k|\tau\mathbf{y})}\right\|_2^2\\
=&\frac{1}{2}\int_{\mathbb{R}^{d}} \left\|{\varphi(\mathbf{x},\varnothing,k)-\nabla\log P(\mathbf{x})}\right\|^2 P(\mathbf{x})d\mathbf{x}\\
&+\frac{1}{2}\mathbb{E}_{\mathbf{y}}\left[\int_{\mathbb{R}^{d}} \left\|{\varphi(\mathbf{x},\mathbf{y},k)-\nabla\log P(\mathbf{x}|\mathbf{y})}\right\|^2 P(\mathbf{x}|\mathbf{y})d\mathbf{x}\right].
\end{align*}}}
Under Assumption~\ref{assump:distri_true}, $M=\mathcal{O}({1}/{k_0})$. Setting $N=n^{(w_x+w_y)/(w_x+w_y+b)}$ and considering Lemma~\ref{lemma:approx_prob},
% the error is bounded by
it follows that 
{\small{\begin{align*}
    \mathbb{E}_{{\mathbf{z}_t}}\left[ \mathcal{R}({\varphi})  \right]&\le 2\mathbb{E}_{{\mathbf{z}_t}}\left[ \mathcal{R}_{\star}({\varphi})  \right]\lesssim \frac{1}{t_0}n^{-\frac{b}{w_x+w_y+b}}\log^{\max(17,d+b/2+1)} n.
\end{align*}}}
Similarly, under Assumption~\ref{assump:distri_true}, $M=\mathcal{O}(\log {k_0})$. Setting $N=n^{(w_x+w_y)/(w_x+w_y+2b)}$ and considering Lemma~\ref{lemma:approx_prob}, the conditional score error is bounded by
{\small{\begin{equation*}
    \mathbb{E}_{{\mathbf{z}_t}}\left[ \mathcal{R}({\varphi})  \right]\lesssim\frac{1}{t_0} n^{-\frac{2b}{w_x+w_y+2b}}\log^{\max(17,(b+1)/2)} n.
\end{equation*}}}
% This proof is complete.

% \subsection{Supporting lemmas}

% We first introduce a standard result of bounding the covering number of a ReLU neural network.

% First, we prove the approximation theory for using ReLU neural networks to approximate the conditional score, that is

% \begin{lemma}\label{lemma::bound loss class}
% Define $m_k=M_k/\sqrt{\log N}$. For any $\mathbf{s}\in \mathcal{F}(M_k, W,\kappa,L,P)$  in Lemma~\ref{lemma:DM_offline_9} and $(\mathbf{x},\mathbf{y})\in \mathcal{D}$, we have $\left|{\ell(\mathbf{s},\mathbf{x},\mathbf{y})}\right|\lesssim \int_{k_0}^{K} m_{k}^2 d k\triangleq M$. Take $k_0=n^{-\mathcal{O}(1)}$ and $K=\mathcal{O}(\log n)$. Under $\delta\ge k_0$, we have $M=\mathcal{O}(\log k_0)$ for $m_k=\frac{1}{\sigma_k}$, and $M=O\left(\frac{1}{k_0}\right)$ for $m_k=\frac{1}{\sigma^2_k} $.
% \end{lemma}
% \begin{proof}
    
% \end{proof}

\bibliographystyle{IEEEtran}
\bibliography{ref.bib}

\end{document}